\documentclass{article}

\PassOptionsToPackage{numbers, compress}{natbib}

\usepackage[final]{neurips_2021}

\usepackage[utf8]{inputenc} %
\usepackage[T1]{fontenc}    %
\usepackage{hyperref}       %
\usepackage{url}            %
\usepackage{booktabs,comment}       %
\usepackage{amsfonts}       %
\usepackage{nicefrac}       %
\usepackage{microtype}      %
\usepackage{xcolor}         %
\usepackage{amsmath}
\usepackage{amssymb}
\usepackage{amsthm}
\usepackage{bbm}
\usepackage{xcolor}
\usepackage{graphicx}
\usepackage{subcaption}
\usepackage{adjustbox}
\usepackage{url}

\usepackage{amsmath}
\usepackage{amsthm}
\usepackage{amssymb}
\usepackage{color}
\usepackage{mathtools}
\usepackage{makecell}
\usepackage{bm}
\usepackage{todonotes}
\usepackage{diagbox}
\usepackage{tablefootnote}
\usepackage{hyperref}
\usepackage[ruled,vlined,linesnumbered]{algorithm2e}
\usepackage{layouts}

\DeclarePairedDelimiter{\ceil}{\lceil}{\rceil}

\DeclareMathOperator{\sgn}{sgn}
\DeclareMathOperator{\poly}{poly}

\newcommand{\pd}[2]{\frac{\partial #1}{\partial #2}} 

\let\baraccent=\= %
\renewcommand{\=}[1]{\stackrel{#1}{=}} %

\providecommand{\RR}{\mathbb{R}}

\providecommand{\PP}{\mathbb{P}}
\providecommand{\EE}{\mathbb{E}}

\providecommand{\eps}{\epsilon}

\mathchardef\mhyphen="2D %

\providecommand{\sm}{\setminus}

\newcommand{\interior}[1]{%
  {\kern0pt#1}^{\mathrm{o}}%
}

\usepackage[capitalise]{cleveref}
\newtheorem{theorem}{Theorem}[section]
\newtheorem{lemma}[theorem]{Lemma}

\newtheorem{claim}[theorem]{Claim}

\newtheorem{assumption}[theorem]{Assumption}

\newtheorem{definition}[theorem]{Definition}

\title{The staircase property: \\ How hierarchical structure can guide deep learning}

\author{%
    Emmanuel Abbe \\ EPFL \\ \texttt{emmanuel.abbe@epfl.ch} \And  Enric Boix-Adsera \\ MIT \\ \texttt{eboix@mit.edu} \And Matthew Brennan \\ MIT \And Guy Bresler \\ MIT \\ \texttt{guy@mit.edu} \And Dheeraj Nagaraj \\ MIT \\ \texttt{dheeraj@mit.edu}
}

\begin{document}

\maketitle

\begin{abstract}
This paper identifies a structural property of data distributions that enables deep neural networks to learn hierarchically. We define the ``staircase'' property for functions over the Boolean hypercube, which posits that high-order Fourier coefficients are reachable from lower-order Fourier coefficients along increasing chains. We prove that functions satisfying this  property can be learned in polynomial time using layerwise  stochastic coordinate descent on regular neural networks -- a class of network architectures and initializations that have homogeneity properties. Our analysis shows that for such staircase functions and neural networks, the gradient-based algorithm learns high-level features by greedily combining lower-level features along the depth of the network. We further back our theoretical results with experiments showing that staircase functions are learnable by more standard ResNet architectures with stochastic gradient descent. Both the theoretical and experimental results support the fact that the staircase property has a role to play in understanding the capabilities of gradient-based learning on regular networks, in contrast to general polynomial-size networks that can emulate any Statistical Query or PAC algorithm, as recently shown.

\end{abstract}

\section{Introduction}

It has been observed empirically that neural networks can learn hierarchically. 
For example, a `car' may be detected by first understanding simpler concepts like `door', 'wheel', and so forth in intermediate layers, which are then combined in deeper layers (c.f. \cite{olah2020zoom,zeiler2014visualizing}). However, on the theoretical side, the mechanisms by which such hierarchical learning occurs are not yet fully understood. In this paper we are motivated by the following question:

\begin{center}
\parbox{5in}{
    \emph{Can we 
    identify naturally structured and interpretable classes of hierarchical functions, and show how regular\footnotemark DNNs are able to learn them?
    }}
\end{center}
\footnotetext{The notion of regularity is specified in \cref{reg_def}; this means network architectures and initializations that have homogeneity properties within layers, in contrast to the emulation architectures in \cite{AS20,newAS}.}

This is a refinement of the generic objective of trying to understand DNNs:
We identify several key desiderata for any theoretical result in this direction.
(1) \emph{Natural structure}: We aim to capture naturally occurring data of interest, so the structural assumption must make conceptual sense. 
(2) \emph{Interpretability}: 
If we hope to clearly interpret the inner workings of neural networks, understanding both how they classify and also how they learn, then
we need a model for data that is interpretable to begin with. Interpretation of the representations occurring within a neural network is most clearly expressed with respect to structural properties of the data.   
Finally, (3) \emph{Regularity of the network}: The network architecture and initialization should be symmetric in a sense defined later on. This prevents using carefully-crafted architectures and initializations to emulate general learning algorithms \cite{AS20,newAS}. 
We view this type of restriction as being partway towards considering practical neural networks that learn in a blackbox fashion. 
The results in this paper aim to satisfy all three high-level objectives. The relation with prior work is discussed in \cref{sec:related}.

This paper proposes a new structurally-defined class of hierarchical functions and proves guarantees for learning by regular neural networks. In order to describe this structure, we first recall that
any function $f:\{+1,-1\}^n\to \mathbb{R}$ can be decomposed in the Fourier-Walsh basis as
\begin{align}
    f(x) = \sum_{S\subseteq [n]} \hat{f}(S) \chi_S (x), \quad \text{where} \quad \hat{f}(S) := \langle f, \chi_S \rangle, \quad
   \chi_S(x) := \prod_{i\in S} x_i
\end{align}
and the inner product between two functions is 
$\langle f, g \rangle = \mathbb{E} f(X) g(X)$ for $X\sim\mathrm{Unif}(\{+1,-1\}^n)$. This decomposition expresses $f(x)$ as a sum of components, each of which is a monomial $\chi_S(x)$, weighted by the Fourier coefficient $\hat{f}(S)$. Our definition of hierarchical structure is motivated by an observation regarding two closely related functions, ``high-degree monomials'' and ``staircase functions'', the latter of which can be learned efficiently and the former of which cannot.

\textbf{Monomials with no hierarchical structure}\quad The class of monomials of any degree $k$ where $k \leq n/2$ (i.e., the class
$\{\chi_S\}_{S \subseteq [n], |S| = k}$) is efficiently learnable by Statistical Query (SQ) algorithms if and only if $k$ is constant~\cite{kearns1998efficient,blum1994weakly}, and the same holds for noisy Gradient Descent (GD) on neural nets with polynomially-many parameters \cite{kearns1998efficient}, and for noisy Stochastic Gradient Descent (SGD) where the batch-size is sufficiently large compared to the gradients' precision \cite{AS20,newAS}. This was also noted in \cite{shalev2017failures} which shows that gradients carry little information to reconstruct $\chi_S$ for large $|S|$, and hence gradient-based training is expected to fail. Thus, we can think of a component $\chi_S$ as \emph{simple} and easily learnable if the degree $|S|$ is small and \emph{complex} and harder to learn if the degree $|S|$ is large. 

\textbf{Staircase functions with hierarchical structure}\quad Now, instead of a single monomial, consider the following \emph{staircase} function (and its orbit class induced by permutations of the inputs), which is a sum of monomials of increasing degree:
\begin{align}\label{eq:staircase}
S_k(x)&=x_1+x_1x_2+x_1x_2x_3+x_1x_2x_3x_4+\dots +\chi_{1:k}\,.
\end{align}
 Here $S_k(x)$ has a hierarchical structure, where $x_1$ builds up to $x_1 x_2$, which builds up to $x_1 x_2 x_3$, and so on until the degree-$k$ monomial $\chi_{1:k}$. Our experiments in \cref{fig:main_expt} show a dramatic difference between learning a single monomial $\chi_{1:k}$ and learning the staircase function $S_k$. Even with $n=30$ and $k = 10$, the same network with $5$ $\mathsf{ReLU}$ ResNet layers and the same hyperparameters can easily learn $S_k$ to a vanishing error (\cref{fig:loss_stair_5}) whereas, as expected, it cannot learn $\chi_{1:k}$ even up to any non-trivial error since $\chi_{1:k}$ is a high-degree monomial (\cref{fig:loss_parity_5}).

An explanation for this phenomenon is that the neural network learns the staircase function $S_k(x)$ by first learning a degree-1 approximation that picks up the feature $x_1$, and then uses this to more readily learn a degree-2 approximation that picks up the feature $x_1x_2$, and so on, progressively incrementing the degree of the approximation and `climbing the staircase' up to the large degrees. We refer to \cref{fig:illustration_1} for an illustration. This is indeed the learning mechanism, as we can see once we plot the Fourier coefficients of the network output against training iteration. Indeed, in \cref{fig:fourier_coeff_parity_5} we see that the network trained to learn $\chi_{1:10}$ cannot learn \emph{any} Fourier coefficient relevant to $\chi_{1:10}$ whereas in \cref{fig:fourier_coeff_stair_5} it is clear that the network trained to learn $S_{10}$ learns the relevant Fourier coefficients in order of increasing complexity and eventually reaches the $\chi_{1:10}$ coefficient.

\begin{figure}[ht]%
\centering
	\begin{minipage}[t][5.1cm][t]{\textwidth}
	\centering
	
	\adjustbox{valign=t}{
	\fbox{\begin{minipage}[t][3.6cm][t] {0.47\textwidth}
	\begin{subfigure}{\textwidth}
	\centering

	\includegraphics[height = 0.5\linewidth, width=\linewidth]{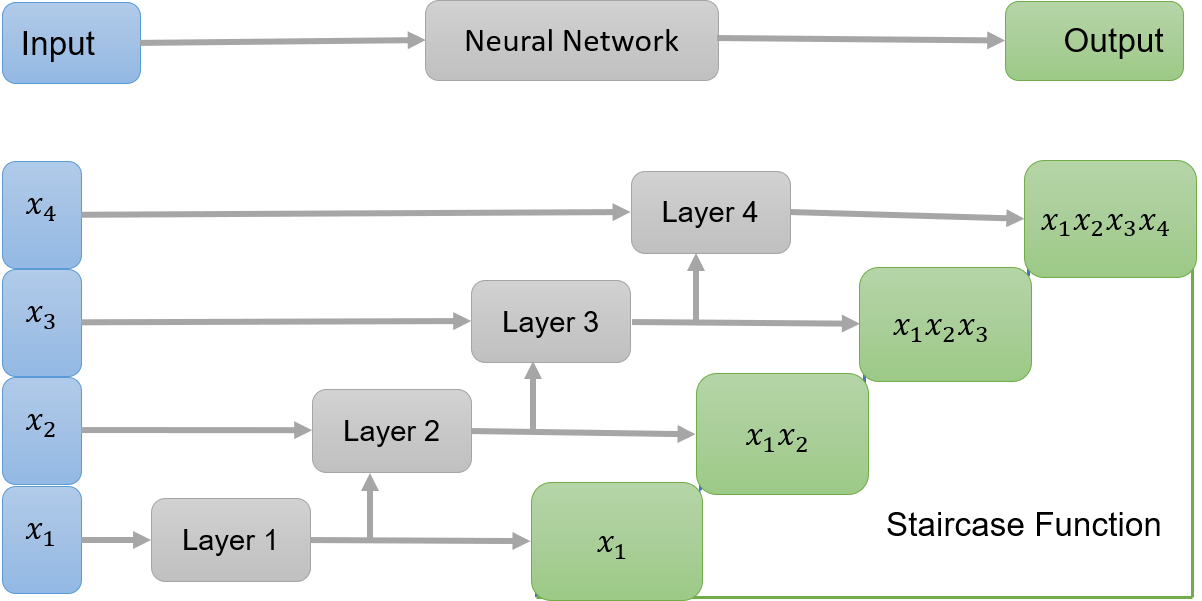}
	\vspace*{0.05cm}
	\caption{An illustration of hierarchical learning where successive layers build upon the features from previous layers.}
	\label{fig:illustration_1}
	\end{subfigure}\end{minipage}}
	
	\centering
	\fbox{\begin{minipage}[t][3.6cm][t]{0.47\textwidth}
	\begin{subfigure}{\textwidth}
		\centering
		\includegraphics[height = 0.5\linewidth, width=\linewidth]{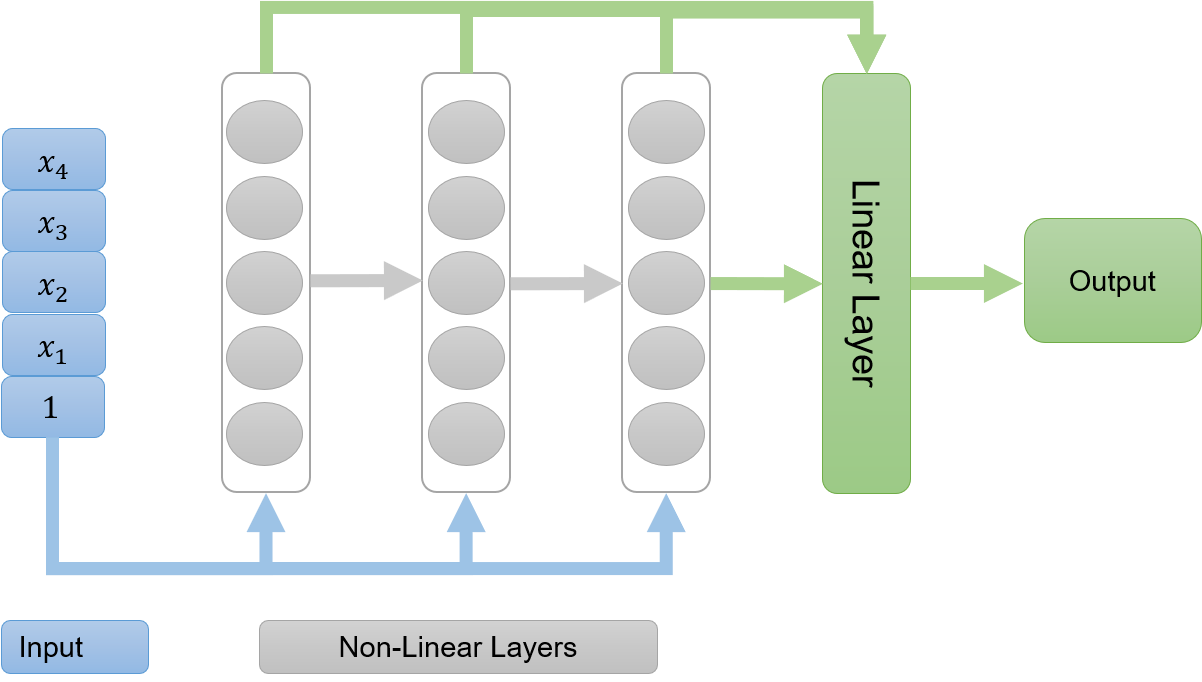}  
		\vspace*{0.05cm}
		  \caption{An illustration of the proposed architecture. The solid blue and grey arrows represent sparse random connections}
    \label{fig:architecture}
	\end{subfigure}\end{minipage}}}
    \end{minipage}
	\caption{Hierarchical learning method and proposed architecture.}
	\label{fig:illustration_network}
\end{figure}

\begin{figure}[ht]%
\centering
	\begin{minipage}[t][8cm][t]{\textwidth}
	\centering
	
	\adjustbox{valign=t}{
	\fbox{\begin{minipage}[t][3.3cm][t] {0.47\textwidth}
	\begin{subfigure}{\textwidth}
	\centering
	\includegraphics[height = 0.5\linewidth, width= \linewidth]{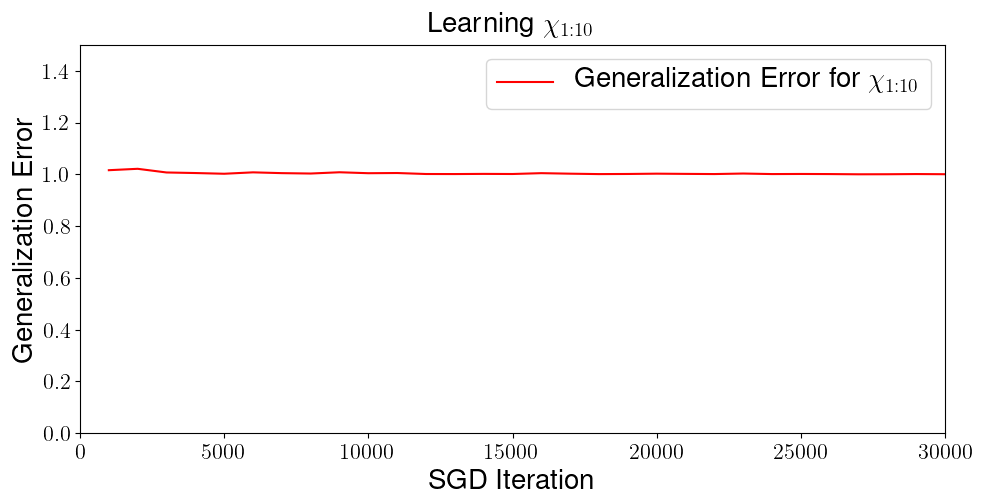}
	\caption{Loss Evolution for Learning Parity}
	\label{fig:loss_parity_5}
	\end{subfigure}\end{minipage}}
	\centering
    
    \fbox{\begin{minipage}[t][3.3cm][t] {0.47\textwidth}
	\begin{subfigure}{\textwidth}
	\centering
	\includegraphics[height = 0.5\linewidth, width= \linewidth]{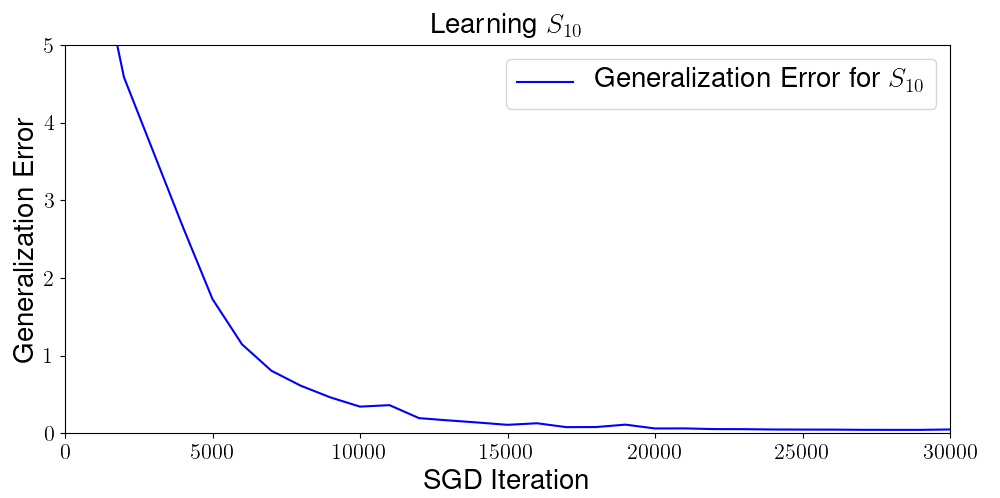}
	\caption{Loss Evolution for Learning $S_{10}$}
	\label{fig:loss_stair_5}
	\end{subfigure}\end{minipage}}
	}

\vspace*{0.5cm}
\centering
\hspace*{-0.2cm}
	\adjustbox{valign=t}{

	\fbox{\begin{minipage}[t][3.3cm][t]{0.47\textwidth}
	\begin{subfigure}{\textwidth}
		\centering
		\includegraphics[height = 0.5\linewidth, width=\linewidth]{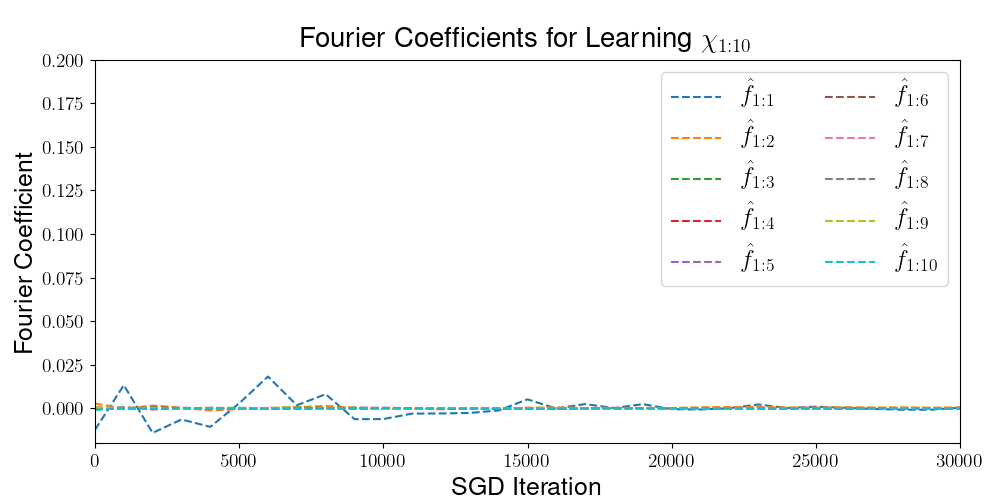}  
		  \caption{Fourier Coefficients for Parity}
    \label{fig:fourier_coeff_parity_5}
	\end{subfigure}\end{minipage}}
	\centering

	\fbox{\begin{minipage}[t][3.3cm][t]{0.47\textwidth}
	\begin{subfigure}{\textwidth}
		\centering
		\includegraphics[height = 0.5\linewidth, width=\linewidth]{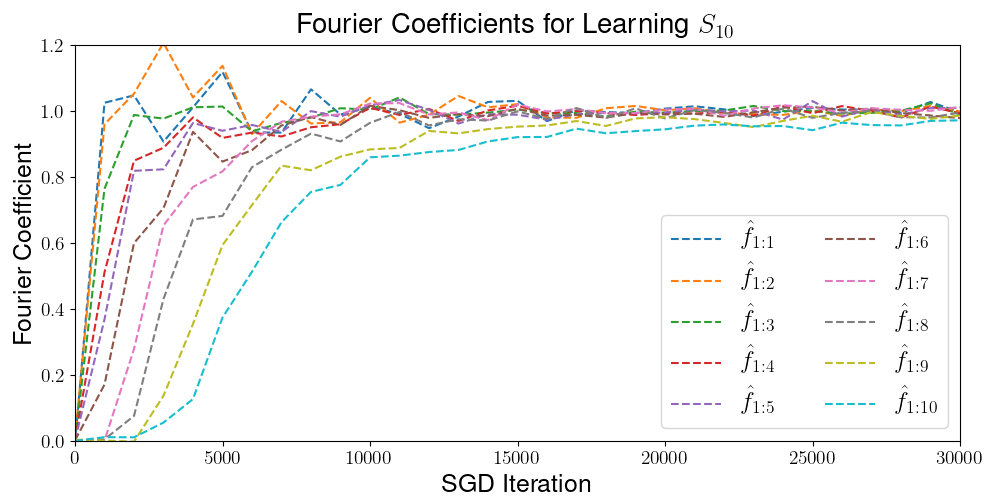}  
		  \caption{Fourier Coefficients for $S_{10}$}
    \label{fig:fourier_coeff_stair_5}
	\end{subfigure}\end{minipage}}}
	 \end{minipage}
  
	\caption{Comparison between training $\chi_{1:10}$ and $S_{10}$ with $n= 30$ on the same 5-layer $\mathsf{ReLU}$ ResNet of width 40. Training is SGD with constant step size on the square loss. Here $\hat{f}_{1:i}$ denotes the Fourier coefficient $\langle \chi_{1:i},f\rangle$ corresponding to the network output $f$.}
	\label{fig:main_expt}
\end{figure}

\textbf{Main results}\quad We shed light on this phenomenon, proving that certain regular networks efficiently learn the staircase function $S_k(x)$, and, more generally, functions satisfying this structural property:
\begin{definition}[Staircase property]\label{def:staircaseprop} For any $M > 1$, a function
$g : \{-1,1\}^n \to \RR$ satisfies the $[1/M,M]$-staircase property over the unbiased binary hypercube if:
\begin{itemize}
    \item for all $S \subset [n]$, if $\hat{g}(S) \neq 0$ then $|\hat{g}(S)| \in [1/M,M]$.
    \item for all $S \subset [n]$, if $\hat{g}(S) \neq 0$ and $|S| \geq 2$, there is $S' \subset S$ such that $|S \sm S'| = 1$ and $\hat{g}(S') \neq 0$.
\end{itemize}
Furthermore, $g$ is said to be an $s$-sparse polynomial if $|\{S : \hat{g}(S) \neq 0\}| \leq s$.
\end{definition}
The parameters $M$ and $s$ appear naturally since a PAC-learning algorithm for $s$-sparse polynomials satisfying the $[1/M,M]$-staircase property must use a number of samples that depends polynomially on $M$ and $s$. Our theoretical result is informally summarized as follows, and we remark that the proof shows that the neural network progressively learns approximations of higher degree:

\begin{theorem}[Informal statement of \cref{thm:forward}]
Let $g : \{-1,1\}^n \to \RR$ be an unknown $s$-sparse polynomial satisfying the $[1/M,M]$-staircase property. Given access to random samples from $\{(x,g(x))\}_{x \sim \{-1,1\}^n}$, there is a regular neural network architecture that approximately learns $g$ in $\poly(n,s,M)$ time and samples when trained with layerwise stochastic coordinate descent. 
\end{theorem}

Even though we only consider hierarchical functions over the Boolean hypercube $\{-1,1\}^n$ in our theoretical result, we believe that the techniques used in this work can be extended to other function spaces of interest, exploiting the orthonormality of the corresponding Fourier basis functions. For this reason we give a fairly general definition of hierarchical functions in \cref{sec:main_def} that goes beyond the Boolean hypercube, as well as beyond the strict notion of increasing chains. This more general class of functions is of further interest because it includes as special cases well-studied classes such as biased sparse parities and decision trees in a smoothed complexity setting (see \cref{rem:dectrees}).

\subsection{Related Work} \label{sec:related}
\textbf{Statistical query emulation results}\quad
For the general class of polynomial-size neural network architectures with any choice of initialization, it is known that SGD on a sufficiently small batch-size can learn\footnote{These reductions are for polynomial-time algorithms and for polynomial precisions on the gradients.} any function class (including functions satisfying the staircase property) that is efficiently learnable from samples \cite{AS20}, while GD can learn any function class that is efficiently learnable from statistical queries (SQ) \cite{newAS}. However, these results rely on highly non-regular architectures and initializations, with different parts of the nets responsible for different tasks that emulate the computations of general learning algorithms. In particular, it is not known how to obtain the emulation results of \cite{newAS} for ``regular'' architectures and initializations as defined in \cref{reg_def}. In contrast, our architecture in \cref{thm:forward} is a regular neural network in this sense, and our analysis further illustrates how features are built greedily over depth rather than by emulating a given algorithm. 

Consider also the orbit class under permutations of the inputs of the ``truncated staircase function'', $S_{j\to k}(x)=\sum_{i=j}^k \chi_{i:k}(x)$, for $1 \le j \le k \le n$. Note that this class is efficiently SQ-learnable when $k=n$, since the monomial $\chi_{1:n}$ is always present and one can recursively check which sub-monomial is present or not by checking at most $n$ monomials at each step.
However, we conjecture that $S_{j\to n}(x)$ is not learnable by regular networks trained with gradient descent if $\min(n-j,j)=\omega(1)$. Therefore such truncated staircases provide a candidate for separating gradient-based learning on regular networks versus general, non-regular networks that allow for emulating any SQ algorithm \cite{newAS}.

\textbf{Hierarchical models of data}\quad
Explicitly adding hierarchical structures into machine learning algorithms such as hierarchical Bayesian modeling and hierarchical linear modeling has proved successful in various machine learning tasks beyond deep learning \cite{kulkarni2016hierarchical,woltman2012introduction,friedman1997bayesian,rokach2005clustering}. For image data, \cite{bruna2013invariant,ye2018deep} propose hierarchical generative models of images and use them to motivate deep convolutional architectures, although these works do not prove that deep learning learns these generative models. \cite{patel2015probabilistic} similarly proposes a `deep rendering model' which hierarchically models levels of abstraction present in data, but does not prove learnability. \cite{malach2018provably} gives a training algorithm for deep convolutional networks that provably learns a deep generative model of images. The paper \cite{mossel2016deep} proposes a generative model of data motivated by evolutionary processes, and proves in a formal sense that ``deep'' algorithms can learn these models, whereas shallow algorithms cannot. In contrast to our work, the ``deep'' algorithms considered by \cite{mossel2016deep} are not descent algorithms on regular deep neural network architectures. In \cite{basri2019convergence} it is shown that during training of two-layer $\mathsf{ReLU}$ networks with SGD, the lower frequency components of the target function are learned first. Unfortunately, their results have an exponential dependence on the degree. %
In our work, we leverage depth and the hierarchical Boolean function structure to ensure that higher-level Fourier coefficients are learned efficiently. Finally, \cite{malach2020implications}, studies learning Boolean circuits of depth $O(\log n)$ via neural networks under product distributions using layer-wise gradient descent. While \cite{malach2020implications} requires the architecture to match the Boolean circuit being learned, in contrast, our architecture is regular and independent of the function learned.

\textbf{Power of depth}\quad Several works have studied how representation power depends on depth. \cite{mhaskar2016learning} shows that deep networks can represent a class of compositionally-created functions more efficiently than shallow networks. \cite{eldan2016power} shows that certain smooth radial functions can be easily represented by three-layer networks but need exponentially-many neurons to be represented by a two-layer network.
Based on an analysis of learning fractal distributions related to the Cantor set, \cite{malach2019deeper} conjectures that if shallow networks are poorly represent a target function, then a deep network cannot be trained efficiently using gradient based methods.
\cite{telgarsky2016benefits} presents a depth separation result by showing that deep networks can produce highly oscillatory functions by building on the oscillations layer by layer. \cite{bresler2020sharp} uses this phenomenon to show a sharp representation theorem for arbitrary-depth $\mathsf{ReLU}$ networks.

Other theoretical works have proved depth separation theorems for training. \cite{chen2020towards} prove that a two-hidden-layer neural network where the first hidden layer is kept random and only the second layer is trained, provably outperforms a just one hidden layer network.
In \cite{allen2019can,allen2020backward} it is proved that deep networks trained end-to-end with SGD and quadratic activation functions can efficiently learn a non-trivial concept class hierarchically, whereas kernel methods and lower-depth networks provably fail to do so. The class of functions studied by \cite{allen2019can, allen2020backward} are those representable as the sum of neurons in a teacher network that is well-conditioned, has quadratic activations and has a depth of at most $\log \log n$, where $n$ is the number of inputs. This function class is expressive but incomparable to the hierarchical function class studied in our work (e.g., we can learn polynomials up to degree $n$, whereas \cite{allen2020backward} is limited to degree $\log(n)$). Furthermore, our function class has the advantage of being naturally interpretable, with complex features (the high-order monomials) being built in a transparent way from simple features (the low-order monomials). In \cite{saxe2013exact,gidel2019implicit,gissin2019implicit}, gradient dynamics are explored for the simplified case of deep linear networks, where an `incremental learning' phenomenon is observed in which the singular values are learned sequentially, one after the other. This phenomenon is reminiscent of the incremental learning of the Fourier coefficients of the Boolean function in our setting (\cref{fig:main_expt}). For real world data sets, \cite{kalimeris2019sgd} empirically shows that in many data sets of interest, simple neural networks trained with SGD first fit a linear classifier to the data and then progressively improve the approximation, similar in spirit to the theoretical results in this paper.

\textbf{Neural Tangent Kernel and random features}\quad A sequence of papers have studied convergence of overparametrized (or Neural Tangent Kernel regime) neural networks to the global minimizer of empirical loss when trained via gradient descent. In this regime, they show that neural networks behave like kernel methods and give training and/or generalization guarantees. Because of the reduction to kernels, these results are essentially non-hierarchical. \cite{bietti2020deep,huang2020deep} in fact show that deep networks in the NTK regime behave no better than shallow networks. We refer to \cite{allen2020backward} for a review of the literature related to NTK and shallow learning. Finally, we mention the related works \cite{andoni2014learning,yehudai2019power,bresler2020corrective} which consider learning low-degree (degree $q$) polynomials over $\mathbb{R}^n$, without a hierarchical structure assumption. They require $n^{\Omega(q)}$ neurons to learn such functions, which is super-polynomial once $q \gg 1$. The results hold in the random features regime, known to be weaker than NTK.

\subsection{Organization}
In \cref{sec:main_results}, we give the problem setup, network architecture and the training algorithm, and also state our rigorous guarantee that this training algorithm learns functions satisfying the staircase property. In \cref{sec:main_def} we discuss possible extensions, defining hierarchical functions satisfying the staircase property in a greater level of generality from \cref{def:staircaseprop}. We refer to \cref{sec:additional_experiments} for additional experiments which validate our theory and conjectures for both the simplest definition of the staircase property in \cref{def:staircaseprop} and the generalizations in \cref{ssec:moregeneral}.

\newcommand{\gt}{\tilde{g}}
\newcommand{\ellh}{\hat{\ell}}
\newcommand{\ellt}{\tilde{\ell}}
\newcommand{\Deltat}{\tilde{\Delta}}
\newcommand{\ellr}{\ell_R}
\newcommand{\elltr}{\ellt_R}
\newcommand{\Vin}{V_{\mathrm{in}}}
\newcommand{\vin}[1]{v_{\mathrm{in},#1}}

\newcommand{\eout}{\eps_0}
\newcommand{\estop}{\eps_{stop}}
\newcommand{\ewant}{\eps}
\newcommand{\ellhr}{\ellh_{R}}
\newcommand{\TrainNeuron}{\textsc{TrainNeuron}}
\newcommand{\TrainNetworkLayerwise}{\textsc{TrainNetworkLayerwise}}
\newcommand{\egrad}{\eps}
\newcommand{\dgrad}{\delta}
\newcommand{\ebias}{\eps_{bias}}
\newcommand{\estat}{\eps_{stat}}
\newcommand{\elocal}{\estat}
\newcommand{\Uneur}{U_{neur}}
\newcommand{\Uparam}{U_{param}}
\newcommand{\erelmax}{\eps_{relmax}}
\newcommand{\tfourier}{\tau_{fourier}}
\newcommand{\elowf}{\eps_{lowfourchange}}
\newcommand{\efouriermove}{\eps_{fourmove}}
\newcommand{\elearnthresh}{\eps_{learnthresh}}
\newcommand{\etrainedsmall}{\eps_{trainedsmall}}
\newcommand{\Erep}[1]{E_{rep,#1}}
\newcommand{\Ereplayer}[1]{E_{replayer,#1}}
\newcommand{\Estepgood}[1]{E_{stepgood,#1}}
\newcommand{\Elayergood}[1]{E_{layergood,#1}}
\newcommand{\Econn}[1]{E_{conn,#1}}
\newcommand{\nshared}{n_{shared}}
\newcommand{\Enothree}[1]{E_{nothree,#1}}
\newcommand{\Eneurbound}[1]{E_{neurbound,#1}}
\newcommand{\Ebias}[1]{E_{bias,#1}}
\newcommand{\erel}[1]{\eps_{rel,#1}}
\newcommand{\dstat}{\delta_{stat}}
\newcommand{\Ulossone}{U_{lossone}}
\newcommand{\Uneurtrain}{U_{neurtrain}}
\newcommand{\Uatrain}{U_{atrain}}
\newcommand{\Ubatchlow}{U_{batchlow}}
\newcommand{\Ubtrain}{U_{paramstrained}}
\newcommand{\Ugradf}{U_{gradf}}
\newcommand{\Ugradl}{U_{gradl}}
\newcommand{\Ugradlr}{U_{gradlr}}
\newcommand{\Usmoothness}{U_{smoothness}}
\newcommand{\Edecloss}[1]{E_{decloss,#1}}
\newcommand{\tbound}{t_{max}}
\newcommand{\edec}{\eps_{dec}}
\newcommand{\Egradclose}[1]{E_{good,{#1}}}
\newcommand{\Estat}{E_{stat}}
\newcommand{\Estatiter}[1]{E_{stat,#1}}
\newcommand{\elocshift}{\eps_{2}}
\newcommand{\elocshifttwo}{\eps_{3}}
\newcommand{\eae}{\eps_4}
\newcommand{\ebv}{\eps_5}
\newcommand{\Upsilont}{\tilde{\Upsilon}}
\newcommand{\Ups}{\Upsilon}
\newcommand{\Upst}{\Upsilont}
\newcommand{\egraddiff}{\eps_{graddiff}}
\newcommand{\ekey}{\eps_2}
\newcommand{\enewrel}{\eps_{newrel}}
\newcommand{\Eerrorbounded}[1]{E_{errbdded,#1}}
\newcommand{\valret}[1]{#1^{\mathrm{round}}}
\newcommand{\wret}{\valret{w}}
\newcommand{\wvret}{\valret{w_v}}
\newcommand{\aeret}{\valret{a_e}}
\newcommand{\aeoneret}{\valret{a_{e_1}}}
\newcommand{\aetworet}{\valret{a_{e_2}}}
\newcommand{\aeprimeret}{\valret{a_{e'}}}
\newcommand{\bvret}{\valret{b_v}}
\newcommand{\aeiret}{\valret{a_{e_i}}}
\newcommand{\Tfinal}{T}
\newcommand{\NeuronSGD}{\textsc{NeuronSGD}}
\newcommand{\Elossdec}[1]{E_{good,#1}}
\newcommand{\wsgd}{w^{SGD}}
\newcommand{\wvsgd}{w_v^{SGD}}
\newcommand{\wperturb}{w^{perturb}}
\newcommand{\wvperturb}{w_v^{perturb}}
\newcommand{\aeperturb}{a_e^{perturb}}
\newcommand{\aeoneperturb}{a_{e_1}^{perturb}}
\newcommand{\aetwoperturb}{a_{e_2}^{perturb}}
\newcommand{\bvperturb}{b_v^{perturb}}
\newcommand{\Ustat}{U_{stat}}
\newcommand{\Ecasetwo}{E_{newactive}}
\newcommand{\Ecasetwoiter}[1]{E_{newactive,#1}}
\newcommand{\Epol}[1]{E_{pol,{#1}}}
\newcommand{\Eactiveparents}[1]{E_{activeparents,{#1}}}
\newcommand{\Emon}[1]{E_{mon,{#1}}}
\newcommand{\elearned}{\eps_{learned}}
\newcommand{\Eparambound}[1]{E_{parambound,#1}}
\newcommand{\Egoodinit}{E_{goodinit}}
\newcommand{\Enobadactive}[1]{E_{nobadactive,#1}}

\section{Regular networks provably learn hierarchical Boolean functions}
\label{sec:main_results}

We state our main theoretical result, which proves that a regular neural network trained with a descent algorithm learns hierarchical Boolean functions in polynomial time.

\subsection{Architecture}\label{ssec:architecture}

Our network architecture has neuron set $V$ and edge set $E$, and is defined as follows (see also \cref{fig:architecture}). The neuron set is $V = \Vin \sqcup V_1 \sqcup \dots V_L$. Here $\Vin = \{\vin{0},\vin{1},\ldots,\vin{n}\}$ is a set of $n+1$ inputs, and each intermediate layer consists of $|V_i| = W$ neurons. Furthermore, the edge set $E$ is a sparse, random subset of all possible directed edges:
\begin{itemize}
\item each $(v_0,v_i) \in \Vin \times V_i$ is in the edge set $E$ independently with probability $p_1$, and
\item each $(v_i,v_{i+1}) \in V_i \times V_{i+1}$ for $i \in [L-1]$ is in the edge set $E$ independently with probability $p_2$.
\end{itemize}
For each edge $e \in E$, let there be a weight parameter $a_e \in \RR$. And for each neuron $v \in V \sm \Vin$, let there be a bias parameter $b_v \in \RR$. The parameters of the network are therefore $a \in \RR^E$ and $b \in \RR^{V \sm \Vin}$. For simplicity of notation, we concatenate these two vectors into one vector of parameters $$w = [a\quad  b] \in \RR^{E} \oplus \RR^{V \sm \Vin}.$$ For each $i \in [n]$, the  $i$th input, $\vin{i} \in \Vin$, computes $x_i$, and the $0$th input, $\vin{0}$, computes a constant:
$$f_{\vin{i}}(x; w) = x_i, \mbox{ and } f_{\vin{0}}(x; w) = 1.$$
Given a neuron $v \in V \sm \Vin$, the function computed at that neuron is a quadratic function of a linear combination of neurons with edges to $v$, (i.e., the activation function is quadratic). And the output of the neural network is the sum of the values of the neurons at the intermediate layers:
$$f_{v}(x; w) = \left(\sum_{e = (u,v) \in E} a_e f_u(x; w)\right)^2 + b_{v}, \quad\mbox{ and }\quad f(x; w) = \sum_{v \in V \sm \Vin} f_{v}(x; w).$$

Our architecture satisfies the following regularity condition:
\begin{definition}[Regular network architecture and initialization]\label{reg_def}
An architecture is regular if for any $1\le i \le j \le L$, for any distinct pair of potential edges $(v_i,v_j), (v'_i,v'_j) \in V_i \times V_j$, the events that these edges are in $E$ are i.i.d.; the same holds for any distinct pair of potential edges $(u,v_j), (u',v_j') \in \Vin \times V_j$; the same holds for any distinct pair of potential edges $(u,v_{\mathrm{out}}), (u',v_{\mathrm{out}}) \in V_j \times v_{\mathrm{out}}$ (where $v_{\mathrm{out}}$ is the output vertex). Furthermore, the initialization is regular if it is i.i.d. over the set of present edges and each weight has a symmetric distribution.

\end{definition}
In our case the weight initialization is i.i.d. and symmetric since we choose it to be identically zero everywhere, which works since we escape saddle points by perturbing during training. On the other hand in our experiments the initialization is an isotropic Gaussian, which also satisfies \cref{reg_def}.

\subsection{Loss function}
Let the loss function be the mean-squared-error between the output $f$ of the network and a function $g : \{-1,1\}^n \to \RR$ that we wish to learn. Namely, for any $x \in \{-1,1\}^n$, $a \in \RR^E$ and $b \in \RR^{V \sm \Vin}$, define the point-wise loss and population loss functions respectively, where $w = [a,\ b]$:
\begin{equation}\label{eq:loss}
\ell(x; w) = \frac{1}{2}(f(x; w) - g(x))^2;\quad \ell(w) = \EE_{x \sim \{-1,1\}^n} \ell(x; w)\,.
\end{equation}

We will train the neural network parameters to minimize an $L_2$-regularized version of the loss function. Let $\lambda_1, \lambda_2 > 0$ be regularization parameters, and define the point-wise regularized loss $\ellr(x;w) = \ell(x;w) + R(w)$ and the population regularized loss $\ellr(w) = \ell(w) + R(w)$, where
\begin{align*}
R(w) = \frac{1}{2}\sum_{\substack{e = (u,v) \in E \\ u \in \Vin}} \lambda_1 a_e^2 + \frac{1}{2}\sum_{\substack{e = (u,v) \in E \\ u \not\in \Vin}} \lambda_2 a_e^2.
\end{align*}
The distinct regularization parameters $\lambda_1, \lambda_2 > 0$ for the weights of edges from the input and previous-layer neurons, respectively, are for purely technical reasons and are explained in \cref{sec:overview}.

\subsection{Training}

We train the neural network to learn a function $g : \{-1,1\}^n \to \RR$ by running \cref{alg:forward}. This algorithm trains layer-wise from layer $1$ to layer $L$. The $i$th layer is trained with stochastic block coordinate descent, iterating through the neurons in $V_i$ in an arbitrary fixed order, and training the parameters of each neuron $v \in V_i$ using the $\TrainNeuron$ subroutine. Each call of $\TrainNeuron$ runs stochastic gradient descent to train the subset  of neural network parameters $w_{v} = \{a_e\}_{e = (u,v) \in E} \cup \{b_{v}\}$ directly associated with neuron $v$ (i.e., the weights of the edges that go into $v$, and the bias of $v$), keeping the other parameters $w_{-v} = \{a_e\}_{e = (u',v') \in E \mbox{ s.t. } v \neq v'} \cup \{b_{v'}\}_{v' \in V \sm (\{v\} \cup \Vin)}$ fixed.

\begin{algorithm}\SetAlgoLined\DontPrintSemicolon
    \KwIn{
    Sample access to the distribution $\{(x,g(x))\}_{x \sim \{-1,1\}^n}$.  Hyperparameters $W, L, p_1, p_2, \lambda_1, \lambda_2, \eta, B, \estop, \alpha, \tau$.
    }
    \KwOut{Trained parameters of neural network after training layer-wise.}
    
    \BlankLine
    $(V,E) \gets$ random network constructed as in \cref{ssec:architecture}.
    
    $w^0 \gets \vec{0}$,    
    $t \gets 0$ \tcp*{    Initialize all weights and biases to zero.}
    
    \For{layer $i = 1$ to $L$}{
    \For{neuron $v \in V_i$}{
        
        \tcp{Train the neuron parameters, $w_v$, fixing other parameters}
        
        $w^{t+1} \gets \TrainNeuron(v,w^t; \lambda_1,\lambda_2,\eta,B,\estop,\alpha,\tau)$ \label{step:trainneuron}

        $t \gets t + 1$
    }
    }
    Return $w^t$
    \caption{$\TrainNetworkLayerwise$}
\label{alg:forward}
\end{algorithm}

\begin{algorithm}\SetAlgoLined\DontPrintSemicolon
    \KwIn{Neuron $v \in V \sm \Vin$. Initial network parameters $w^{0}$.
    Access to random samples $(x,g(x))$ for $x \sim \{-1,1\}^n$.  Hyperparameters $\lambda_1, \lambda_2, \eta, B, \estop, \alpha, \tau$.
    }
    \KwOut{Parameters of network after the subset of parameters $w_v$ of neuron $v$ is perturbed and then trained with $\NeuronSGD$, while all other parameters $w_{-v}$ remain fixed.}
    
    \BlankLine
    
    \tcp{To avoid saddle points, randomly perturb the neuron parameters}
    
    $\wvperturb \gets w^{0}_v + z$, where $z$ is a noise vector whose entries are i.i.d. in  $\mathrm{Unif}([-\eta,\eta])$.
    
    $\wperturb \gets [w_{-v}^0, \wvperturb]$ \label{step:perturb}
    \vspace{0.5em}

    \tcp{Run stochastic gradient descent on neuron parameters, until approximate stationarity} \label{step:traintostationary}
    $\wsgd \gets \NeuronSGD(v,\wperturb; \lambda_1,\lambda_2,B,\estop,\alpha)$  \label{step:callneuronsgd}
    
    \vspace{0.5em}
    \tcp{Prune the neuron's small weights}
    $\wvret \gets \wvsgd$, rounding to $0$ every entry of magnitude less than $\tau$ \label{step:truncate}

    Return $\wret = [w_{-v}^0, \wvret]$
    \caption{$\TrainNeuron(v,w^0;\lambda_1,\lambda_2,\eta,B,\estop,\alpha,\tau)$}
\label{alg:trainneuron}
\end{algorithm}

\begin{algorithm}
    \KwIn{Neuron $v \in V \sm \Vin$. Initial network parameters $w^0$. Access to random samples $(x,g(x))$ for $x \sim \{-1,1\}^n$. Hyperparameters $\lambda_1,\lambda_2,\eta,B,\estop,\alpha,\tau$.
    }
    \KwOut{Parameters of network after the subset of parameters $w_v$ corresponding to neuron $v$ is trained, all other parameters $w_{-v}$ remain fixed.}
    $t \gets 0$
    
    \While{true}{\label{step:startloop}
    
    \tcp{Approximate  $\nabla_{w_v} \ellr$ with  minibatch size $B$}
    
    Draw i.i.d. data samples $(x^{t,1},g(x^{t,1})),\ldots,(x^{t,B},g(x^{t,B}))$ \label{step:drawbatch}
    
     $\xi^t \gets \frac{1}{B} \sum_{i=1}^B \nabla_{w_v} \ellr(x^{t,i}; w_{-v}^0,w^{t}_v)$ \label{step:appxgrad}
     
     \tcp{Stop if we have reached an approximate stationary point}
    \lIf{$\|\xi^t\| \leq \estop$}{ break out of the loop}
    
    \tcp{Update $w_v$ in direction of the approximate gradient}
    $w_v^{t+1} \gets w_v^{t} - \alpha \xi^t$ 
    
    $t \gets t + 1$
    }
    \label{step:endloop}
    Return $[w_{-v}^0, w_v^t]$
    \caption{$\NeuronSGD(v,w^0;\lambda_1,\lambda_2,\eta,B,\estop,\alpha,\tau)$}
\label{alg:neuronsgd}
\end{algorithm}

\subsection{Theoretical result}
\label{sec:thmstatement}
We prove that \cref{alg:forward} learns functions satisfying the staircase property in the sense of \cref{def:staircaseprop}. We defer the exact bounds on the parameters considered to \cref{sec:formal_theorem}.
\begin{theorem}\label{thm:forward}
Let $g : \{-1,1\}^n \to \RR$ be an unknown $s$-sparse polynomial satisfying the $[1/M,M]$-staircase property for some given $s,M >1$. Given an accuracy parameter $\ewant > 0$, a soundness parameter $0 < \delta < 1$, and access to random samples from $\{(x,g(x))\}_{x \sim \{-1,1\}^n}$, there is a setting of hyperparameters for \cref{alg:forward} that is polynomially-bounded, i.e.,
$$1 / \poly(n,s,M,1/\ewant,1/\delta) \leq W ,L, p_1, p_2, \lambda_1, \lambda_2, \eta, B, \estop, \alpha, \tau \leq \poly(n,s,M,1/\ewant,1/\delta),$$
such that \cref{alg:forward} runs in $\poly(n,s,M,1/\ewant,1/\delta)$ time and samples and with probability $\geq 1 - \delta$
returns trained weights $w$ satisfying the bound $\ell(w) \leq \ewant$ on the population loss.
\end{theorem}

\subsection{Proof overview}
\label{sec:proofoverview}
\label{sec:overview}
We now briefly describe how \cref{alg:forward} learns, giving a high-level depiction of the training process in the case that the target function is the staircase function $S_3(x) = x_1 + x_1 x_2 + x_1 x_2 x_3$. We refer to \cref{fig:training_illu} for an illustration of the training procedure, where grey neurons are `blank' (i.e., have identically zero output) and the green neurons are `active' (i.e., compute a non-zero function). Initially all neurons are blank and the total output of the network is $0$.
\begin{figure}[h!]
    \centering
    \includegraphics[width = \linewidth]{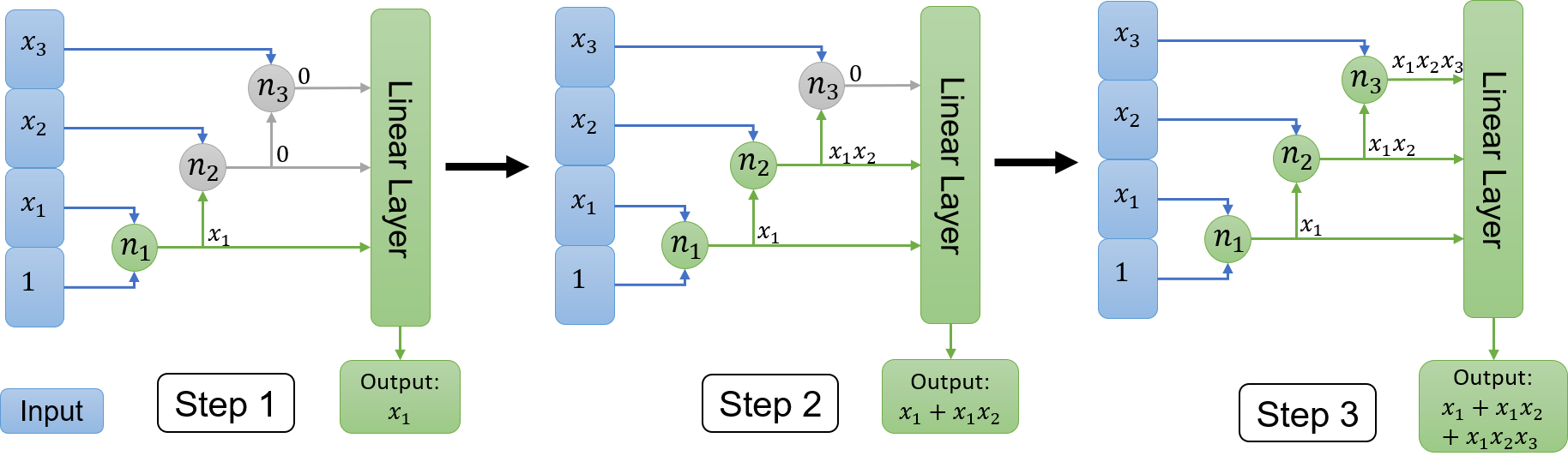}
    \caption{An illustration of the training procedure for learning $S_3(x) = x_1 + x_1x_2 + x_1 x_2 x_3 $. The grey neurons are `blank' and the green neurons are `active'.}
    \label{fig:training_illu}
\end{figure}

We set the random network topology connectivity hyperparameters $p_1$ and $p_2$ to be small, so that the network is sparse. We can show that the following invariant is maintained throughout training: any neuron has at most two active parents. Intuitively, this is because we can bound the number of active neurons at any iteration during training by $s+n+1$, so the number of neuron tuples $(u_1,u_2,u_3,v) \in V^4$ such that $u_1,u_2,u_3$ are active and all have edges to $v$ is in expectation bounded by $(s+n+1)^3(p_1)^3W \ll 1$. Since any neuron during training has at most two active parents, we may tractably analyze $\TrainNeuron$ for training new neurons: in a key technical lemma, we show that every active neuron $v$ has exactly two active parents $u$ and $u'$, and approximately computes a monomial given by the product of the parents' values, $f_{v}(x) \approx \chi_{S_v}\approx f_u(x) f_{u'}(x)$.

We cannot set $p_1$ and $p_2$ to be too small or else the network will not be connected enough to learn. Thus, we must also set the connectivity parameters so that for any pair $(u,u') \in \Vin \times (V \sm V_L)$, the neurons $u$ and $u'$ share many children, and at least one of these children may learn the product if it is useful. For this it is sufficient to take the expected number of shared children $p_1p_2 W \gg 1$ very large. We now present a run of the algorithm, breaking it up into ``steps'' for exposition.

\textbf{Step 1:}\quad
The algorithm iterates over neurons and trains them one by one using $\TrainNeuron$. Most of the neurons trained are left blank: for example, if a neuron $v$ has the two inputs $x_2$ and $1$, then by our key technical lemma the neuron could either remain blank or learn the product of the inputs, $x_2 = x_2 \cdot 1$. But the mean-squared error cannot decrease by learning $x_2$, since $x_2$ is orthogonal to the staircase function in the $L_2$ sense (i.e., $\langle S_3(x), x_2 \rangle = 0$, because the staircase function does not have $x_2$ as a monomial), so the neuron $v$ remains blank. Let $t_1$ be the first iteration at which the algorithm reaches a neuron $n_1 \in V_1$ that has $x_1$ and $1$ as inputs.
When the network trains $n_1$ using the  sub-routine, we show that it learns to output $x_1 = x_1 \cdot 1$, since that is the highest-correlated function to $S_3(x)$ that $n_1$ can output. Combined with the linear layer, the overall neural network output becomes $f(x;w^{t_1}) \approx x_1$.

\textbf{Step 2:}\quad
The error function after Step 1 is $E_1(x) = S_3(x) - f(x;w^{t_1}) \approx x_1 x_2 + x_1x_2x_3$. Again, for many iterations the training procedure keeps neurons blank, until at iteration $t_2$ it reaches a neuron $n_2$ with inputs $x_1$ (due to neuron $n_1$) and $x_2$ (directly from the input). Similarly to Step 1, when we train  $n_2$, we show that it learns to output $x_1x_2$, which is the function with highest correlation to $E_1(x)$ which $n_2$ can output. Thus, the neural network now learns to output $f(x;w^{t_2}) \approx x_1 + x_1 x_2$, so the error function has decreased to $E_2(x) = S_3(x) - f(x;w^{t_2}) \approx x_1x_2x_3$. The training proceeds in this manner until all the monomials in $S_3(x)$ are learned by the network.

\textbf{Error Propagation and Regularization:}\quad
A significant obstacle in analyzing layer-wise training is that outputs of neurons are inherently noisy because of incomplete training, and the error may grow exponentially along the depth of the network. In order to avoid this issue, we have two distinct regularization parameters $\lambda_1,\lambda_2$ and connectivity parameters $p_1,p_2$ for edges from inputs versus edges from neurons. In our proof of \cref{thm:forward}, we set $\lambda_1 \ll \lambda_2$ and $p_1 \gg p_2$, which ensures that after training a neuron (say $n_2$ above) the weight from the neuron $n_1$ (which has regularization parameter $\lambda_2$) is much smaller than the weight directly from the input $x_2$ (which has regularization parameter $\lambda_1$). Since the inputs are noise-free, this disallows exponential growth of errors along the depth. We conjecture that if the network is trained end-to-end instead of layer-wise, then one can avoid this technical difficulty and set $\lambda_1 = \lambda_2$ and $p_1 = p_2$, because of a backward feature correction phenomenon \cite{allen2020backward} where the lower layers' accuracy improves as higher levels are trained.

\section{General Hierarchical Structure}
\label{sec:main_def}
\label{ssec:gen_def}

\subsection{Extension to biased binary inputs and implications}\label{rem:dectrees}\label{ssec:biasedhypercube}
We extend the main result of this paper to more general setting of functions over a space of i.i.d. binary variables that have zero expectation but are not necessarily supported on $\{+1,-1\}$. For instance, if $\{X_i\}_{i \in [n]}$ are i.i.d.\ and Boolean (on $\{+1,-1\}$) with $\mathbb{E}(X_i)=b$ for some $b \in [-1,1]$, the centered variables $\tilde{X}_i=X_i-b$ are valued in $\{1-b,-1-b\}$. Over these centered variables, the Fourier coefficients of a function are given (up to normalization) by $\hat{f}(S) = \EE_{\tilde{X}} f(\tilde{X}) \prod_{i \in S} \tilde{X}_i$ for any $S \subset [n]$. Thus, the staircase property of \cref{def:staircaseprop} generalizes clearly: the function $f : \{1-b,-1-b\}^n \to \RR$ satisfies the staircase property if for any $S \subset [n]$ with $|S| \geq 2$ such that $\hat{f}(S) \neq 0$, there is a subset $S' \subset S$ such that $|S \sm S'| = 1$ and $\hat{f}(S') \neq 0$.

Showing a similar result to \cref{thm:forward} for staircase functions on the variables $\tilde{X}_i$ with a quadratic activation requires a slight modification of our argument since $\tilde{X}_i^2$ is no longer constant (and equal to 1), so one cannot use the simple identity $Z_1 Z_2=(Z_1 +Z_2)^2/2-1$ that holds for variables valued in $\{+1,-1\}$ to prove that a neuron learns the product of its inputs when trained. However, adding skip connections from the previous layer with quadratic activation, along with the fact that $r_1r_2 = ((r_1+r_2)^2 - r_1^2 - r_2^2)/2$, one can hierarchically learn new features as products of previously-learned features. Alternatively, one can change the activation so that each neuron maps a vector input $v$ to $(a \cdot v + b)^2 + c \cdot y^{.2}$, where $a,b,c$ are trainable parameters, to learn products of general binary variables. A similar proof to that of \cref{thm:forward} is then expected to hold, implying that one can learn staircase functions over i.i.d.\ random variables that are binary and centered (beyond $\{+1,-1\}$ specifically). We will now discuss two interesting examples that fall under this setting.

\textbf{Biased sparse parities and kernel separation
}\quad Consider the problem of learning sparse biased parities, i.e., the class of monomials of degree $\log(n)$ with a $\{+1,-1\}$-valued input distribution that is i.i.d.\ with $\mathbb{E}X_1=b=1/2$. It is shown in \cite{malach2021quantifying} that such a distribution class is not learnable by any kernel method with poly-many features, while it is learnable by gradient-based learning on neural networks of polynomial size. The result of \cite{malach2021quantifying} relies on an architecture that allows emulating an SQ algorithm -- far from a regular network as considered in this paper. However, sparse biased parities are staircase functions over unbiased binary variables, with polynomially-many nonzero coefficients since the degree is logarithmic. So an extension of \cref{thm:forward} to arbitrary binary centered variables would imply that regular networks can learn sparse biased parities, implying a separation between kernel-based learning and gradient-based learning on regular networks.

\textbf{Decision trees under smoothed complexity model}\quad
Secondly, in a smoothed complexity setting, where the input distribution is drawn from the biased binary hypercube, and the bias of each variable is randomly chosen, the class of $\log(n)$-depth decision trees satisfies the general staircase property. This is because Lemma 3 in \cite{kalai2009learning} implies that with high probability there is a $\poly(n/\eps)$-sparse polynomial that $\eps$-approximates the decision tree and satisfies the staircase property over the biased binary hypercube. Thus, the extension of our result to the case of biased binary inputs would imply that regular neural networks learn decision trees in the smoothed complexity model.\footnote{Such a conjecture was recently made by \cite{brutzkus2020id3}, which leaves as an open problem in Section 1.3 whether neural networks can learn $\log(n)$-juntas in a smoothed complexity setting, and implicitly poses the same problem about the more general case of $\log(n)$-depth decision trees.}

\subsection{Extension to more general $L_2$ spaces}\label{ssec:moregeneral}

We now give an even more general version of the staircase property in \cref{def:staircaseprop}. Since neural networks are efficient at representing affine transforms of the data and smooth low-dimensional functions \cite{bresler2020corrective}, we generalize the class of hierarchical functions over the space of continuous, real valued functions on $[-R,R]^n \subseteq \mathbb{R}^n ;  R \in \mathbb{R}^{+}\cup \{\infty\}$ without any reference to underlying measures but with enough flexibility to add additional structures like measures and the corresponding $L^2$ norms. Set $R \in \mathbb{R}^{+}\cup \{\infty\}$ and consider any sequence of functions $\mathcal{H} := \{h_k : \RR \to \RR\}_{ k\in \mathbb{N}\cup \{0\}}$ such that $h_0$ is the constant function $1$, and any affine transform $\mathcal{A} : \mathbb{R}^n \to \mathbb{R}^n$ such that $\mathcal{A}(x) = Ax +b$ for $A \in \mathbb{R}^{n\times n}; b\in \mathbb{R}^n$.
We call a function $f:[-R,R]^d \to \mathbb{R}$ to be $(\mathcal{H},\mathcal{A})$-polynomial if there exists a finite index set $I_f \subset \left(\mathbb{N}\cup \{0\}\right)^n$ such that for some real numbers $(\alpha_{\mathbf{k}})_{\mathbf{k}\in I_f}$:  $$f(x) = \textstyle\sum_{\mathbf{k}:= (k_1,\dots,k_n)\in I_f} \alpha_{\mathbf{k}}\prod_{i=1}^{n}h_{k_i}(y_i) $$

 Where $y:= \mathcal{A}(x)$. We also define $\mathsf{Ord}(\mathbf{k}) :=  |\{i: k_i \neq 0\}|$ and a partial order `$\preceq $'  over $\left(\mathbb{N}\cup \{0\}\right)^n$ such that $\mathbf{k}^{\prime} \preceq \mathbf{k}$ iff $k^{\prime}_i \in \{0,k_i\} $ for every $i \in [n]$.
For $M \geq 1$, we will call  a $(\mathcal{H},\mathcal{A})$-polynomial to be $(1/M,M)$ hierarchical if 
\begin{enumerate}
    \item  $1/M \leq |\alpha_{\mathbf{k}}| \leq M $ for every $\mathbf{k} \in I_f$.
    \item For every $\mathbf{k} \in I_f$ such that $\mathsf{Ord}(\mathbf{k}) \geq 2$, there exists $\mathbf{k}^{\prime} \in I_f$ such that $\mathsf{Ord}(\mathbf{k}^{\prime}) = \mathsf{Ord}(\mathbf{k}) -1$ and $\mathbf{k}^{\prime} \preceq \mathbf{k}$.
\end{enumerate}
We now extend the definition to general continuous functions. Suppose $d_{\mathcal{S}}$ is a pseudo-metric on the space of bounded continuous functions $\mathcal{C}^b([-R,R]^n;\mathbb{R})$. We call $f \in \mathcal{C}^b([-R,R]^n;\mathbb{R}) $ to be $(1/M,M, \mathcal{S})$ hierarchical if for every $\epsilon > 0$, there exists a $(1/M,M)$ hierarchical $(\mathcal{H},\mathcal{A})$-polynomial $f_{\epsilon}$ such that:
$d_{\mathcal{S}}(f,f_{\epsilon}) < \epsilon $. We note some examples below:
\begin{enumerate}
    \item Let $\mu$ be the uniform measure over $\{-1,1\}^n$, $d_\mathcal{S} $ be the $L^2$ norm induced by $\mu$,
    $\mathcal{H} = \{1,x\}$ and $\mathcal{A}$ be identity mapping. %
    We note that functions over the unbiased Boolean hypercube satisfying the $[1/M,M]$-staircase property in \cref{def:staircaseprop} correspond to $(1/M,M,L^2(\mu))$ hierarchical functions as defined above.
    \item In the case when $\mu$ is the biased product measure over $\{-1,1\}^n$, we can take $\mathcal{H} = \{1,x\}$ and $\mathcal{A}(x) = x - \mathbb{E}x$. This recovers the definition in \cref{ssec:biasedhypercube}.
    \item When $\mu$ is the isotropic Gaussian measure, we can take $R = \infty$, $\mathcal{H}$ to be the set of 1-D Hermite polynomials and $\mathcal{A}$ to be the identity. In case $\mu = \mathcal{N}(m,\Sigma)$, we instead take $\mathcal{A} = \Sigma^{-1/2}(x - m)$.
    \item When $R < \infty$ and $\mathcal{S}$ is the $L^2$ norm with respect to the Lebesgue measure. $\mathcal{C}([-R,R]^n;\mathbb{R})$, we can take $\mathcal{H} = \{\exp(i\frac{\pi k x}{R}): k\in \mathbb{Z}\}$ and $\mathcal{A} = I$. This allows us to interpret $(\mathcal{H},\mathcal{A})$-polynomial approximations as Fourier series approximations. 
    \item When $R < \infty$ and $\mathcal{S}$ is the uniform norm (or sup norm) over $\mathcal{C}([-R,R]^n;\mathbb{R})$, we can take $\mathcal{H} = \{1,x,x^2,\dots,\}$ and $\mathcal{A} = I$. Since any continuous function can be approximated by a polynomial, this presents a large class of functions of interest. 
\end{enumerate}
In items 1-4, we consider these specific function classes $\mathcal{H}$ in order to make $(\mathcal{H},\mathcal{A})$ monomials orthonormal under $L^2(\mu)$. We leave it as a direction of future work to extend our theoretical learning results in \cref{thm:forward} to such function classes.

\subsection{Composable chains}
Finally, we discuss a distinct way to generalize the staircase property.
One can relax the strict inclusion property of \cref{def:staircaseprop} with a single element removed, to more general notions of increasing chains. For instance, if $\hat{g}(S) \neq 0$, one may require that there exists an $S'$ such that $|S'| < |S|$ and $|S \Delta S'| = O(1)$, and we conjecture that regular networks will still learn sparse polynomials with this structure. More generally, one may require that for any $S$ such that $\hat{g}(S) \neq 0$, there exists a constant number of $S_j$'s such that $|S_j|<|S|$ and $\hat{g}(S_j) \neq 0$ for all $j$, and such that $S$ can be composed by $\{S_j\}$ and the input features $x_1,\dots,x_n$, where in the Boolean setting the composition rule corresponds to products. Finally, one could further generalize the results by changing the feature space, i.e., using regular networks that take not just the standard inputs $x_1,\dots,x_n$, but also have other choices of features $\phi_1(x^n), \dots, \phi_p(x^n)$ as inputs, for $p$ polynomial in $n$.

\section{Limitations and societal impacts}\label{sec:limitations1}
For simplicity of the proofs, the architecture and training algorithm are not common in practice: quadratic activations, a sparse connectivity graph, and layer-wise training \cite{belilovsky2019greedy,nokland2019training,belilovsky2020decoupled} with stochastic block coordinate descent. We also perturb the weights with noise in order to avoid saddle points \cite{jin2017escape}, and we prune the low-magnitude weights to simplify the analysis (although this may not deteriorate performance much in practice \cite{han2015learning}). We emphasize that these limitations are purely technical as they make the analysis tractable, and we conjecture from our experiments that $\mathsf{ReLU}$ ResNets trained with SGD efficiently learn functions satisfying the staircase property. This work does not deal directly with real world data, so may not have direct societal impacts. However, it aims to rigorously understand and interpret deep learning, which may aid us in preventing unfair behavior by AI.

\subsubsection*{Acknowledgments}

We are grateful to Philippe Rigollet for insightful conversations. EA was supported in part by the NSF-Simons Sponsored Collaboration on the Theoretical Foundations of Deep
Learning, NSF award 2031883 and Simons Foundation award 814639. EB was supported in part by an NSF
Graduate Fellowship and an Apple Fellowship and NSF grant DMS-2022448. GB was supported in part by NSF CAREER award CCF-1940205. DN was supported in part by NSF grant DMS-2022448.

\bibliographystyle{unsrtnat}
\bibliography{bibliography}

\begin{thebibliography}{44}
\providecommand{\natexlab}[1]{#1}
\providecommand{\url}[1]{\texttt{#1}}
\expandafter\ifx\csname urlstyle\endcsname\relax
  \providecommand{\doi}[1]{doi: #1}\else
  \providecommand{\doi}{doi: \begingroup \urlstyle{rm}\Url}\fi

\bibitem[Olah et~al.(2020)Olah, Cammarata, Schubert, Goh, Petrov, and
  Carter]{olah2020zoom}
Chris Olah, Nick Cammarata, Ludwig Schubert, Gabriel Goh, Michael Petrov, and
  Shan Carter.
\newblock Zoom in: An introduction to circuits.
\newblock \emph{Distill}, 5\penalty0 (3):\penalty0 e00024--001, 2020.

\bibitem[Zeiler and Fergus(2014)]{zeiler2014visualizing}
Matthew~D Zeiler and Rob Fergus.
\newblock Visualizing and understanding convolutional networks.
\newblock In \emph{European conference on computer vision}, pages 818--833.
  Springer, 2014.

\bibitem[Abbe and Sandon(2020)]{AS20}
Emmanuel Abbe and Colin Sandon.
\newblock On the universality of deep learning.
\newblock In H.~Larochelle, M.~Ranzato, R.~Hadsell, M.~F. Balcan, and H.~Lin,
  editors, \emph{Advances in Neural Information Processing Systems}, volume~33,
  pages 20061--20072. Curran Associates, Inc., 2020.
\newblock URL
  \url{https://proceedings.neurips.cc/paper/2020/file/e7e8f8e5982b3298c8addedf6811d500-Paper.pdf}.

\bibitem[Abbe et~al.(2021)Abbe, Kamath, Malach, Sandon, and Srebro]{newAS}
Emmanuel Abbe, Pritish Kamath, Eran Malach, Colin Sandon, and Nathan Srebro.
\newblock On the power of differentiable learning versus pac and sq learning.
\newblock \emph{Preprint.}, 2021.

\bibitem[Kearns(1998)]{kearns1998efficient}
Michael Kearns.
\newblock Efficient noise-tolerant learning from statistical queries.
\newblock \emph{Journal of the ACM (JACM)}, 45\penalty0 (6):\penalty0
  983--1006, 1998.

\bibitem[Blum et~al.(1994)Blum, Furst, Jackson, Kearns, Mansour, and
  Rudich]{blum1994weakly}
Avrim Blum, Merrick Furst, Jeffrey Jackson, Michael Kearns, Yishay Mansour, and
  Steven Rudich.
\newblock Weakly learning dnf and characterizing statistical query learning
  using fourier analysis.
\newblock In \emph{Proceedings of the twenty-sixth annual ACM symposium on
  Theory of computing}, pages 253--262, 1994.

\bibitem[Shalev-Shwartz et~al.(2017)Shalev-Shwartz, Shamir, and
  Shammah]{shalev2017failures}
Shai Shalev-Shwartz, Ohad Shamir, and Shaked Shammah.
\newblock Failures of gradient-based deep learning.
\newblock In \emph{International Conference on Machine Learning}, pages
  3067--3075. PMLR, 2017.

\bibitem[Kulkarni et~al.(2016)Kulkarni, Narasimhan, Saeedi, and
  Tenenbaum]{kulkarni2016hierarchical}
Tejas~D Kulkarni, Karthik Narasimhan, Ardavan Saeedi, and Josh Tenenbaum.
\newblock Hierarchical deep reinforcement learning: Integrating temporal
  abstraction and intrinsic motivation.
\newblock In \emph{NIPS}, 2016.

\bibitem[Woltman et~al.(2012)Woltman, Feldstain, MacKay, and
  Rocchi]{woltman2012introduction}
Heather Woltman, Andrea Feldstain, J~Christine MacKay, and Meredith Rocchi.
\newblock An introduction to hierarchical linear modeling.
\newblock \emph{Tutorials in quantitative methods for psychology}, 8\penalty0
  (1):\penalty0 52--69, 2012.

\bibitem[Friedman et~al.(1997)Friedman, Geiger, and
  Goldszmidt]{friedman1997bayesian}
Nir Friedman, Dan Geiger, and Moises Goldszmidt.
\newblock Bayesian network classifiers.
\newblock \emph{Machine learning}, 29\penalty0 (2):\penalty0 131--163, 1997.

\bibitem[Rokach and Maimon(2005)]{rokach2005clustering}
Lior Rokach and Oded Maimon.
\newblock Clustering methods.
\newblock In \emph{Data mining and knowledge discovery handbook}, pages
  321--352. Springer, 2005.

\bibitem[Bruna and Mallat(2013)]{bruna2013invariant}
Joan Bruna and St{\'e}phane Mallat.
\newblock Invariant scattering convolution networks.
\newblock \emph{IEEE transactions on pattern analysis and machine
  intelligence}, 35\penalty0 (8):\penalty0 1872--1886, 2013.

\bibitem[Ye et~al.(2018)Ye, Han, and Cha]{ye2018deep}
Jong~Chul Ye, Yoseob Han, and Eunju Cha.
\newblock Deep convolutional framelets: A general deep learning framework for
  inverse problems.
\newblock \emph{SIAM Journal on Imaging Sciences}, 11\penalty0 (2):\penalty0
  991--1048, 2018.

\bibitem[Patel et~al.(2015)Patel, Nguyen, and Baraniuk]{patel2015probabilistic}
Ankit~B Patel, Tan Nguyen, and Richard~G Baraniuk.
\newblock A probabilistic theory of deep learning.
\newblock \emph{arXiv preprint arXiv:1504.00641}, 2015.

\bibitem[Malach and Shalev-Shwartz(2018)]{malach2018provably}
Eran Malach and Shai Shalev-Shwartz.
\newblock A provably correct algorithm for deep learning that actually works.
\newblock \emph{arXiv preprint arXiv:1803.09522}, 2018.

\bibitem[Mossel(2016)]{mossel2016deep}
Elchanan Mossel.
\newblock Deep learning and hierarchal generative models.
\newblock \emph{arXiv preprint arXiv:1612.09057}, 2016.

\bibitem[Basri et~al.(2019)Basri, Jacobs, Kasten, and
  Kritchman]{basri2019convergence}
Ronen Basri, David Jacobs, Yoni Kasten, and Shira Kritchman.
\newblock The convergence rate of neural networks for learned functions of
  different frequencies.
\newblock \emph{arXiv preprint arXiv:1906.00425}, 2019.

\bibitem[Malach and Shalev-Shwartz(2020)]{malach2020implications}
Eran Malach and Shai Shalev-Shwartz.
\newblock The implications of local correlation on learning some deep
  functions.
\newblock \emph{Advances in Neural Information Processing Systems}, 33, 2020.

\bibitem[Mhaskar et~al.(2016)Mhaskar, Liao, and Poggio]{mhaskar2016learning}
Hrushikesh Mhaskar, Qianli Liao, and Tomaso Poggio.
\newblock Learning functions: when is deep better than shallow.
\newblock \emph{arXiv preprint arXiv:1603.00988}, 2016.

\bibitem[Eldan and Shamir(2016)]{eldan2016power}
Ronen Eldan and Ohad Shamir.
\newblock The power of depth for feedforward neural networks.
\newblock In \emph{Conference on learning theory}, pages 907--940. PMLR, 2016.

\bibitem[Malach and Shalev-Shwartz(2019)]{malach2019deeper}
Eran Malach and Shai Shalev-Shwartz.
\newblock Is deeper better only when shallow is good?
\newblock \emph{arXiv preprint arXiv:1903.03488}, 2019.

\bibitem[Telgarsky(2016)]{telgarsky2016benefits}
Matus Telgarsky.
\newblock Benefits of depth in neural networks.
\newblock In \emph{Conference on learning theory}, pages 1517--1539. PMLR,
  2016.

\bibitem[Bresler and Nagaraj(2020{\natexlab{a}})]{bresler2020sharp}
Guy Bresler and Dheeraj Nagaraj.
\newblock Sharp representation theorems for relu networks with precise
  dependence on depth.
\newblock \emph{Advances in Neural Information Processing Systems}, 33,
  2020{\natexlab{a}}.

\bibitem[Chen et~al.(2020)Chen, Bai, Lee, Zhao, Wang, Xiong, and
  Socher]{chen2020towards}
Minshuo Chen, Yu~Bai, Jason~D Lee, Tuo Zhao, Huan Wang, Caiming Xiong, and
  Richard Socher.
\newblock Towards understanding hierarchical learning: Benefits of neural
  representations.
\newblock \emph{arXiv preprint arXiv:2006.13436}, 2020.

\bibitem[Allen-Zhu and Li(2019)]{allen2019can}
Zeyuan Allen-Zhu and Yuanzhi Li.
\newblock What can resnet learn efficiently, going beyond kernels?
\newblock \emph{arXiv preprint arXiv:1905.10337}, 2019.

\bibitem[Allen-Zhu and Li(2020)]{allen2020backward}
Zeyuan Allen-Zhu and Yuanzhi Li.
\newblock Backward feature correction: How deep learning performs deep
  learning.
\newblock \emph{arXiv preprint arXiv:2001.04413}, 2020.

\bibitem[Saxe et~al.(2013)Saxe, McClelland, and Ganguli]{saxe2013exact}
Andrew~M Saxe, James~L McClelland, and Surya Ganguli.
\newblock Exact solutions to the nonlinear dynamics of learning in deep linear
  neural networks.
\newblock \emph{arXiv preprint arXiv:1312.6120}, 2013.

\bibitem[Gidel et~al.(2019)Gidel, Bach, and Lacoste-Julien]{gidel2019implicit}
Gauthier Gidel, Francis Bach, and Simon Lacoste-Julien.
\newblock Implicit regularization of discrete gradient dynamics in linear
  neural networks.
\newblock In \emph{Advances in Neural Information Processing Systems},
  volume~32, 2019.
\newblock URL
  \url{https://proceedings.neurips.cc/paper/2019/file/f39ae9ff3a81f499230c4126e01f421b-Paper.pdf}.

\bibitem[Gissin et~al.(2019)Gissin, Shalev-Shwartz, and
  Daniely]{gissin2019implicit}
Daniel Gissin, Shai Shalev-Shwartz, and Amit Daniely.
\newblock The implicit bias of depth: How incremental learning drives
  generalization.
\newblock \emph{arXiv preprint arXiv:1909.12051}, 2019.

\bibitem[Kalimeris et~al.(2019)Kalimeris, Kaplun, Nakkiran, Edelman, Yang,
  Barak, and Zhang]{kalimeris2019sgd}
Dimitris Kalimeris, Gal Kaplun, Preetum Nakkiran, Benjamin~L Edelman, Tristan
  Yang, Boaz Barak, and Haofeng Zhang.
\newblock $\{$SGD$\}$ on neural networks learns functions of increasing
  complexity.
\newblock In \emph{Advances in Neural Information Processing Systems 32: Annual
  Conference on Neural Information Processing Systems 2019}, 2019.

\bibitem[Bietti and Bach(2020)]{bietti2020deep}
Alberto Bietti and Francis Bach.
\newblock Deep equals shallow for relu networks in kernel regimes.
\newblock \emph{arXiv preprint arXiv:2009.14397}, 2020.

\bibitem[Huang et~al.(2020)Huang, Wang, Tao, and Zhao]{huang2020deep}
Kaixuan Huang, Yuqing Wang, Molei Tao, and Tuo Zhao.
\newblock Why do deep residual networks generalize better than deep feedforward
  networks?---a neural tangent kernel perspective.
\newblock \emph{Advances in Neural Information Processing Systems}, 33, 2020.

\bibitem[Andoni et~al.(2014)Andoni, Panigrahy, Valiant, and
  Zhang]{andoni2014learning}
Alexandr Andoni, Rina Panigrahy, Gregory Valiant, and Li~Zhang.
\newblock Learning polynomials with neural networks.
\newblock In \emph{International conference on machine learning}, pages
  1908--1916. PMLR, 2014.

\bibitem[Yehudai and Shamir(2019)]{yehudai2019power}
Gilad Yehudai and Ohad Shamir.
\newblock On the power and limitations of random features for understanding
  neural networks.
\newblock \emph{arXiv preprint arXiv:1904.00687}, 2019.

\bibitem[Bresler and Nagaraj(2020{\natexlab{b}})]{bresler2020corrective}
Guy Bresler and Dheeraj Nagaraj.
\newblock A corrective view of neural networks: Representation, memorization
  and learning.
\newblock In \emph{Conference on Learning Theory}, pages 848--901. PMLR,
  2020{\natexlab{b}}.

\bibitem[Malach et~al.(2021)Malach, Kamath, Abbe, and
  Srebro]{malach2021quantifying}
Eran Malach, Pritish Kamath, Emmanuel Abbe, and Nathan Srebro.
\newblock Quantifying the benefit of using differentiable learning over tangent
  kernels.
\newblock \emph{arXiv preprint arXiv:2103.01210}, 2021.

\bibitem[Kalai et~al.(2009)Kalai, Samorodnitsky, and Teng]{kalai2009learning}
Adam~Tauman Kalai, Alex Samorodnitsky, and Shang-Hua Teng.
\newblock Learning and smoothed analysis.
\newblock In \emph{2009 50th Annual IEEE Symposium on Foundations of Computer
  Science}, pages 395--404. IEEE, 2009.

\bibitem[Brutzkus et~al.(2020)Brutzkus, Daniely, and Malach]{brutzkus2020id3}
Alon Brutzkus, Amit Daniely, and Eran Malach.
\newblock Id3 learns juntas for smoothed product distributions.
\newblock In \emph{Conference on Learning Theory}, pages 902--915. PMLR, 2020.

\bibitem[Belilovsky et~al.(2019)Belilovsky, Eickenberg, and
  Oyallon]{belilovsky2019greedy}
Eugene Belilovsky, Michael Eickenberg, and Edouard Oyallon.
\newblock Greedy layerwise learning can scale to imagenet.
\newblock In \emph{International conference on machine learning}, pages
  583--593. PMLR, 2019.

\bibitem[N{\o}kland and Eidnes(2019)]{nokland2019training}
Arild N{\o}kland and Lars~Hiller Eidnes.
\newblock Training neural networks with local error signals.
\newblock In \emph{International Conference on Machine Learning}, pages
  4839--4850. PMLR, 2019.

\bibitem[Belilovsky et~al.(2020)Belilovsky, Eickenberg, and
  Oyallon]{belilovsky2020decoupled}
Eugene Belilovsky, Michael Eickenberg, and Edouard Oyallon.
\newblock Decoupled greedy learning of cnns.
\newblock In \emph{International Conference on Machine Learning}, pages
  736--745. PMLR, 2020.

\bibitem[Jin et~al.(2017)Jin, Ge, Netrapalli, Kakade, and
  Jordan]{jin2017escape}
Chi Jin, Rong Ge, Praneeth Netrapalli, Sham~M Kakade, and Michael~I Jordan.
\newblock How to escape saddle points efficiently.
\newblock In \emph{International Conference on Machine Learning}, pages
  1724--1732. PMLR, 2017.

\bibitem[Han et~al.(2015)Han, Pool, Tran, and Dally]{han2015learning}
Song Han, Jeff Pool, John Tran, and William~J Dally.
\newblock Learning both weights and connections for efficient neural networks.
\newblock \emph{arXiv preprint arXiv:1506.02626}, 2015.

\bibitem[He et~al.(2016)He, Zhang, Ren, and Sun]{he2016deep}
Kaiming He, Xiangyu Zhang, Shaoqing Ren, and Jian Sun.
\newblock Deep residual learning for image recognition.
\newblock In \emph{Proceedings of the IEEE conference on computer vision and
  pattern recognition}, pages 770--778, 2016.

\end{thebibliography}

\clearpage
\appendix 
\newcommand{\kt}[1]{\kappa^{{#1}}}
\newcommand{\kc}[1]{\kt{c_{#1}}}
\newcommand{\obs}{{\color{red}[obsolete]}}
\newcommand{\lsim}{\lesssim}
\newcommand{\gsim}{\gtrsim}
\newcommand{\Egood}[1]{E_{good,#1}}

\section{Experiments}
\label{sec:additional_experiments}
The expository experiments given in Figure~\ref{fig:main_expt} compared the training of $S_{k}$ to the training of $\chi_{1:k}$ with each iteration of SGD drawing fresh i.i.d samples from the data. In this section, we fix the number of samples $m$ and cycle through it with some mini-batch size $B$ at each iteration. In order to maintain the comparison fair, we normalize $S_{k}$ in order for it to have the same $L_2$ norm as $\chi_{1:k}$ (whenever there is a comparison). For example, in the case of uniform measure over the hypercube, we replace $S_k$ with 
$S_k/\sqrt{k}$. We also conduct the experiments for various underlying distributions (such as Gaussians and biased product distributions on the Hypercube), the double staircase function and various choices of $n$ and $k$. When the underlying distribution, $n$ and $k$ are fixed, we will attempt to learn $S_k$ and $\chi_{1:k}$ with the same neural network along with the same parameters and hyper-parameters. We use the ReLU resnet architecture everywhere, with the same width across the layers(see \cite{he2016deep}). We train the network by minimizing the square loss via. SGD. The errors and Fourier coefficients plotted below are all computed with fresh samples (of size $3\times 10^4$).

In all the experiments below, we note that the functions satisfying the staircase property are learnt hierarchically - i.e, the network learns the simpler features first and then builds up to the complex features. However, the network is unable to learn just the complex features by themselves (like $\chi_{1:k}$) to any non-trivial accuracy.

\paragraph{Learning with Unbiased Parities:}
We consider the same parameters as in Figure~\ref{fig:main_expt}, but with a fixed number of samples and $S_k/\sqrt{k}$ instead of $S_k$ in order to normalize. We take $n=30$, $k =10$, number of samples $m = 6\times 10^4$, mini-batch size $B = 20$, depth $5$, width $40$. The results are plotted in Figure~\ref{fig:unbiased_bernoulli_expt}. The Fourier coefficient $\hat{f}_S$ for $S \subseteq [n]$ denotes $\mathbb{E}f(x)\chi_S(x)$ for $f$ being either $\chi_{1:k}$ or $S_k/\sqrt{k}$.
\begin{figure}[ht]%
\centering
	\fbox{\begin{minipage}[t][4.4cm][t]{\textwidth}
	\centering
	
	\adjustbox{valign=t}{
	\fbox{\begin{minipage}[t][3.4cm][t] {0.31\textwidth}
	\begin{subfigure}{\textwidth}
	\centering
	\includegraphics[height = 0.8\linewidth, width= \linewidth]{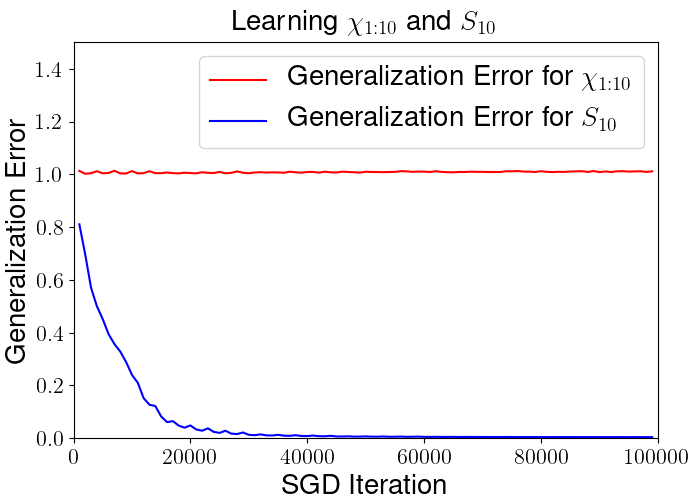}
	\caption{Loss Comparison for Parity and Staircase}
	\label{fig:loss_parity_stair_bern_ub}
	\end{subfigure}\end{minipage}}
	\centering
    
    \fbox{\begin{minipage}[t][3.4cm][t] {0.31\textwidth}
	\begin{subfigure}{\textwidth}
	\centering
	\includegraphics[height = 0.8\linewidth, width= \linewidth]{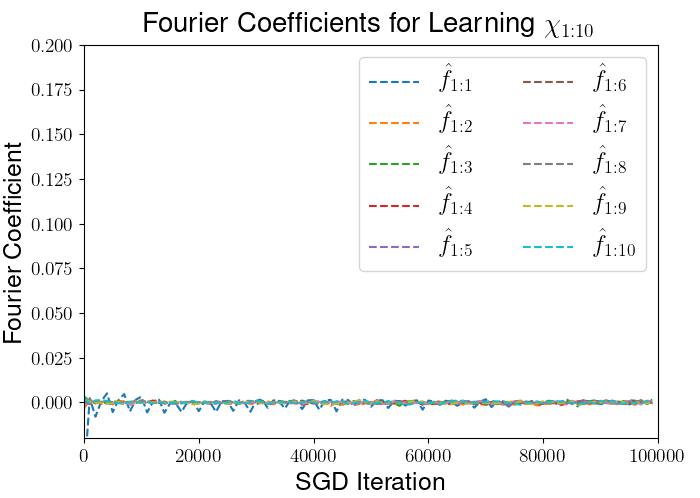}
	\caption{Fourier Coefficients for learning $\chi_{1:10}$}
	\label{fig:fourier_stair_bern_ub}
	\end{subfigure}\end{minipage}}
	
		\fbox{\begin{minipage}[t][3.4cm][t]{0.31\textwidth}
	\begin{subfigure}{\textwidth}
		\centering
		\includegraphics[height = 0.8\linewidth, width=\linewidth]{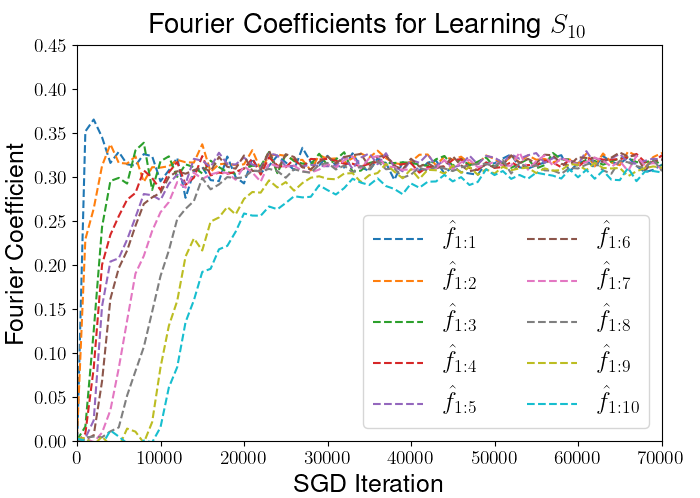}  
		  \caption{Fourier Coefficients for learning $S_{10}$}
    \label{fig:fourier_parity_bern_ub}
	\end{subfigure}\end{minipage}}}
	 \end{minipage}}
	 \caption{Learning Staircase and Parity functions with UnBiased Rademacher data.}
	 \label{fig:unbiased_bernoulli_expt}
\end{figure}

\paragraph{Learning with Gaussian Data:} We draw $x$ from the standard Gaussian distribution of $\mathbb{R}^n$ instead of the uniform measure over $\{-1,1\}^n$. This scenario is harder since monomials $\prod_{i=1}^{k}x_i$ can have heavy tails, will makes them occasionally take very large values. Hence, we take $k$ to be small and $n$ to be large. Instead of $S_k$, we consider $S_k/\sqrt{k}$ to ensure that its $L_2$ norm under the Gaussian measure is $1$. In figure~\ref{fig:gaussian_expt} we take $n=100$, $k =5$ number of samples $m = 3\times 10^5$, mini-batch size $B = 20$, depth $8$, width $50$. The Fourier coefficient $\hat{f}_S$ for $S \subseteq [n]$ denotes $\mathbb{E}f(x)\chi_S(x)$ for $f$ being either $\chi_{1:k}$ or $S_k/\sqrt{k}$.
\begin{figure}[ht]%
\centering
	\fbox{\begin{minipage}[t][4.4cm][t]{\textwidth}
	\centering
	
	\adjustbox{valign=t}{
	\fbox{\begin{minipage}[t][3.4cm][t] {0.31\textwidth}
	\begin{subfigure}{\textwidth}
	\centering
	\includegraphics[height = 0.8\linewidth, width= \linewidth]{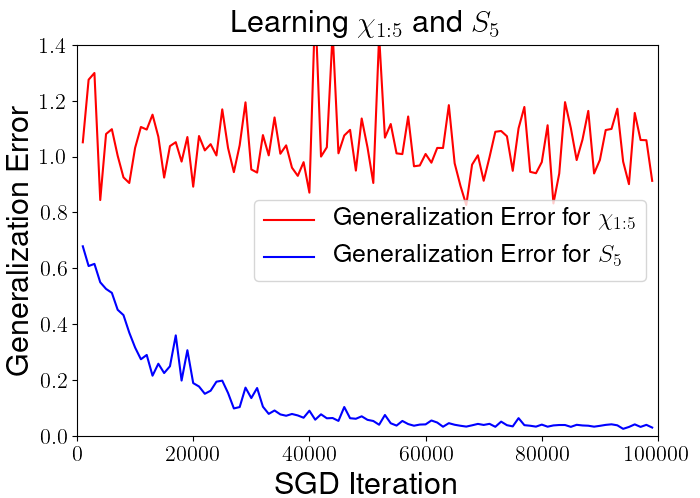}
	\caption{Loss Comparison for Parity and Staircase}
	\label{fig:loss_parity_stair_gauss}
	\end{subfigure}\end{minipage}}
	\centering
    
    \fbox{\begin{minipage}[t][3.4cm][t] {0.31\textwidth}
	\begin{subfigure}{\textwidth}
	\centering
	\includegraphics[height = 0.8\linewidth, width= \linewidth]{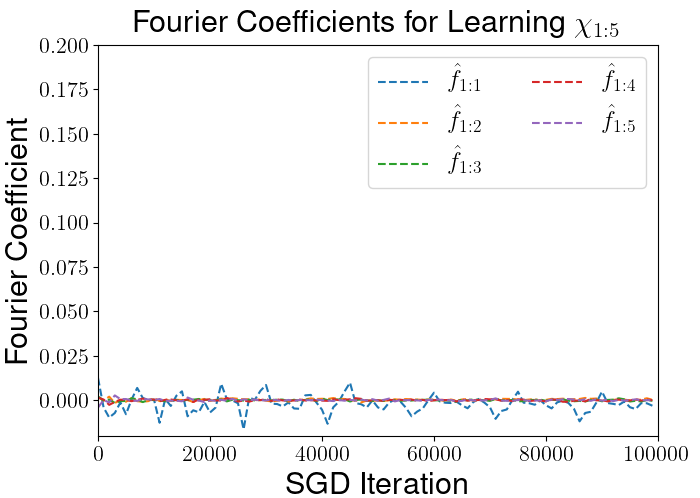}
	\caption{Fourier Coefficients for learning $\chi_{1:5}$}
	\label{fig:fourier_stair_gauss}
	\end{subfigure}\end{minipage}}
	
		\fbox{\begin{minipage}[t][3.4cm][t]{0.31\textwidth}
	\begin{subfigure}{\textwidth}
		\centering
		\includegraphics[height = 0.8\linewidth, width=\linewidth]{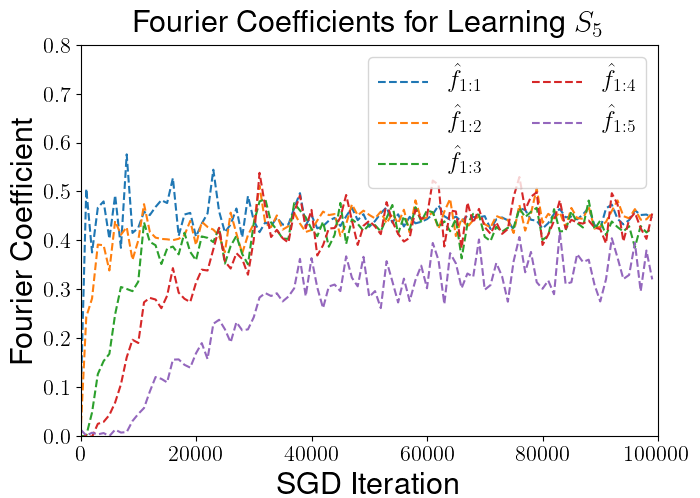}  
		  \caption{Fourier Coefficients for learning $S_5$}
    \label{fig:fourier_parity_gauss}
	\end{subfigure}\end{minipage}}}
	 \end{minipage}}
	 \caption{Learning Staircase and Parity functions with Gaussian data.}
	 \label{fig:gaussian_expt}
\end{figure}

\paragraph{Learning with Biased Parities:} In Figure~\ref{fig:biased_bernoulli_expt}, we consider the co-ordinates of $x$ to be drawn i.i.d from $\{-1,1\}$, but biased such that $\mathbb{P}(x_1 = 1) = p = 0.75$. In the definitions of $S_k$ and $\chi_{1:k}$, we replace $x_i$ with $\bar{x}_i := \frac{x_i - 2p + 1}{\sqrt{4p(1-p)}}$ and attempt to learn $S_k(\bar{x})/\sqrt{k}$ and $\chi_{1:k}(\bar{x})$. We take $n=30$, $k =7$, number of samples $m = 6\times 10^4$, mini-batch size $B = 20$, depth $5$, width $40$. The Fourier coefficient $\hat{f}_S$ for $S \subseteq [n]$ denotes $\mathbb{E}f(\bar{x})\chi_S(\bar{x})$ for $f$ being either $\chi_{1:k}$ or $S_k/\sqrt{k}$.
\begin{figure}[ht]%
\centering
	\fbox{\begin{minipage}[t][4.4cm][t]{\textwidth}
	\centering
	
	\adjustbox{valign=t}{
	\fbox{\begin{minipage}[t][3.4cm][t] {0.31\textwidth}
	\begin{subfigure}{\textwidth}
	\centering
	\includegraphics[height = 0.8\linewidth, width= \linewidth]{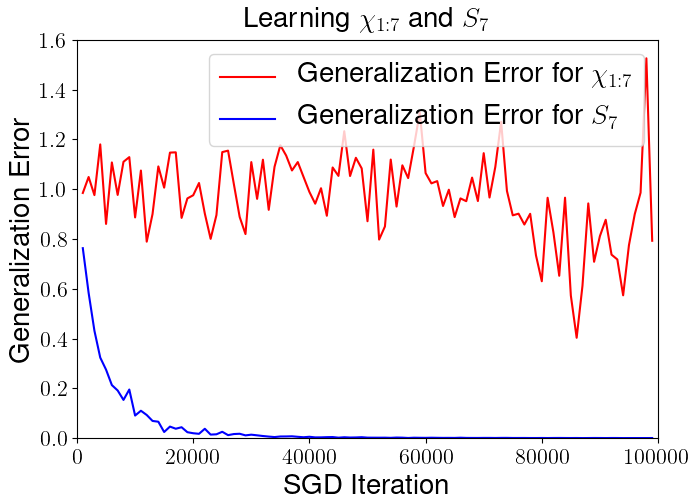}
	\caption{Loss Comparison for Parity and Staircase}
	\label{fig:loss_parity_stair_bern_b}
	\end{subfigure}\end{minipage}}
	\centering
    
    \fbox{\begin{minipage}[t][3.4cm][t] {0.31\textwidth}
	\begin{subfigure}{\textwidth}
	\centering
	\includegraphics[height = 0.8\linewidth, width= \linewidth]{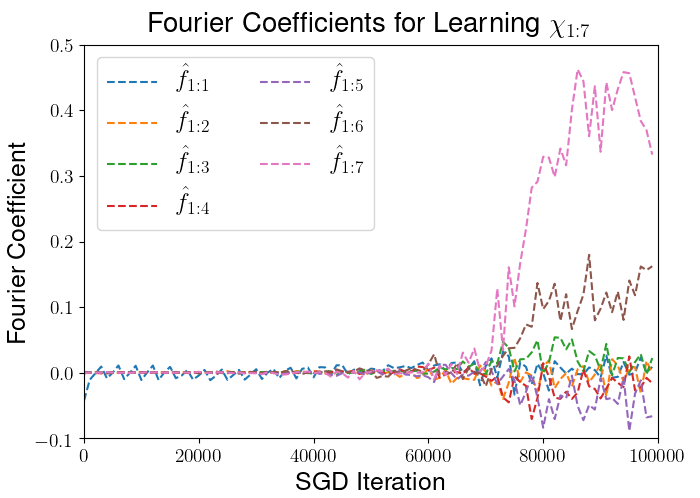}
	\caption{Fourier Coefficients for learning $\chi_{1:7}$}
	\label{fig:fourier_stair_bern_b}
	\end{subfigure}\end{minipage}}
	
		\fbox{\begin{minipage}[t][3.4cm][t]{0.31\textwidth}
	\begin{subfigure}{\textwidth}
		\centering
		\includegraphics[height = 0.8\linewidth, width=\linewidth]{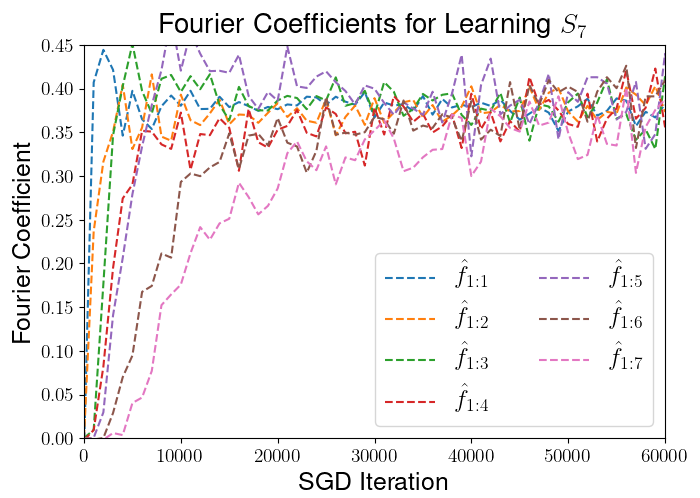}  
		  \caption{Fourier Coefficients for learning $S_7$}
    \label{fig:fourier_parity_bern_b}
	\end{subfigure}\end{minipage}}}
	 \end{minipage}}
	 \caption{Learning Staircase and Parity functions with biased Rademacher data.}
	 \label{fig:biased_bernoulli_expt}
\end{figure}

\paragraph{Learning the Double Staircase:}
We now consider learning the double staircase function, which has the structure defined in Definition~\ref{def:staircaseprop}. Define $S_{k,l} = S_k(x) + x_1 x_{k+1} + x_1 x_{k+1}x_{k+2} + \dots + x_1 \prod_{i=1}^{l-1}x_{k+i}$. We take $k = l = 7$ and $n = 30$, width $50$, depth $5$, mini-batch size $B = 20$ and number of samples $m = 10^5$. For simplicity, we choose the underlying distribution to be the uniform distribution over $\{-1,1\}^n$. The Fourier coefficients here are same as that for the staircase function under the uniform measure over $\{-1,1\}^n$.
\begin{figure}[ht]%
\centering
	\fbox{\begin{minipage}[t][5.3cm][t]{\textwidth}
	\centering
	
	\adjustbox{valign=t}{
	\fbox{\begin{minipage}[t][4.5cm][t] {0.47\textwidth}
	\begin{subfigure}{\textwidth}
	\centering
	\includegraphics[height = 0.7\linewidth, width= \linewidth]{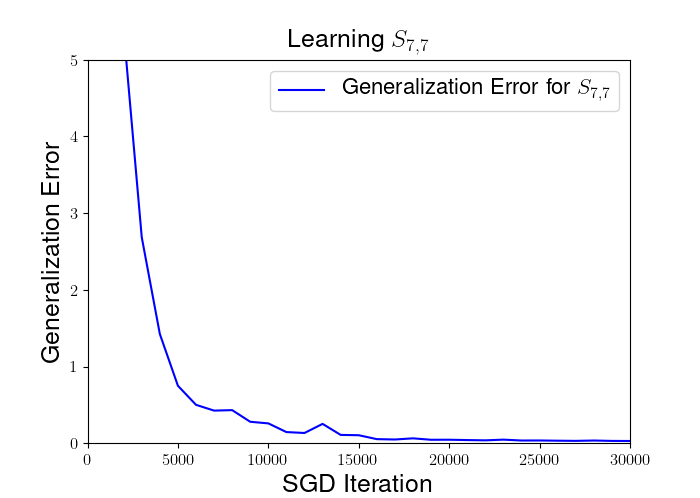}
	\caption{Loss for learning $S_{7,7}$}
	\label{fig:loss_multi}
	\end{subfigure}\end{minipage}}
	\centering
    
    \fbox{\begin{minipage}[t][4.5cm][t] {0.47\textwidth}
	\begin{subfigure}{\textwidth}
	\centering
	\includegraphics[height = 0.7\linewidth, width= \linewidth]{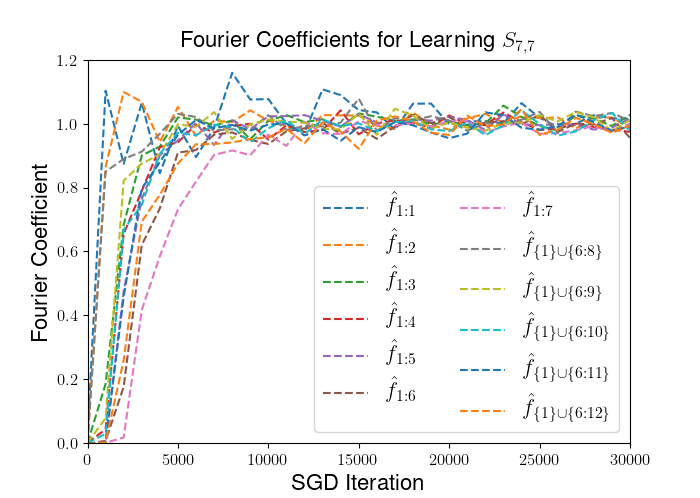}
	\caption{Fourier Coefficients for learning $S_{7:7}$}
	\label{fig:fourier_multi}
	\end{subfigure}\end{minipage}}}
	 \end{minipage}}
	 \caption{Learning the double-Staircase function with Rademacher data.}
	 \label{fig:multi_expt}
\end{figure}

\section{Formal Theorem Statement}
\label{sec:formal_theorem}
We restate the main theorem, giving an explicit set of hyperparameters that works.
\begin{theorem}\label{thm:restatedtheorem}
There is a universal constant $C > 0$ such that following holds. Let $g : \{-1,1\}^n \to \RR$ be an unknown $s$-sparse polynomial satisfying $[1/M,M]$-staircase property over the unbiased Boolean hypercube (\cref{def:staircaseprop}) for known $s,M >1$. Given an accuracy parameter $\ewant > 0$, a soundness parameter $0 < \delta < 1$, and access to random samples from $\{(x,g(x))\}_{x \sim \{-1,1\}^n}$, with the following setting of hyperparameters for \cref{alg:forward},
\begin{align}
L &= n \label{eq:setL} \\
W &= (64M^2(n+s+1)^3 L / \delta)^{24} \label{eq:setW} \\
p_1 &= (64M^2(n+s+1)^3 L / \delta)^{-9} \label{eq:setp1} \\
p_2 &= (64M^2(n+s+1)^3 L / \delta)^{-13} \label{eq:setp2} \\
\tau &= 1/(2^{20}M^7L)\label{eq:settau} \\
\eta &= 4\tau = 1/(2^{18}M^7L)\label{eq:seteta}
\end{align}
Define $$\kt{} = WLMs/(\eps \delta).$$
For a sufficiently small constant $c_{\lambda} > 0$,
\begin{align}
\lambda_2 &= c_{\lambda} \kt{-28} \leq 1  \label{eq:setlambda2} \\
\sqrt{\lambda_1 / \lambda_2} &= 1/ (64 M^2 L) \leq 1 \label{eq:setlambda1lambda2ratio}
\end{align}
For a sufficiently small constant $c_{stop} > 0$ that may depend on $c_{\lambda}$,
\begin{align}\label{eq:setestop}
\estop = c_{stop}\kt{-430}
\end{align}
For a sufficiently small constant $c_{\alpha} > 0$ that may depend on $c_{\lambda},c_{stop}$,
\begin{align}
\alpha = c_{\alpha} (\lambda_1\lambda_2)^{5} \kt{-72} \label{eq:setalpha}
\end{align}
For a sufficiently large constant $c_B > 0$ that may depend on $c_{\lambda},c_{stop},c_{\alpha}$,
\begin{align}
B = c_B (\lambda_1\lambda_2)^{-4} \kt{910}\label{eq:setB}
\end{align} Then, with probability at least $1 - \delta$, $\TrainNetworkLayerwise$ (\cref{alg:forward}) runs in $O(\kt{2394}) = O((nsWM/(\eps \delta))^{172368})$ time and samples, and
returns trained weights $w$ satisfying that the population loss is bounded to the desired accuracy:
\begin{align*}
\ell(w) \leq \ewant.
\end{align*}
\end{theorem}

\subsection{Basic definitions}\label{ssec:basicdefinitions}

A key concept in our proof will be ``blank'' neurons and ``active'' neurons. We say that a neuron is blank if it computes the zero function identically, and also all input and output edges have zero weight:
\begin{definition}
A neuron $v \in V \sm \Vin$ is {\em blank} at parameters $w = \{a_e\}_{e \in E} \cup \{b_v\}_{v \in V \sm \Vin}$ of the network if:
\begin{itemize}
    \item $f_v(x; w) = 0$ for all $x \in \{-1,1\}^n$, and
    \item $b_v = 0$, and
    \item $a_e = 0$ for all $e = (u_1,u_2) \in E$ such that $v \in \{u_1,u_2\}$.
\end{itemize}
\end{definition}

\begin{definition}[Active neuron]
A neuron $v \in V$ is active if and only if it is not blank.
\end{definition}

We will also often refer to parents of a neuron, which are the neurons that have edges into the neuron:
\begin{definition}[Parent neurons]
The parents of a neuron $v \in V$ are the set $P_v = \{u : (u,v) \in E\}$. 
\end{definition}

Finally, we also define what it means for a neuron to compute a monomial up to certain relative error:
\begin{definition}
Let $S \subset [n]$, $w$ be a setting of network parameters, and $v \in V$ be a neuron. We write that neuron $v$ computes $\chi_S(x)$ up to $\eps$ relative error if
$$f_u(x;w) = r\chi_S(x) + h(x)$$ for some scaling factor $r \in \RR$, some function $h : \{-1,1\}^n \to \RR$ such that $\hat{h}(S) = \EE_x[h(x) \chi_S(x)] = 0$, and such that $|h(x)| \leq |r| \eps$ for all $x$.
\end{definition}

\subsection{Proof organization}

Our proof is organized into three modular sections, described below.

\paragraph{\cref{sec:neuronsgdproof}: $\NeuronSGD$ correctness}
In this section, we prove that calling $\NeuronSGD$ will with high probability return an approximate stationary point of the loss in polynomial time and samples. The main technical difficulty in this section is to prove that the loss is smooth throughout training. To overcome this, we use the fact that the $L_2$ regularization ensures that the network's parameters are bounded during training.

\paragraph{\cref{sec:trainneuronproof}: $\TrainNeuron$ correctness}
In this section, we analyze calls to $\TrainNeuron(v,w)$ when $v$ is a neuron with at most two active parents. Roughly speaking, we prove that if (A) $v$ has two active parents that approximately compute monomials $\chi_{S_1}(x)$ and $\chi_{S_2}(x)$, and (B) the error $\EE_x[(f(x;w) - g(x)) \chi_{S_1}(x) \chi_{S_2}(x)]$ is large, then after training $v$ approximately computes $\chi_{S_1}(x)\chi_{S_2}(x)$. Otherwise, the neuron $v$ remains blank and all the weights in the network are unchanged. The proofs in this section consist of analyzing of the stationary points of the loss, since $\NeuronSGD$ is guaranteed to train to such a stationary point.

\paragraph{\cref{sec:trainnetworklayerwiseproof}: $\TrainNetworkLayerwise$ correctness}
In this section, we prove \cref{thm:restatedtheorem}. We show inductively on the training iteration that during training each neuron is either blank or it approximately represents one of the nonzero monomials of $g$, up to small relative error. Because the network is taken to be quite sparse (see hyperparameter setting above), at any iteration every neuron has at most two active parents. Therefore, the guarantees that we have proved for $\TrainNeuron$ apply to control the progress on each iteration.

\section{Correctness of $\NeuronSGD$: finds approximate stationary point} \label{sec:neuronsgdproof}
In this section, we show that with high probability $\NeuronSGD$ reaches an approximate stationary point of the regularized loss if the minibatch size is a large enough polynomial in the relevant parameters. We now introduce notation used to state and prove the main result of this section.

\begin{assumption}[Assumptions and notation for \cref{lem:neuronsgdstationarity}]\label{ass:neuronsgdconditions}

The inputs to $\NeuronSGD$ are a neuron $v \in V \sm \Vin$, and an initialization of parameters $$w^0 = \{a_e^0\}_{e \in E} \cup \{b_v^0\}_{v \in V \sm \Vin},$$ such that the following hold:
\begin{itemize}
    \item At initialization, all neurons have magnitude upper-bounded by $\Uneur > 1$:
    \begin{align*}
        \max_{u \in V} \max_{x \in \{-1,1\}^n} |f_u(x;w^0)| \leq \Uneur.
    \end{align*}
    
    \item Neuron $v$ has outward edges' weights equal to zero at initialization: i.e., $$a_e^0 = 0\mbox{ for all }e = (v,u) \in E.$$
    
    \item During training, only the subset of parameters
    $$w_v = \{a_e\}_{e = (u,v) \in E} \cup \{b_v\}$$
    corresponding to the inputs to neuron $v$. Therefore, $$w_{-v}^{t} = w_{-v}^0 \mbox{ and } w^{t} = [w_{-v}^0, w_v^{t}]$$ for any iteration $t$. In particular, $w^{t} = \{a_e^{t}\}_{e \in E} \cup \{b_v^{t}\}_{v \in V}$.
    
    \item Let $T$ denote the number of iterations, so the method returns $$w^T = [w_{-v}^0, w_v^T].$$
    
    \item $g : \{-1,1\}^n \to \RR$ is an $s$-sparse polynomial satisfying the $[1/M,M]$-staircase property for some $M \geq 1$, as in \cref{thm:restatedtheorem}.
\end{itemize}
\end{assumption}

In all of the results of this section, we assume that \cref{ass:neuronsgdconditions} holds. Now we state the main result of the section. Since the hyperparameters are fixed we omit explicit dependence by writing $\NeuronSGD(v,w) := \NeuronSGD(v,w; \lambda_1,\lambda_2,\eta,B,\estop,\alpha,\tau)$.

\begin{lemma}\label{lem:neuronsgdstationarity}
Consider running $\NeuronSGD(v,w^0)$ (\cref{alg:neuronsgd}) where $v \in V \sm \Vin$ is a neuron, and $w^0$ are the initial parameters of the network.

Let $\dgrad > 0$ and define
$$\tbound = \ceil{3\ellr(w^0) / (\alpha (\estop)^2)} + 1.$$
Suppose that for some large enough universal constant $C$, the mini-batch size is at least
$$B \geq C(\lambda_1\lambda_2)^{-3}\kt{8}(\Uneur)^4 (1 + \ellr(w^0)^4) \log(2\tbound / \dgrad) / \estop^2,$$
and the learning rate is at most
$$\alpha < 1 / (C (\lambda_1\lambda_2)^{-5}\kt{16}(\Uneur)^{16}(1+\ellr(w^0)^4)).$$
Then the following hold with probability at least $1 - \dgrad$:
\begin{enumerate}
\item The loss does not increase: $$\ellr(w^T) \leq \ellr(w^0)$$
    \item The output $w^T$ is a $2\estop$-approximate stationary point of the loss with respect to $w_v$:
$$\|\nabla_{w_v} \ellr(w^T)\| \leq 2\estop$$
\item The number of iterations of stochastic gradient descent until a stationary point is reached is polynomially-bounded:
$$T \leq \tbound$$
\end{enumerate}
\end{lemma}

The proof is a standard analysis of stochastic gradient descent finding an approximate stationary point of a nonconvex loss. However, care must be taken because the loss is not uniformly smooth: if the parameters of the network grow to infinity, then the gradient of the loss may also grow to infinity. In order to overcome this technical obstacle, we prove that the $L_2$ regularization term ensures that the parameters of the network are bounded during training.

Specifically, we prove inductively on the loop iteration $t \in \{0,\ldots,T\}$ that with high probability the loss $\ellr(w^t)$ does not increase. For the inductive step, we note that the $L_2$ regularization and the upper bound on the loss implies that the parameters are polynomially upper-bounded at each iteration. In turn, this means that the loss is smooth in the neighborhood of the current iterate. And since the current iterate is not close to a stationary point (since otherwise we exit the loop), the loss decreases with high probability,  completing the inductive step. The proof is given below, although several auxiliary claims must be proved first.

\subsection{Parameters are bounded by loss}
In this section, we prove \cref{claim:upperboundedparamsbyupperboundedloss}, which shows that the parameters of the network are polynomially-bounded by the loss during training. First, let us show several auxiliary results.

We observe that training the weights $w_v$ only affects the value of neuron $v$, since all output edges from the neuron $v$ have zero weight:
\begin{lemma}\label{lem:neuronsaftertrainingwv}
Under \cref{ass:neuronsgdconditions}, for any setting of the parameters $w_v = \{a_e\}_{e = (u,v) \in E} \cup \{b_v\}$,
\begin{align}
f(x;[w_{-v}^0,w_v]) &= f(x; [w_{-v}^0, \vec{0}]) + f_v(x; [w_{-v}^0, w_v]). \label{eq:trainneurondecomposition}
\end{align}
And for any neuron $u \neq v \in V$,
\begin{align}
f_u(x;[w_{-v}^0,w_v]) &= f(x; [w_{-v}^0,\vec{0}])\label{eq:otherneuronsindependent}
\end{align}
\end{lemma}
\begin{proof}
 For any $u \in V \sm (\Vin \cup \{v\})$ we claim that $f_u(x;[w_{-v}^0,w_v]) = f_u(x;[w_{-v}^0,\vec{0}])$. If $u$ is a successor of $v$ then by induction on the depth of $u$, we have that $f_u$ is independent of the value of neuron $f_v$, since all outward edges from $v$ have zero weight under $w_{-v}^0$ by \cref{ass:neuronsgdconditions}. On the other hand, if $u$ is not a successor of $v$ then it is independent of $w_v$.
 
Finally, $f_v(x;[w_{-v}^0,\vec{0}]) = 0$ for all $x \in \{-1,1\}^n$ because all the edges to $v$ have zero weight, and $v$ has zero bias. So $f_v(x;[w_{-v}^0,w_v]) = f_v(x;[w_{-v}^0,\vec{0}]) + f_v(x;[w_{-v}^0,\vec{0}])$. \cref{eq:otherneuronsindependent} follows from recalling the definition $f(x; [w_{-v}^0,w_v]) = \sum_{u \in V \sm \Vin} f_u(x;[w_{-v}^0,w_v])$.
\end{proof}

In order to prove~\cref{lem:neuronsgdstationarity}, we must first prove several auxiliary claims.

\begin{claim}\label{claim:uneurlocalbound}
Under \cref{ass:neuronsgdconditions}, for any neuron $u \neq v \in V$ and setting $w_v = \{a_e\}_{e = (u,v) \in E} \cup \{b_v\}$ of the parameters of neuron $v$, define the set of network parameters $w = [w_{-v}^0, w_v]$. Then
\begin{equation*}|f_u(x;w)| \leq \Uneur.\end{equation*}
\end{claim}
\begin{proof}
By \cref{eq:otherneuronsindependent},
$|f_u(x;w)| = |f_u(x;[w_{-v}^0, \vec{0}])| = |f_u(x;[w_{-v}^0, w^{0}])| = |f_u(x;w^0)| \leq \Uneur$.
\end{proof}

\begin{claim}\label{claim:gisupperbounded}
Suppose that $g$ is an $s$-sparse polynomial satisfying the $[1/M,M]$-staircase property, as in \cref{ass:neuronsgdconditions}. Then,
$\max_x |g(x)| \leq Ms$.
\end{claim}
\begin{proof} For any $x \in \{-1,1\}^n$, by the H\"{o}lder inequality,
$$|g(x)| = |\sum_{S \subset [n]} \hat{g}(S) \chi_S(x)| \leq |\{\hat{g}(S) \neq 0\}| \cdot \max_S |\hat{g}(S)| \leq Ms.$$
\end{proof}

\begin{claim}[Parameters are upper-bounded during training]\label{claim:upperboundedparamsbyupperboundedloss}
For any setting $w_v = \{a_e\}_{e = (u,v) \in E} \cup \{b_v\}$ of the parameters of neuron $v$, define the set of network parameters $w = [w_{-v}^0, w_v]$. Then
\begin{align}
\max_{e \in E} |a_e| \lsim (\lambda_1\lambda_2)^{-1/2}\sqrt{\ellr(w)} \label{eq:atrainupper}
\end{align}
\begin{align}|b_v| \lsim (\lambda_1\lambda_2)^{-1}\kt{2} (\Uneur)^2 \ellr(w) \label{eq:btrainupper} \end{align}
In particular, we obtain the following bound for the parameters $w_v$ associated with neuron $v$:
\begin{align*}
\|w_v\|_{\infty} \lsim (\lambda_1\lambda_2)^{-1}\kt{2} (\Uneur)^2 \max(1,\ellr(w)).
\end{align*}
\end{claim}
\begin{proof}
For any $e \in E$, the bound on $|a_e|$ follows because of the $L_2$ regularization term \begin{align*}\frac{1}{2} \lambda_1\lambda_2 \cdot (a_e)^2 \leq \frac{1}{2}\lambda_1 \cdot (a_e)^2 \leq R(w) \leq \ellr(w).\end{align*}

We now prove the bound on $|b_v|$, using the above bound on $|a_e|$. For any $x \in \{-1,1\}^n$,
\begin{align*}
|f(x;w)| &\geq |f_v(x;w)| - \sum_{u \in V \sm (\Vin \cup \{v\})} |f_u(x;w)| \\
&\geq |f_v(x;w)| - WL \Uneur & \mbox{using \cref{claim:uneurlocalbound}} \\
&\geq |f_v(x;w)| - \kt{}\Uneur
\end{align*}
Furthermore, for any $x \in \{-1,1\}^n$, recall that $f_v(x;w) = \left(\sum_{e = (u,v) \in E} a_e f_u(x;w)\right)^2 + b_v$, so
\begin{align*}
|f_v(x;w)| &\geq |b_v| - (|\{(u,v) \in E\}| \cdot \max_e |a_e| \cdot \max_{u \in V \sm \{v\}} |f_u(x;w)|)^2\\
&\geq |b_v| - ((2W) 2(\lambda_1\lambda_2)^{-1/2} \cdot \sqrt{\ellr(w)} \cdot \Uneur)^2 & \mbox{using \cref{eq:atrainupper}} \\
&\geq |b_v| - 16 (\lambda_1\lambda_2)^{-1}\kt{2} (\Uneur)^2 \ellr(w)
\end{align*}
Recall from \cref{claim:gisupperbounded} that $|g(x)| \leq Ms \leq \kappa$ for any $x \in \{-1,1\}^n$. This implies
\begin{align*}
\ellr(w) &\geq 
\EE_{x \sim \{-1,1\}^n} \frac{1}{2}(f(x;w) - g(x))^2 \\
&\geq \EE_{x \sim \{-1,1\}^n} \frac{1}{2}(\max(0,|f_v(x;w)| - \kt{}\Uneur - \kappa))^2 \\
&\geq (|b_v| - 16 (\lambda_1\lambda_2)^{-1}\kt{2} (\Uneur)^2 \ellr(w) - 2\kt{}\Uneur)^2,
\end{align*}
so we must have
\begin{align*}
|b_v| &\leq \sqrt{\ellr(w)} + 16 (\lambda_1\lambda_2)^{-1}\kt{2} (\Uneur)^2 \ellr(w) - 2\kt{}\Uneur \\
&\lsim (\lambda_1\lambda_2)^{-1}\kt{2} (\Uneur)^2 \ellr(w).
\end{align*}
\end{proof}

\subsection{Stochastic gradient approximation is close during training}
The main result of this section is \cref{claim:gradclose}, which proves that if the loss is bounded during training, then the stochastic gradient approximations $\xi^t$ are close to the true gradients with high probability.

First, we prove that if the parameters associated with the neuron $v$ are bounded, then the magnitude of the function computed by the network at each neuron is polynomially upper-bounded:
\begin{claim}[Neurons are upper-bounded during training]\label{claim:uneurtrainupperbound}
Under \cref{ass:neuronsgdconditions}, for any setting $w_v$ of the parameters of neuron $v$, define the network parameters $w = [w_{-v}^0, w_v]$. Then,
\begin{align*}
\max_{u \in V} \max_{x \in \{-1,1\}^n} |f_u(x;w)| \lsim (\lambda_1\lambda_2)^{-1}\kt{2} (\Uneur)^2 \ellr(w).
\end{align*}
\end{claim}
\begin{proof}
The bound holds for all $u \in V \sm \{v\}$ by \cref{claim:uneurlocalbound}. For $v$, recall that $f_v(x;w) = \left(\sum_{e = (u,v) \in E} a_e f_u(x;w)\right)^2 + b_v$, so
\begin{align*} |f_v(x;w)| &\leq (|\{(u,v) \in E\}| \cdot \max_{e = (u,v)} |a_e| \cdot \max_{u \in V \sm \{v\}} |f_u(x;w)|)^2 + |b_v| \\
&\lsim (2W \cdot (\lambda_1\lambda_2)^{-1/2} \sqrt{\ellr(w)} \cdot \Uneur)^2 + (\lambda_1\lambda_2)^{-1}\kt{2}(\Uneur)^2 \ellr(w) & \mbox{by \cref{claim:upperboundedparamsbyupperboundedloss}} \\
&\lsim (\lambda_1\lambda_2)^{-1}\kt{2} (\Uneur)^2 \ellr(w).
\end{align*}
\end{proof}

\begin{claim}[Gradient of neuron $v$ is upper-bounded during training]\label{claim:ugradf}
Under \cref{ass:neuronsgdconditions}, for any setting $w_v$ of the parameters of neuron $v$, the gradient of $f_v$ with respect to $w_v$ is bounded:
$$\max_{x \in \{-1,1\}^n}\|\nabla_{w_v} f_v(x; w)\|_{\infty} \lsim \max(\kt{} \Uneur \max_{e = (u,v) \in E} |a_e|,1).$$
\end{claim}
\begin{proof}
For any $x \in \{-1,1\}^n$,
\begin{align*}
&\|\nabla_{w_v} f_v(x; w)\|_{\infty} = \|\nabla_{w_v}((\sum_{e=(u,v) \in E} a_e f_u(x;w))^2 + b_v)\|_{\infty} \\
&= \max\left(\left(2\sum_{e=(u,v) \in E} a_e f_u(x;w)\right) \cdot \max_{e = (u,v) \in E}|f_u(x;w)|,1\right) \\
&\leq \max(4W \Uneur \max_{e = (u,v) \in E} |a_e|,1) \\
&\lsim \max(\kt{} \Uneur \max_{e = (u,v) \in E} |a_e|,1)
\end{align*}
\end{proof}

\begin{claim}[Gradient of loss is upper-bounded during training]\label{claim:ugradlr}
Under~\cref{ass:neuronsgdconditions}, for any $x \in \{-1,1\}^n$ and for any setting $w_v$ of the parameters of neuron $v$, the gradient of the loss with respect to $w_v$ is bounded. Namely, defining the set of network parameters $w = [w_{-v}^0, w_v]$, we have
\begin{align*}
\|\nabla_{w_v} \ellr(x; w)\|_{\infty} \lsim (\lambda_1\lambda_2)^{-3/2}\kt{4}(\Uneur)^3 \max(\ellr(w)^2,1).
\end{align*}
\end{claim}
\begin{proof}
For the subsequent arguments, define the ``error'' function $$\zeta(x; w) = f(x;w) - g(x).$$ This is the gap between the learned function $f$ from the true function $g$. The definition of $\zeta$ allows us to write the gradient of the unregularized loss at $x \in \{-1,1\}^n$ as:
\begin{align*}
\nabla_{w_v} \ell(x; w) &= \frac{1}{2} \nabla_{w_v} (f(x;w) - g(x))^2 \\
&= \frac{1}{2} \nabla_{w_v} (\zeta(x;[w_{-v}^0,\vec{0}]) + f_v(x;w))^2 &\mbox{by \cref{eq:trainneurondecomposition}} \\
&= (\zeta(x;[w_{-v}^0,\vec{0}]) + f_v(x;w)) \nabla_{w_v} f_v(x;w)\end{align*}
So we may upper-bound the gradient of the unregularized loss at $x$ by:
\begin{align*}
\|\nabla_{w_v}& \ell(x;w)\|_{\infty} \\ &\leq |\zeta(x;[w_{-v}^0,\vec{0}]) + f_v(x;w)| \cdot \|\nabla_{w_v} f_v(x;w)\|_{\infty} \\
&= |f(x;w) - g(x)| \cdot \|\nabla_{w_v} f_v(x;w)\|_{\infty}, &\mbox{by \cref{eq:trainneurondecomposition}} \\
&\lsim (|f(x;w) - g(x)|) \max(\kt{} \Uneur \max_{e = (u,v) \in E} |a_e|, 1) &\mbox{by \cref{claim:ugradf}} \\
&\lsim (|f(x;w) - g(x)|) \max((\lambda_1\lambda_2)^{-1/2}\kt{} \Uneur \sqrt{\ellr(w)}, 1) &\mbox{by \cref{claim:upperboundedparamsbyupperboundedloss}} \\
&\lsim (|f(x;w)| + Ms) \max((\lambda_1\lambda_2)^{-1/2}\kt{} \Uneur \sqrt{\ellr(w)}, 1) &\mbox{by \cref{claim:gisupperbounded}} \\
&\lsim ((\lambda_1\lambda_2)^{-1}\kt{3} (\Uneur)^2 \ellr(w) + Ms)\max((\lambda_1\lambda_2)^{-1/2}\kt{} \Uneur \sqrt{\ellr(w)}, 1) &\mbox{by \cref{claim:uneurtrainupperbound}}\\
&\lsim (\lambda_1\lambda_2)^{-3/2}\kt{4}(\Uneur)^3 \max(\ellr(w)^2,1)
\end{align*}
Finally, the triangle inequality implies an upper-bound on the gradient of the regularized loss at $x$:
\begin{align*}
\|\nabla_{w_v} \ellr(x;w)\|_{\infty} &\leq \|\nabla_{w_v} \ell(x;w)\|_{\infty} + \max(\lambda_1,\lambda_2) \max_{e = (u,v) \in E} |a_e| \\
&\lsim (\lambda_1\lambda_2)^{-3/2}\kt{4}(\Uneur)^3 \max(\ellr(w)^2,1),
\end{align*}
by \cref{claim:upperboundedparamsbyupperboundedloss}, and using $0 < \lambda_1,\lambda_2 \leq 1$.
\end{proof}
Finally, we use the above bounds to prove that with high probability the stochastic gradients computed by $\NeuronSGD$ are close to the true gradients if the minibatch size is taken to be a large enough polynomial.
\begin{claim}[Stochastic gradient approximation is close during training]\label{claim:gradclose}
Under \cref{ass:neuronsgdconditions}, there is a large enough constant $C$ such that for any $\egrad > 0, \dgrad > 0$ and iteration $0 \leq t \leq T$, if
$$B \geq C (\lambda_1\lambda_2)^{-3}\kt{8} (\Uneur)^4 \log(1/\dgrad) \max(\ellr(w^t)^4,1) / \egrad^2,$$
then
$$\PP[\|\xi^{t} - \nabla_{w_v} \ellr(w^{t})\| > \egrad \mid w^{t}] \leq \dgrad.$$
\end{claim}
\begin{proof}
Recall the definition of $\xi^{t}$ from \cref{step:drawbatch,step:appxgrad} of \cref{alg:neuronsgd}. Namely, draw i.i.d. $x^{t,1},\ldots,x^{t,B} \sim \{-1,1\}^n$, and define the random variable $\xi^t$ as follows: $$\xi^{t} = \frac{1}{B} \sum_{i=1}^B \nabla_{w_v} \ellr (x^{t,i}; w^{t}).$$
By linearity of expectation and differentiation, $\xi^{t}$ is an unbiased estimator of the true gradient: $$\EE[\xi^{t}] = \frac{1}{B} \sum_{i=1}^B \nabla_{w_v} \EE[\ellr(x^{t,i}; w^{t})] = \nabla_{w_v} \EE_{x \sim \{-1,1\}^n}[\ellr(x; w^{t})] = \nabla_{w_v} \ellr(w^{t}).$$
So it suffices to prove that $\xi^{t}$ concentrates around its mean. We will use the Hoeffding inequality.
\begin{align*}
\PP[\|\xi^t - \nabla_{w_v} \ellr(x;w^{t})\| > \egrad \mid w^t] 
&\leq 
\PP[\|\xi^t - \nabla_{w_v}\ellr(x;w^{t})\|_{\infty} > \egrad / \sqrt{2W} \mid w^t] \\
&\leq (2W) \exp(-2B\egrad^2 / (2\sqrt{2W}\max_x \|\nabla_{w_v} \ellr(x;w^{t})\|_{\infty})^2).
\end{align*}
The first inequality uses that $\xi^t$ is of length at most $2W$, since there are at most $W+n+1 \leq 2W$ parameters associated with the neuron $v$, since there are at most $W$ edges to $v$ from the previous layer, at most $n$ edges to $v$ from the inputs, and one bias parameter. The second inequality is the Hoeffding bound.
Therefore, the inequality 
$$
\PP_{x\sim \{-1,1\}^n}[\|\xi - \nabla_{w_v} \ellr(x;w^{t})\| > \egrad] \leq \dgrad
$$
follows by using \cref{claim:ugradlr} to upper bound $\|\nabla_{w_v} \ellr(x; w^{t})\|_{\infty}$, and choosing the constant $C$ in the statement of the claim large enough.
\end{proof}

\subsection{Objective is smooth during training}
The final claim bounds the smoothness of the loss function at each iterate during training. This is needed to prove that with high probability $\NeuronSGD$ does not increase the loss.
\begin{claim}\label{claim:losssmooth}
Under \cref{ass:neuronsgdconditions}, given a setting of parameters $w_v$ for neuron $v$, let $w = [w_{-v}^0, w]$. Then the Hessian of the loss at $w$ with respect to the parameters $w_v$ has bounded norm
$$
\|\nabla^2_{w_v} \ellr(w)\| \lsim (\lambda_1\lambda_2)^{-1}\kt{4}(\Uneur)^4\max(\ellr(w),1). 
$$
\end{claim}
\begin{proof}
For the subsequent arguments, define the error function $$\zeta(x;w) = f(x; w) - g(x),$$ in the same way as defined for the proof of \cref{claim:ugradlr}.
We now write the Hessian of the loss with respect to the parameters $w_v = \{a_e\}_{e = (u,v) \in E} \cup \{b_v\}$ at any point $x \in \{-1,1\}^n$. For any $e = (u,v), e' = (u',v) \in E$:
\begin{align*}
&\frac{\partial^2 \ell(x;w)}{\partial a_{e'} \partial a_e} \\
&= \pd{}{a_{e'}} \pd{\ell(x;w)}{a_{e}} \\
&= \pd{}{a_{e'}} ((\zeta(x;[w_{-v}^0, \vec{0}]) + f_v(x;w)) \pd{f_v(x;w)}{a_e}) \\
&= \left(\pd{f_v(x;w)}{a_{e'}}\right)\left(\pd{f_v(x;w)}{a_e}\right) + (\zeta(x;[w_{-v}^0, \vec{0}]) + f_v(x;w)) \cdot \left(\frac{\partial^2 f(x;w)}{\partial a_{e'} \partial a_e}\right) \\
&= \left(\pd{f_v(x;w)}{a_{e'}}\cdot \pd{f_v(x;w)}{a_e}\right) \\
&\quad\quad + (\zeta(x;[w_{-v}^0, \vec{0}])+ f_v(x;w)) \cdot \left(\pd{}{a_{e'}} 2 \sum_{e''=(u'',v) \in E} a_{e''} f_{u''}(x;w) f_u(x;w)\right) \\
&= \left(\pd{f_v(x;w)}{a_{e'}}\cdot \pd{f_v(x;w)}{a_e}\right) + (\zeta(x;[w_{-v}^0, \vec{0}]) + f_v(x;w)) \cdot \left(2 f_{u'}(x;w) f_u(x;w)\right)
\end{align*}
So 
\begin{align*}
&\left|\frac{\partial^2 \ell(x;w)}{\partial a_{e'} \partial a_e}\right| \\
&\leq \left|\pd{f_v(x;w)}{a_{e'}}\cdot \pd{f_v(x;w)}{a_e}\right| + 2 |\zeta(x;[w_{-v}^0, \vec{0}]) + f_v(x;w)| \cdot \left|f_{u'}(x;w) f_u(x;w)\right| \\
&\leq \max_{e'' = (u'',v) \in E} \left|\pd{f_v(x;w)}{a_{e''}}\right|^2 + 2 |\zeta(x;[w_{-v}^0, \vec{0}]) + f_v(x;w)| \cdot (\Uneur)^2 & \mbox{by \cref{claim:uneurlocalbound}} \\
&\lsim \max(\kt{2}(\Uneur)^2 \max_{e'' = (u'',v) \in E} |a_{e''}|^2, 1) + 2 |\zeta(x;[w_{-v}^0, \vec{0}]) + f_v(x;w)| \cdot (\Uneur)^2 & \mbox{by \cref{claim:ugradf}} \\
&\lsim (\lambda_1\lambda_2)^{-1}\kt{2} (\Uneur)^2 \max(\ellr(w),1) + 2 |\zeta(x;[w_{-v}^0, \vec{0}]) + f_v(x;w)| \cdot (\Uneur)^2 & \mbox{by \cref{claim:upperboundedparamsbyupperboundedloss}} \\
&= (\lambda_1\lambda_2)^{-1}\kt{2} (\Uneur)^2 \max(\ellr(w),1) + 2 |f(x;w) - g(x)| \cdot (\Uneur)^2 & \mbox{by \cref{eq:trainneurondecomposition}} \\
&\leq (\lambda_1\lambda_2)^{-1}\kt{2} (\Uneur)^2 \max(\ellr(w),1) + 2 |f(x;w)| \cdot (\Uneur)^2 + 2Ms \cdot (\Uneur)^2 &\mbox{by \cref{claim:gisupperbounded}} \\
&\lsim (\lambda_1\lambda_2)^{-1}\kt{2} (\Uneur)^2 \max(\ellr(w),1) + 2 (\lambda_1\lambda_2)^{-1}\kt{3}(\Uneur)^4 \ellr(w) + 2Ms \cdot (\Uneur)^2  &\mbox{by \cref{claim:uneurtrainupperbound}} \\
&\lsim (\lambda_1\lambda_2)^{-1}\kt{3} (\Uneur)^4 \max(\ellr(w),1).
\end{align*}
Similarly, for any $e = (u,v) \in E$,
\begin{align*}
\left|\frac{\partial^2 \ell(x;w)}{\partial a_{e} \partial b_v}\right| &= \left|\pd{}{b_v} ((\zeta(x;[w_{-v}^0, \vec{0}]) + f_v(x;w)) \pd{f_v(x;w)}{a_e})\right|\\
&= \left|\pd{f_v(x;w)}{a_{e}}\cdot \pd{f_v(x;w)}{b_v}\right|  \\
&= \left|\pd{f_v(x;w)}{a_{e}}\right| &\mbox{since $\pd{f_v(x;w)}{b_v}$ = 1} \\
&\lsim \max(\kt{}\Uneur \max_{e' = (u',v) \in E} |a_{e'}|,1) &\mbox{by \cref{claim:ugradf}} \\
&\lsim (\lambda_1\lambda_2)^{-1/2}\kt{}(\Uneur)^2 \max(\sqrt{\ellr(w)},1) &\mbox{by \cref{claim:upperboundedparamsbyupperboundedloss}}.
\end{align*}
And similarly:
\begin{align*}
\left|\frac{\partial^2 \ell(x;w)}{\partial b_v \partial b_v}\right| &= \left|\pd{f_v(x;w)}{b_v}\cdot \pd{f_v(x;w)}{b_v}\right| = 1.
\end{align*}
Finally, this allows us to bound the operator norm of the Hessian of the regularized loss:
\begin{align*}\|\nabla_{w_v}^2 \ellr(w)\| &= \|\EE[\nabla_{w_v}^2 \ellr(x; w)]\| \\
&\leq \|\EE[\nabla_{w_v}^2 \ell(x; w)]\| + \max(\lambda_1, \lambda_2) \\
&= \EE[\max_{\phi : \|\phi\| = 1} \phi^T (\nabla_{w_v}^2 \ell(x; w)) \phi] + \max(\lambda_1, \lambda_2) &\mbox{by the Courant-Fischer Theorem} \\
&\lsim 2W (\lambda_1\lambda_2)^{-1}\kt{3}(\Uneur)^4\max(\ellr(w),1) & \mbox{since $w_v$ has length at most $2W$} \\
&\lsim (\lambda_1\lambda_2)^{-1}\kt{4}(\Uneur)^4\max(\ellr(w),1).
\end{align*}
Note that we use that $w_v$ has length at most $2W$, which is true since there are at most $W+n < 2W$ edge parameters and $1$ bias parameter associated with neuron $v$.
\end{proof}

\subsection{Loss decreases if gradient approximation is good}

In this section, we prove \cref{claim:gradcloselossdecreases}, which shows that if the gradient approximation $\xi^t$ on iteration $t$ in $\NeuronSGD$ is sufficiently accurate, then the loss decreases. In order to show this, we first prove a claim that is essentially a converse of \cref{claim:upperboundedparamsbyupperboundedloss}: namely, we show that if the parameters $w_v$ of neuron $v$ are upper-bounded, then the loss is upper-bounded as well.
\begin{claim}[Bounded change in parameters implies bounded change in loss]\label{claim:boundparamschangemeansboundloss}
For a given setting $w_v$ of the parameters of neuron $v$, define the network parameters $w = [w_{-v}^0, w_v]$. Furthermore, for any real-valued vector of parameters $\mu$ of the same length as $w_v$, define $w'_v = w_v + \mu$ and $w' = [w_{-v}^0, w'_v]$. Then the following holds:
\begin{align*}
\ellr(w') &\lsim (\lambda_1\lambda_2)^{-4}\kt{12}(\Uneur)^{12} (1 + \ellr(w)^4 + \|\mu\|_{\infty}^4)
\end{align*}
\end{claim}
\begin{proof}
\begin{align*}
\ell(w') &= \EE_{x \sim \{-1,1\}^n} \ell(x;w') \\
&= \frac{1}{2} \EE_{x \sim \{-1,1\}^n} \left(f(x;w') - g(x)\right)^2 \\
&= \frac{1}{2} \EE_{x \sim \{-1,1\}^n} \left(f(x;[w_{-v}^0, 0]) + f_v(x; w') - g(x)\right)^2 &\mbox{by \cref{eq:trainneurondecomposition}} \\
&\leq \EE_{x \sim \{-1,1\}^n} \left(Ms + WL\Uneur + |f_v(x;w')|\right)^2 & \mbox{by \cref{claim:gisupperbounded} and \cref{claim:uneurlocalbound}} \\
&\lsim \EE_{x \sim \{-1,1\}^n} \left(\kt{} \Uneur + |f_v(x;w')|\right)^2 \\
&\lsim \kt{2}\Uneur^2 + \EE_{x \sim \{-1,1\}^n} |f_v(x;w')|^2 \\
&\lsim \kt{2} \Uneur^2 + \EE_{x \sim \{-1,1\}^n} \left(|b'_v| + \left(\sum_{e = (u,v) \in E} a'_e f_u(x;w')\right)^2 \right)^2 \\
&\lsim \kt{2} \Uneur^2 + \left((2W)^2\|w'_v\|_{\infty}^2 \Uneur^2\right)^2 & \mbox{by \cref{claim:uneurlocalbound}} \\
&\lsim \kt{4} \Uneur^4 \max(1,\|w'_v\|_{\infty}^4) \\
&\lsim \kt{4} \Uneur^4 (1 + \|w_v\|_{\infty}^4 + \|\mu\|_{\infty}^4).
\end{align*}
Furthermore, by \cref{claim:upperboundedparamsbyupperboundedloss},
\begin{align*}
\|w_v\|_{\infty} &\lsim (\lambda_1\lambda_2)^{-1}\kt{2}(\Uneur)^2\max(1,\ellr(w)).
\end{align*}
So
\begin{align*}
\ellr(w') - \ellr(w) &\leq \ell(w') + \frac{1}{2}\sum_{e} \max(\lambda_1,\lambda_2) \|w'_v\|_{\infty}^2 \\
&\lsim \kt{4} \Uneur^4 \max(\|w_v\|_{\infty}^4 + \|\mu\|_{\infty}^4) + (2W) \max(\lambda_1,\lambda_2)( \|w_v\|_{\infty}^2 + \|\mu\|_{\infty}^2 ) \\
&\lsim (\lambda_1\lambda_2)^{-4}\kt{12}(\Uneur)^{12} (1 + \ellr(w)^4 + \|\mu\|_{\infty}^4),
\end{align*}
using $\lambda_1,\lambda_2 \leq 1$ for the last line.
\end{proof}

\begin{definition}
For any iteration $0 \leq t \leq T$, let $\Egradclose{t}$ be the event that for all $0 \leq t' \leq t$ we have $$\|\xi^{t'} - \nabla_{w_v} \ellr(w^{t'})\| \leq \estop / 3.$$
\end{definition}
\begin{claim}\label{claim:gradcloselossdecreases}
Suppose that the learning rate $\alpha$ satisfies $\alpha < (\lambda_1\lambda_2)^{5} / (C\kt{16}(\Uneur)^{16}(1+\ellr(w^0)^4))$ for some large enough universal constant $C$, and let $0 \leq t \leq T-1$. If $\Egradclose{t}$ holds, then $$\ellr(w^{t+1}) \leq \ellr(w^t) - \alpha\|\xi^t\|^2 / 3.$$
\end{claim}
\begin{proof}
The proof is by induction on $t$. For any $0 \leq t < T$, suppose that $\Egradclose{t}$ holds. By Taylor's theorem there is $\theta_t \in (0,1)$ such that:
\begin{align*}
\ellr(w^{t+1}) &= \ellr(w^{t}) - \alpha \xi^{t} \cdot \nabla_{w_v} \ellr(w^{t}) + \frac{\alpha^2}{2} ((\nabla^2_{w_v} \ellr(\theta_t w^{t} + (1-\theta_t)w^{t+1})) \xi^{t}) \cdot \xi^{t}
\end{align*}
We note that $\Egradclose{t}$ implies
\begin{align*}\xi^{t} \cdot \nabla_{w_v} \ellr(w^{t}) &\geq \|\xi^t\|^2 - \estop \|\xi^t\| / 3 \geq (2/3)\|\xi^t\|^2,\end{align*}
where the second inequality is due to $\|\xi^t\| > \estop$, because $t < T$.

Furthermore, we note that $$\ellr(w^t) \leq \ellr(w^0).$$ If $t = 0$ then the above inequality is trivial, and if $t \geq 1$ then it is true by the inductive hypothesis, since $\Egradclose{t}$ implies $\Egradclose{t-1}$. This allows us to prove that the loss is smooth in the neighborhood of $w^t$:
\begin{align*}
|((\nabla^2_{w_v}& \ellr(\theta_t w^{t} + (1-\theta_t)w^{t+1})) \xi^{t}) \cdot \xi^{t}|  \\
&\leq \|\nabla^2_{w_v} \ellr(\theta_t w^{t} + (1-\theta_t)w^{t+1})\| \|\xi^t\|^2 \\
&\lsim (\lambda_1\lambda_2)^{-1}\kt{4} (\Uneur)^4 \max(\ellr(\theta_t w^{t} + (1-\theta_t)w^{t+1}),1) \|\xi^t\|^2 & \mbox{by Claim~\ref{claim:losssmooth}} \\
&= (\lambda_1\lambda_2)^{-1}\kt{4} (\Uneur)^4 \max(\ellr([w_{-v}^0, w_v^{t} + \alpha(1-\theta_t)\xi^t]),1) \|\xi^t\|^2 \\
&\lsim (\lambda_1\lambda_2)^{-5}\kt{16}(\Uneur)^{16} (1 + \ellr(w^t)^4 + \|\alpha(1-\theta_t)\xi^t\|_{\infty}^4) \|\xi^t\|^2 & \mbox{by Claim~\ref{claim:boundparamschangemeansboundloss}}
\end{align*}
So 
\begin{align*}
|((\nabla^2_{w_v}& \ellr(\theta_t w^{t} + (1-\theta_t)w^{t+1})) \xi^{t}) \cdot \xi^{t}| (\lambda_1\lambda_2)^{5}\kt{-16}(\Uneur)^{-16} \\
&\lsim (1 + \ellr(w^t)^4 + \|\alpha\xi^t\|_{\infty}^4) \|\xi^t\|^2 \\ 
&\lsim  (1 + \ellr(w^t)^4 + (\alpha \estop)^4 + \|\alpha\nabla_{w_v} \ellr(w^t)\|_{\infty}^4) \|\xi^t\|^2 & \mbox{by $\Egradclose{t}$} \\
&\lsim  (1 + \ellr(w^t)^4 + \|\alpha\nabla_{w_v} \ellr(w^t)\|_{\infty}^4) \|\xi^t\|^2 & \mbox{by $\alpha, \estop \leq 1$} \\
&\lsim (1 + \ellr(w^t)^4 + (\alpha(\lambda_1\lambda_2)^{-3/2}\kt{4} (\Uneur)^3 \max(\ellr(w^t)^2,1))^4) \|\xi^t\|^2 &\mbox{by \cref{claim:ugradlr}} \\
&\lsim (1 + \ellr(w^t)^4) \|\xi^t\|^2,
\end{align*}
where the last line is by making the learning rate $\alpha$ small enough that it satisfies $\alpha \leq (\lambda_1\lambda_2)^{3/2}\kt{-4} (\Uneur)^{-3} \min(\ellr(w^t)^{-2},1)$.
So plugging these bounds back into Taylor's theorem:
\begin{align*}
\ellr(w^{t+1}) &\leq \ellr(w^{t}) - \alpha ((2/3) - \alpha C (\lambda_1\lambda_2)^{-5}\kt{16}(\Uneur)^{16}(1+\ellr(w^t)^4))\|\xi^{t}\|^2,
\end{align*}
where $C$ is some universal constant. Taking $\alpha < (\lambda_1\lambda_2)^{5} / (3C \kt{16}(\Uneur)^{16}(1+\ellr(w^0)^4)) \leq (\lambda_1\lambda_2)^{5} / (3C \kt{16}(\Uneur)^{16}(1+\ellr(w^t)^4))$, we conclude that
\begin{align*}
\ellr(w^{t+1}) &\leq \ellr(w^{t}) - \alpha \|\xi^{t}\|^2 / 3.
\end{align*}
\end{proof}

\subsection{Proof of \cref{lem:neuronsgdstationarity}}\label{sec:stationarityproof}

We now combine the above claims to prove the main result of this section: i.e., that $\NeuronSGD$ returns an approximate stationary point in a polynomial number of iterations.

\begin{proof}[Proof of \cref{lem:neuronsgdstationarity}]
Recall that $\tbound = \ceil{3\ellr(w^0) / (\alpha (\estop)^2)} + 1$. We make the following claim:
\begin{claim}\label{claim:egoodlowbound}
Under the setting of \cref{ass:neuronsgdconditions}, suppose that the learning rate $\alpha$ satisfies $$\alpha < 1 / (C (\lambda_1\lambda_2)^{-5}\kt{16}(\Uneur)^{16}(1+\ellr(w^0)^4))$$ and that the minibatch size is at least
$$B \geq C(\lambda_1\lambda_2)^{-3}\kt{8}(\Uneur)^4 (1 + \ellr(w^0)^4) \log(2\tbound / \dgrad) / \estop^2$$
for some large enough universal constant $C$. Then, for any $t \geq 0$ we have
$$\PP[\Egradclose{\min(t+1,T)} \mid \Egradclose{\min(t,T)}] \geq 1 - \dgrad/(\tbound+1).$$
\end{claim}
\begin{proof}
We split into two cases. If $T \leq t$:
$$\PP[\Egradclose{\min(t+1,T)} \mid \Egradclose{\min(t,T)} \mbox{ and } T \leq t] = \PP[\Egradclose{\min(t,T)} \mid \Egradclose{\min(t,T)} \mbox{ and } T \leq t] = 1.$$
Otherwise, if $T > t$, then 
\begin{align*}
\PP&[\Egradclose{\min(t+1,T)} \mid \Egradclose{\min(t,T)} \mbox{ and } T > t] \\ &= \PP[\|\xi^{t+1} - \nabla_{w_v} \ellr(w^{t+1})\| \leq \estop / 3 \mid \Egradclose{\min(t,T)} \mbox{ and } T > t] \\
&= \PP[\|\xi^{t+1} - \nabla_{w_v} \ellr(w^{t+1})\| \leq \estop / 3 \mid \Egradclose{\min(t,T)} \mbox{ and } T > t \mbox{ and } \ellr(w^{t+1}) \leq \ellr(w^0)] \\
&\geq 1 - \dgrad / (\tbound + 1)
\end{align*}
where in the second-to-last inequality we used \cref{claim:gradcloselossdecreases}, and in the last inequality we used \cref{claim:gradclose} and the fact that $\xi^{t+1}$ is independent of $w^{0},\ldots,w^{t}$ conditioned on $w^{t+1}$.
\end{proof}
Combining \cref{claim:egoodlowbound} with a union bound for all $t \in \{0,\ldots,\tbound\}$, and noting that \cref{claim:gradclose} implies $\PP[\Egradclose{0}] \geq 1 - \dgrad / (\tbound + 1)$, it follows that $$\PP[\Egradclose{\min(\tbound,T)}] \geq 1 - \dgrad.$$ We claim that if $\Egradclose{\min(\tbound,T)}$ holds, then we must have $T \leq \tbound$. Indeed, otherwise, the event $\Egradclose{\tbound}$ holds, so applying \cref{claim:gradcloselossdecreases} we must have $$\ellr(w^{\tbound+1}) \leq \ellr(w^0) - \alpha \sum_{t=0}^{\tbound} \|\xi^t\|^2 / 3 \leq \ellr(w^0) - \alpha \tbound (\estop)^2 / 3 < 0,$$ which is a contradiction because the loss cannot be negative. Therefore, we conclude that:
$$\PP[\Egradclose{T} \mbox{ and } T \leq \tbound] \geq 1 - \dgrad.$$
If $\Egradclose{T}$ holds, \cref{claim:gradcloselossdecreases} implies that $\ellr(w^T) \leq \ellr(w^0)$. Furthermore, under the event $\Egradclose{T}$ we must have $\|\xi^T\| \leq \estop$, so $\|\nabla_{w_v} \ellr(w^T)\| \leq \estop + \estop / 3 \leq 2\estop$.
\end{proof}

\section{Correctness of $\TrainNeuron$: learns product of inputs} \label{sec:trainneuronproof}

The main results of this section are \cref{lem:atmost1active,lem:2activenotprimed,lem:2activeprimedblank,lem:2activeprimedactive,lem:2activeprimedproblowbound}, which control how $\TrainNeuron$ (\cref{alg:trainneuron}) updates individual neurons during the training of the entire network. Because of the sparsity of the network, in this section we only reason about how $\TrainNeuron$ updates neurons with at most two active inputs. These will be the only results that will be needed to prove correctness of $\TrainNetworkLayerwise$ in \cref{sec:trainnetworklayerwiseproof}. We also assume that each of the neurons in the previous layers is either blank (i.e., always computes zero), or it represents a monomial $\chi_S(x)$ up to some small relative error, since this will hold true inductively on the training iteration.

Suppose that we train an initially blank neuron $v$ by running $\TrainNeuron(v,w^0) := \TrainNeuron(v,w^0; \lambda_1,\lambda_2,\eta,B,\estop,\alpha,\tau)$. If $v$ has at most one active parent, then we prove in \cref{lem:atmost1active} that with high probability $v$ remains blank after training. This is because by analyzing the stationary points of the loss one can see that the $L_2$ regularization term sends the weights of the input edges to $v$ to close to zero, and these are rounded to exactly zero in \cref{step:truncate} of $\TrainNeuron$.

If instead $v$ has two active parents $u_1$ and $u_2$, then the situation is more delicate. Suppose in this case that $u_1$ approximately computes a monomial $\chi_{S_1}(x)$, and $u_2$ approximately computes a monomial $\chi_{S_2}(x)$. We prove that training the neuron $v$ allows it to approximately compute the product of these two inputs: i.e., the monomial $\chi_{S_1}(x)\chi_{S_2}(x)$. If the error function does not have a large component in the direction of $\chi_{S_1}(x) \chi_{S_2}(x)$, then the $L_2$ regularization will again prevail and send the input weights to zero, and $v$ will remain blank after training (proved in \cref{lem:2activenotprimed}). On the other hand, if the error function does have a large component in the direction of $\chi_{S_1}(x) \chi_{S_2}(x)$ then the regularization will be relatively insignificant to the decrease in the loss from learning $\chi_{S_1}(x)\chi_{S_2}(x)$, and so with lower-bounded probability the neuron $v$ will learn to approximately compute $\chi_{S_1}(x) \chi_{S_2}(x)$ (proved in \cref{lem:2activeprimedblank,lem:2activeprimedactive,lem:2activeprimedproblowbound}). Thus, training neuron $v$ computes a monomial equal to the product of monomials computed by neurons at lower depth only if it significantly decreases the loss, and so this ensures that a bounded number of neurons in the network are active during training.

We also note that an obstacle to applying $\TrainNeuron$ to train the network layerwise is the possible exponential error blow-up along the depth in the approximation of each neuron computed by the monomials. In order to overcome this obstacle, we must carefully bound the blow-up in the relative error of the new neuron created. For this, we roughly prove in \cref{lem:2activeprimedactive} that if neuron $u_1$ has relative error $\eps_1$, and neuron $u_2$ is an input in $\Vin$ and therefore has relative error $\eps_2 = 0$, then the new neuron trained will have relative error at most $$\enewrel = \eps_1 (1 + O(\sqrt{\frac{\lambda_1}{\lambda_2}} \kt{O(1)}) + \sqrt{\lambda_1\lambda_2} \kt{O(1)}.$$ By taking the ratio of $\lambda_1$ and $\lambda_2$ sufficiently small, it holds that $$\enewrel = \eps_1 ( 1 + O(1/L)) + \sqrt{\lambda_1\lambda_2} \kt{O(1)},$$ so the relative error of a neuron can blow up to at most $\sqrt{\lambda_1\lambda_2}\kt{O(1)}$ by the $L$th layer. This is very small if we take $\sqrt{\lambda_1\lambda_2}$ sufficiently small, and so the relative error of the neurons is controlled throughout training.

\subsection{At most two active inputs: assumption and notation}
Formally, the following assumption is shared by the main results of this section:

\begin{assumption}[At most two parent vertices are active]\label{ass:atmosttwoactive}
$\TrainNeuron$ is run with a neuron $v \in V \sm \Vin$, and a parameter initialization $w^0$, such that the following hold:
\begin{itemize}
    \item There are two parent vertices $u_1,u_2 \in P_v$ as well as constants $r_1,r_2,\eps_1,\eps_2 \in \RR$, sets $S_1, S_2 \subset [n]$ and functions $h_1,h_2 : \{-1,1\}^n \to \RR$ such that for each $i \in \{1,2\}$ and $x \in \{-1,1\}^n$ the following holds:
    $$f_{u_i}(x;w^0) = r_i \chi_{S_i}(x) + h_i(x)\mbox{ and } |h_i(x)| \leq |r_i| \eps_i,$$
    and $\hat{h}_i(S_i) = \EE_{x \sim \{-1,1\}^n} [h_i(x) \chi_{S_i}(x) ] = 0$ for each $i \in \{1,2\}$.
    \item On the other hand, for any vertex $u' \in P_v \sm \{u_1,u_2\}$, for all $x \in \{-1,1\}^n$ we have $$f_{u'}(x;w^0) = 0.$$
    \item The neuron $v$ is blank at initialization (i.e., all input and output weights and the bias associated with $v$ are zero):
$$a_{e}^0 = 0 \mbox{ for all } e \in E \mbox{ such that } v \in e, \mbox{ and } b_v^0 = 0.$$
    \item We use $\gamma_i \in \{\lambda_1, \lambda_2\}$ to denote the regularization parameter associated with $a_{(u_i, v)}$. Namely, $\gamma_i = \lambda_1$ if $u_i \in \Vin$ and $\gamma_i = \lambda_2$ otherwise.
    \item We use $S = S_1 \cup S_2 \sm (S_1 \cap S_2)$ to denote the symmetric difference between $S_1$ and $S_2$.
    \item We write the error at parameters $w$ as:
$$\zeta(x; w) = f(x; w) - g(x),$$
and its Fourier coefficients for $S \subset [n]$ is:
$$\hat{\zeta}(S;w) = \EE_{x \sim \{-1,1\}^n} [\zeta(x;w) \chi_S(x)].$$

\end{itemize}
\end{assumption}

Under \cref{ass:atmosttwoactive}, we may decompose the function learned during training as follows:
\begin{claim}[Decomposition of learned function]
Suppose that \cref{ass:atmosttwoactive} holds, and write and write $e_i = (u_i,v) \in E$ for each $i \in \{1,2\}$ for shorthand. For any setting of parameters $w_v = \{a_e\}_{e = (u,v) \in E} \cup \{b_v\}$, we have
\begin{align}
f(x; [w_{-v}^0, w_v]) &= f(x; w^0) + f_v(x; [w_{-v}^0, w_v]) \label{eq:decompsimple}\\ 
&= f(x; w^0) + \left(\sum_{i \in [2]} a_{e_i} f_{u_i}(x;w^0) \right)^2 + b_v \\
&= f(x; w^0) + \left(\sum_{i \in [2]} a_{e_i} (r_i \chi_{S_i}(x) + h_i(x))\right)^2 + b_v. \label{eq:decompexpanded}
\end{align}
\end{claim}
\begin{proof}
The first line follows from the definition of $f$ in \cref{ssec:architecture}, using that $a_e^0 = 0$ for all $e = (v,u) \in E$. The second line uses that $f_{u'}(x;w^0) = 0$ for all $(u',v) \in E$ such that $u' \not\in \{u_1,u_2\}$. The third line uses that $f_{u_i}(x;w^0) = r_i \chi_{S_i}(x) + h_i(x)$.
\end{proof}

\subsection{Reduction to analyzing the idealized loss}
The main technical challenge in \cref{sec:trainneuronproof} is to analyze the approximate stationary points of the loss function $\ellr([w_{-v}^0, w_v])$ with respect to $w_v$. In order to do this, we introduce an ``idealized loss function'', which will be a close approximation to the true loss. Let $w = [w_{-v}^0,w_v]$. If $S_1 \neq S_2$, the idealized loss function is defined as:
\begin{align*}\ellt(w) &= \frac{1}{2}(2r_1r_2a_{e_1}a_{e_2} + \hat{\zeta}(S;w^0))^2 + \frac{1}{2}((r_1 a_{e_1})^2 + (r_2a_{e_2})^2 + b_v + \hat{\zeta}(\emptyset;w^0))^2 \\
&\quad\quad + \frac{1}{2}\sum_{\substack{S' \subset [n] \\ S' \neq \emptyset, S}} (\hat{\zeta}(S';w^0))^2
\end{align*}
And if $S_1 = S_2$, it is defined as:
\begin{align*}
\ellt(w) &= \frac{1}{2}((r_1a_{e_1} + r_2 a_{e_2})^2 + b_v + \hat{\zeta}(\emptyset;w^0))^2 + \frac{1}{2}\sum_{\substack{S' \subset [n] \\ S' \neq \emptyset}} (\hat{\zeta}(S';w^0))^2.\end{align*}
Similarly, define the regularized version:
$$\elltr(w) = \ellt(w) + R(w).$$
As we will see below, $\ellt$ is the loss function that would arise if we had $h_1(x) = h_2(x) = 0$ (i.e., if all of the parents to vertex $v$ computed a monomial noiselessly).
We prove in \cref{lem:idealizedlossgradclose} that $\ellt$ is close to the true unregularized loss $\ell$ and that the gradients of $\ellt$ with respect to $w_v$ are close to the gradients of $\ell$ with respect to $w_v$. The benefit of this result is that in the proofs we may analyze the stationary points of the simpler loss $\ellt$ instead of the actual loss $\ell$.

\begin{lemma}\label{lem:idealizedlossgradclose}
Suppose \cref{ass:atmosttwoactive} holds on the initialization $w^0$. Then for any parameter vector $w_v = \{a_e\}_{(u,v) \in E} \cup \{b_v\}$, and letting $w = [w_{-v}^0, w_v]$, we have 
\begin{align}|\ellt(w) - \ell(w)| \lsim \max_{u \in V} \max_{x \in \{-1,1\}^n} (\|w_v\|_{\infty}^4 + 1) (|f_u(x;w^0)|^3 + 1) (\max_{i \in \{1,2\}} \eps_i), \label{eq:idealizedlossclose}\end{align}
\begin{align}\|\nabla_{w_v} \ell(w) - \nabla_{w_v}\ellt(w)\|_{\infty} \lsim \max_{u \in V} \max_{x \in \{-1,1\}^n} (\|w_v\|_{\infty}^3 + 1) (|f_u(x;w^0)|^3 + 1)(\max_{i \in \{1,2\}} \eps_i).\label{eq:idealizedlossgradclose}\end{align}
\end{lemma}
\begin{proof}
First, for any $x \in \{-1,1\}^n$, define the functions $\Upsilont$ and $\Upsilon$:
\begin{align*}
\Upsilont(x; w) &= \left(\sum_{i \in [2]} a_{e_i} r_i \chi_{S_i}(x)\right)^2 + b_v + \zeta(x) \\
\Upsilon(x; w) &= \left(\sum_{i \in [2]} a_{e_i} (r_i \chi_{S_i}(x) + h_i(x))\right)^2 + b_v + \zeta(x)
\end{align*}
The reason for these definitions is that we may write the idealized and actual loss functions in terms of $\Upsilont$ and $\Upsilon$, respectively. First, by Parseval's theorem on the Boolean hypercube, the idealized loss function is:
\begin{align*}
\ellt(w) &= \frac{1}{2} \EE_{x \sim \{-1,1\}^n}\big[\big( (2r_1r_2a_{e_1}a_{e_2} + \hat{\zeta}(S;w^0)) \chi_S(x) + (r_1^2(a_{e_1})^2 + b_v + \hat{\zeta}(\emptyset;w^0)) \\
& \qquad\qquad\qquad\qquad\qquad + \sum_{\substack{S' \subset [n] \\ S' \neq \emptyset, S}} \hat{\zeta}(S';w^0) \chi_S(x)\big)^2\big] \\
&= \frac{1}{2}\EE_x[\Upst(x; w)^2].
\end{align*}
Furthermore, the actual loss function may be written as:
\begin{align*}
\ell(w) &= \EE_{x \sim \{-1,1\}^n}[\ell(x;w)] \\
&= \frac{1}{2}\EE_{x \sim \{-1,1\}^n}\left[\left(\left(\sum_{i \in [2]} a_{e_i} (r_i \chi_{S_i}(x) + h_i(x))\right)^2 + b_v + \zeta(x)\right)^2\right] &\mbox{by \cref{eq:decompexpanded}} \\
&= \frac{1}{2}\EE_x[\Ups(x;w)^2].
\end{align*}
We bound $\ellt(w) - \ell(w)$ by bounding $\Upsilont$ and $\Upsilont - \Upsilon$ pointwise for any $x \in \{-1,1\}^n$. First,
\begin{align*}
|\Upst(x; w)| &\leq (2 \max_i |a_{e_i} r_i|)^2 + |b_v| + |\zeta(x;w^0)| \\
&\leq 4 \max_i |a_{e_i}|^2 (\max_u |f_u(x;w^0)| + \eps_i)^2 + |b_v| + |\zeta(x;w^0)| \\
&\lsim (\|w_v\|_{\infty}^2 + 1) (\max_u |f_u(x;w^0)| + 1)^2 + |\zeta(x;w^0)| \\
&\leq (\|w_v\|_{\infty}^2 + 1) (\max_u |f_u(x;w^0)| + 1)^2 + (Ms + WL\max_{u} |f_u(x;w^0)|) \\
&\lsim \kt{} (\|w_v\|_{\infty}^2 + 1) (\max_u |f_u(x;w^0)|^2 + 1) := U_1(x).
\end{align*}
Let us compare $\Ups$ to $\Upst$:
\begin{align*}
|\Ups(x;w) - \Upst(x;w)| &= \left|\left(\sum_{i \in [2]} a_{e_i} (r_i \chi_{S_i}(x) + h_i(x))\right)^2 - \left(\sum_{i \in [2]} a_{e_i} r_i \chi_{S_i}(x)\right)^2\right| \\
&\lsim (\max_i |a_{e_i}|)^2 \left(|r_1| + |r_2| + |\eps_1| + |\eps_2|)(|\eps_1| + |\eps_2|\right) \\
&\lsim \|w_v\|_{\infty}^2 (\max_u |f_u(x;w^0)| + 1)(\max_{i \in \{1,2\}} \eps_i) := U_2(x).
\end{align*}
Of course, by the triangle inequality we also have:
\begin{align*}
|\Ups(x;w)| \leq U_1(x) + U_2(x) := U_3(x).
\end{align*}
This lets us prove the first bound in the claim:
\begin{align*}
|\ellt(w) - \ell(w)| &= \left|\EE_{x}[\Upsilont(x;w)^2 - \Upsilon(x;w)^2]\right| \\
&\leq \EE_{x}[\left|\Upsilont(x;w)^2 - \Upsilon(x;w)^2\right|] \\
&\lsim \EE_{x}[\left|\Upsilont(x;w) + \Upsilon(x;w)\right| \left|\Upsilont(x;w) - \Upsilon(x;w)\right|] \\
&\lsim \EE_x[(U_1(x) + U_3(x))(U_2(x))] \\
&\lsim \max_{u \in V} \max_{x \in \{-1,1\}^n} \kt{} (\|w_v\|_{\infty}^4 + 1) (|f_u(x;w^0)|^3 + 1) (\max_{i \in \{1,2\}} \eps_i).
\end{align*}
For the second part of the claim, we bound the gradient of $\ellt - \ell$. For this, let us first bound and compare the gradients of $\Upst$ and $\Ups$:
\begin{align*}
\|\nabla_{w_v} \Upsilont(x;w)\|_{\infty} &\lsim \max(1, \max_{i \in \{1,2\}} |a_{e_i} r_i| \cdot \max_{i' \in \{1,2\}} |r_i|) \\
&\lsim (\|w_v\|_{\infty} + 1) \cdot (\max_{u \in V} |f_u(x;w)|^2 + 1) := U_4(x).
\end{align*}
Let us compare $\nabla_{w_v} \Upst(x;w)$ to $\nabla_{w_v} \Ups(x;w)$:
\begin{align*}
    \|\nabla_{w_v} \Upst(x;w) - \nabla_{w_v} \Ups(x;w)\|_{\infty} &\lsim \left(\max_{i \in \{1,2\}} |a_{e_i}|\right) \cdot (|r_1| + |r_2| + |\eps_1| + |\eps_2|)(|\eps_1| + |\eps_2|) \\
    &\lsim \|w_v\|_{\infty} (\max_u |f_u(x; w^0)| + 1) (\max_{i \in \{1,2\}} \eps_i) := U_5(x).
\end{align*}
And by triangle inequality we have:
\begin{align*}
\|\nabla_{w_v} \Upsilon(x;w)\|_{\infty} &\leq U_4(x) + U_5(x) := U_6(x).
\end{align*}
The above bounds may be combined to prove that the gradient of the true loss is close to the gradient of the idealized loss:
\begin{align*}
\|\nabla_{w_v} \ell(w) - \nabla_{w_v} \ellt(w)\|_{\infty} &= \|\EE_x[\Upsilon(x;w) (\nabla_{w_v} \Upsilon(x;w)) - \Upsilont(x;w) (\nabla_{w_v} \Upsilont(x;w))]\|_{\infty} \\
&= \|\EE_x[(\Upsilon(x;w) - \Upsilont(x;w)) (\nabla_{w_v} \Upsilon(x;w)) \\
& \qquad\qquad - \Upsilont(x;w) (\nabla_{w_v} \Upsilont(x;w) - \nabla_{w_v} \Upsilon(x;w))]\|_{\infty} \\
&\leq \max_x U_2(x)U_6(x) + U_1(x)U_5(x) \\
&\lsim \max_{u \in V} \max_{x \in \{-1,1\}^n} (\|w_v\|_{\infty}^3 + 1) (|f_u(x;w^0)|^3 + 1)(\max_{i \in \{1,2\}} \eps_i).
\end{align*}
\end{proof}

\subsection{Approximate stationarity of $\wsgd$, and loss does not increase}

In this subsection, we prove that with large enough minibatch size $B$ and small enough learning rate, with high probability the vector $\wsgd$ computed in \cref{step:callneuronsgd} of $\TrainNeuron$ (i) is an approximate stationary point of the idealized loss $\ellt([w_{-v}^0, w_v])$ with respect to the parameters $w_v$, and (ii) satisfies $\ellr(\wsgd) \leq \ellr(\wperturb)$. This is proved by appealing to the guarantees for $\NeuronSGD$ in \cref{lem:neuronsgdstationarity} and the fact proved in \cref{lem:idealizedlossgradclose} that the idealized loss $\ellt$ and the true loss $\ell$ are close. First we prove a helper lemma bounding $\ellr(\wperturb)$ and $\max_{u \in V} \max_{x \in \{-1,1\}^n} |f_u(x;\wperturb)|$.
\begin{claim}\label{claim:wperturbbounds}
Under \cref{ass:atmosttwoactive}, the following bounds are satisfied:
$$\max_{u \in V} \max_{x \in \{-1,1\}^n} |f_u(x;\wperturb)| \lsim \max_{u \in V \sm \{v\}} |f_u(x; w^0)|^2 + 1, \mbox{ and }$$
$$\ellr(\wperturb) \lsim \kt{2}(\max_u |f_u(x; w^0)|^4 + 1) + \ellr(w^0).$$
\end{claim}
\begin{proof}
First, note that
$w_v^0 = 0$ by \cref{ass:neuronsgdconditions}. So since the noise added at \cref{step:perturb} has each entry in $\mathrm{Unif}[-\eta,\eta]$, we must have $\|\wvperturb\|_{\infty} \leq \eta$. This is the input to the call of $\NeuronSGD$ in \cref{step:callneuronsgd} of $\TrainNeuron$, and because of \cref{lem:neuronsaftertrainingwv} it satisfies
\begin{align*}
\max_{x \in \{-1,1\}^n} |f_v(x; \wperturb)| &\leq \eta + (2\eta \max_{u \in V} \max_{x \in \{-1,1\}^n} |f_u(x; w^0)|)^2 \\
&\lsim (\max_{u \in V} \max_{x \in \{-1,1\}^n} |f_u(x; w^0)|)^2 + 1
\end{align*}
Therefore,
\begin{align*}
\max_{u \in V} \max_{x \in \{-1,1\}^n} |f_u(x; \wperturb)| &\leq \max(\max_{u \in V \sm \{v\}} |f_u(x; w^0)|, |f_v(x; \wperturb)|) \\
&\lsim (\max_{u \in V \sm \{v\}} |f_u(x; w^0)|)^2 + 1.
\end{align*}
Furthermore, by splitting the loss into the unregularized part and the regularization terms:
\begin{align*}
&\ellr(\wperturb) \\
&\leq \ell(\wperturb) + \frac{\max(\lambda_1,\lambda_2)}{2}(2W \eta^2 + \sum_{e \in E} |a_e^0|^2)   \\
&\leq \max_x \frac{1}{2}(g(x) - WL \max_u f_u(x; \wperturb))^2 + \frac{\max(\lambda_1,\lambda_2)}{2}(2W \eta^2 + \sum_{e \in E} |a_e^0|^2) &\mbox{by \cref{claim:gisupperbounded}} \\
&\lsim \kt{2}(\max_u |f_u(x; w^0)|^4 + 1) + \frac{\max(\lambda_1,\lambda_2)}{2}(2W \eta^2 + \sum_{e \in E} |a_e^0|^2) \\
&\lsim \kt{2}(\max_u |f_u(x; w^0)|^4 + 1) + W \eta^2 + \ellr(w^0) \\
&\lsim \kt{2}(\max_u |f_u(x; w^0)|^4 + 1) + \ellr(w^0).
\end{align*}
\end{proof}

The main result of the subsection may now be stated and proved:
\begin{lemma}\label{lem:trainneuronstationarity}
Consider running $\TrainNeuron(v,w^0)$ (\cref{alg:trainneuron}), where the assumptions \cref{ass:atmosttwoactive} hold. There is a large enough constant $C'$ such that for any $\delta > 0$, if we define
$$\tbound = C'(\kt{2}(\max_{u \in V} \max_{x \in \{-1,1\}^n} |f_u(x; w^0)|^4 + 1) + \ellr(w^0)) / (\alpha (\estop)^2),$$
and if the minibatch size is at least
$$B \geq \max_{u \in V} \max_{x \in \{-1,1\}^n} C' (\lambda_1\lambda_2)^{-3}\kt{8} (|f_u(x; w^0)|^8 + 1) (\kt{8}(|f_u(x; w^0)|^{16} + 1) + \ellr(w^0)^4) \log(2\tbound / \delta) / \estop^2,$$
and if the learning rate is at most
$$\alpha < \min_{u \in V} \min_{x \in \{-1,1\}^n} 1 / (C'(\lambda_1\lambda_2)^{-5}\kt{16}( |f_u(x; w^0)|^{32} + 1) (\kt{8} (|f_u(x;w^0)|^{16} + 1) + \ellr(w^0)^4)),$$
then $\PP[\Estat] \geq 1 - \delta$, where $\Estat$ is the event that the following hold:
\begin{enumerate}
    \item The loss at $\wsgd$ is not larger than the loss at \label{item:trainneuron1} $\wperturb$:
    $$\ellr(\wsgd) \leq \ellr(\wperturb)$$
    \item The parameters $\wsgd$ are an approximate stationary point with respect to $w_v$: \label{item:trainneuronstationary}
    $$\|\nabla_{w_v} \ellr(\wsgd)\|_{\infty} \leq 2\estop$$
    $$\|\nabla_{w_v} \elltr(\wsgd)\|_{\infty} \leq \estat(w^0,\eps_1,\eps_2) = \estat,$$
    where $$\estat = 2\estop + C'\max_{u \in V} \max_{x \in \{-1,1\}^n} (\lambda_1\lambda_2)^{-3}\kt{12}(\ellr(w^0)^3 + 1)(|f_u(x;w^0)|^{30} + 1)(\max_{i \in \{1,2\}} \eps_i)$$
    \item The call to $\TrainNeuron$ runs in time $O(\kt{} B \tbound)$. \label{item:trainneuron3}
    \item We have the following bound on the returned parameters:
    $$\|\wvsgd\|_{\infty} \leq \Ustat(w^0) = \Ustat,$$
    where
    $\Ustat = C'(\lambda_1\lambda_2)^{-1}\kt{4} (\max_u \max_{x \in \{-1,1\}^n} |f_u(x;w^ 0)|^{8} + 1)(\ellr(w^0) + 1)$.
    \label{item:trainneuron4}
\end{enumerate}
\end{lemma}
\begin{proof}
The proof is by plugging the bounds of Claim~\ref{claim:wperturbbounds} into \cref{lem:neuronsgdstationarity}, which provides guarantees for $\NeuronSGD$.

Let $C$ be the constant from \cref{lem:neuronsgdstationarity}. For large enough constant $C'$, we have
\begin{align*}
\tbound \geq 3\ellr(\wperturb) / (\alpha (\estop)^2),
\end{align*}
\begin{align*}
B &\geq C' (\lambda_1\lambda_2)^{-3}\kt{8} (\max_{u \in V \sm \{v\}} |f_u(x; w^0)|^8 + 1) (\kt{8}(\max_u |f_u(x; w^0)|^{16} + 1) + \ellr(w^0)^4) \log(2\tbound / \delta) / \estop^2\\
&\geq C (\lambda_1\lambda_2)^{-3}\kt{8} (\max_{u \in V \sm \{v\}} |f_u(x; w^0)|^2 + 1)^4 (1 + (\kt{2}(\max_u |f_u(x; w^0)|^4 + 1) + \ellr(w^0))^4) \log(2\tbound / \delta) / \estop^2 \\
&\geq C (\lambda_1\lambda_2)^{-3}\kt{8} (\max_u |f_u(x;\wperturb)|^4) (1 + \ellr(\wperturb)^4) \log(2\tbound / \delta) / \estop^2,\end{align*}
and, similarly,
\begin{align*}
\alpha \leq 1/(C(\lambda_1\lambda_2)^{-5}\kt{16}(\max_{u \in U} \max_{x \in \{-1,1\}^n} |f_u(x; \wperturb)|)^{16} (1 + \ellr(\wperturb)^4)).
\end{align*}

In particular, the bounds in $\NeuronSGD$ hold with probability at least $1 - \delta$. Let $\Estat$ be the event that they hold. Under $\Estat$, \cref{item:trainneuron1} of the lemma immediately follows. Furthermore, since the $\NeuronSGD$ method runs for at most $\tbound$ iterations and each iteration takes at most $B\kappa$ time, \cref{item:trainneuron3} follows. Finally, since 
$$\|\nabla_{w_v} \ellr(\wsgd)\| \leq 2\estop,$$ we conclude that for some large enough constant $C''$ we have
\begin{align*}
\|\nabla_{w_v}& \elltr(\wsgd)\|_{\infty} \\
&\leq 2\estop + \|\nabla_{w_v} \elltr(\wsgd) - \nabla_{w_v} \ellr(\wsgd)\|_{\infty} \\
&\leq 2\estop + C''\max_{u \in V} \max_{x \in \{-1,1\}^n} (\|\wvsgd\|^3_{\infty} + 1)(|f_u(x;\wperturb)|^3 + 1)(\max_{i \in \{1,2\}} \eps_i), & \mbox{by \cref{lem:idealizedlossgradclose}}.
\end{align*} 
This may be further bounded by noting that by \cref{claim:upperboundedparamsbyupperboundedloss} and \cref{claim:wperturbbounds},
\begin{align*}
\|\wvsgd\|_{\infty} &\lsim (\lambda_1\lambda_2)^{-1}\kt{2} (\max_u \max_{x \in \{-1,1\}^n} |f_u(x;\wperturb)|^2 + 1)(\ellr(\wsgd) + 1) \\
&\lsim (\lambda_1\lambda_2)^{-1}\kt{2} (\max_u \max_{x \in \{-1,1\}^n} |f_u(x;\wperturb)|^2 + 1)(\ellr(\wperturb) + 1) \\
&\lsim (\lambda_1\lambda_2)^{-1}\kt{4} (\max_u \max_{x \in \{-1,1\}^n} |f_u(x;w^ 0)|^{8} + 1)(\ellr(w^0) + 1).\end{align*}
Thus, for large enough constant $C'''$ and assuming $C'$ is also large enough, by again applying \cref{claim:wperturbbounds},
\begin{align*}
\|\nabla_{w_v}& \elltr(\wsgd)\|_{\infty} \\
&\leq 2\estop + C'\max_{u \in V} \max_{x \in \{-1,1\}^n} (\lambda_1\lambda_2)^{-3}\kt{12}(\ellr(w^0)^3 + 1)(|f_u(x;w^0)|^{30} + 1)(\max_{i \in \{1,2\}} \eps_i).
\end{align*}
In the above, the second inequality follows from applying \cref{claim:wperturbbounds}. This proves \cref{item:trainneuronstationary}, concluding the proof of the lemma.
\end{proof}

In the subsequent proofs of this section, for brevity of notation write $\estat = \estat(w^0,\eps_1,\eps_2)$ and $\Ustat = \Ustat(w^0)$.

\subsection{Blank input weights are trained to zero}

Before proving the main lemmas in this section, let us prove one last helper claim, which states that the parameters on which $f_v$ does not depend are set to zero by $\TrainNeuron$.
\begin{claim}[Blank neuron weights are zero]\label{claim:notdepend}
Under \cref{ass:atmosttwoactive}, and if the event $\Estat$ of \cref{lem:trainneuronstationarity} holds, and if
\begin{align}
\tau > 2\estat / \min(\lambda_1,\lambda_2) := \eps^{(1)}, \label{eq:notdependcond}
\end{align}
then for all $e = (u,v) \in E$ such that $f_u(x;w^0) = 0$ for all $x$ (i.e., parents $u$ that are blank at initialization), it holds that $\aeret = 0$.
\end{claim}
\begin{proof}
Recall that $\wsgd = [w_{-v}^0, \wvsgd]$ (i.e., all parameters except for the parameters to neuron $v$ are frozen during training). For any $e = (u,v) \in E$ such that $u$ is blank, the derivative of the unregularized loss at $x \in \{-1,1\}^n$ with respect to $a_e$ is:
\begin{align*}\pd{\ell(x; \wsgd)}{a_e} &= \pd{}{a_e} (\frac{1}{2} (f(x; w^0) + f_v(x; \wsgd) - g(x))^2), &\mbox{by \cref{eq:decompsimple}} \\
&= (f(x; w^0) + f_v(x; \wsgd) - g(x)) \cdot \pd{}{a_e} f_v(x; \wsgd) \\
&= (f(x; w^0) + f_v(x; \wsgd) - g(x)) \cdot 0 = 0 &\mbox{since $f_u(x;w^t) = 0$}
\end{align*}
Therefore $$|\pd{\ellr(w^t)}{a_e}| \geq \min(\lambda_1,\lambda_2)|a_e^t| - |\pd{\ell(w^t)}{a_e}| = \min(\lambda_1,\lambda_2)|a_e^t|.$$
So in particular $$|a_e^t| \leq 2\estop / \min(\lambda_1,\lambda_2) \leq \estat / \min(\lambda_1,\lambda_2) < \tau,$$ so by the truncation step of \cref{step:truncate}, the algorithm returns trained weights $\wret$ with $\aeret = 0$.
\end{proof}

\subsection{$\TrainNeuron$ correctness : training a neuron with at most one active input (\cref{lem:atmost1active})}
We may now state and prove the first main result of this section -- i.e., if a neuron with at most one active input is trained, then it remains blank after training.

\begin{lemma}[$\TrainNeuron$ correctness: at most one active input]\label{lem:atmost1active}
Suppose that \cref{ass:atmosttwoactive} holds, the event $\Estat$ from \cref{lem:trainneuronstationarity} holds, and also $r_2 = \eps_2 = 0$ (i.e., neuron $u_2$ is blank). Suppose also that
\begin{align}
\tau > 2\estat / \min(\lambda_1,\lambda_2) := \eps^{(1)}, \label{eq:1acond1}
\end{align}
\begin{align}
\tau > |\hat{\zeta}(\emptyset;w^0)| + \estat + (r_1)^2|2\estat / \min(\lambda_1,\lambda_2)|^2, \label{eq:1acond2}
\end{align}and
\begin{align}
\estat < \min(\lambda_1,\lambda_2) / (2 r_1)^2. \label{eq:1acond3}
\end{align}
Then after running $\TrainNeuron(v,w^0)$, we have $\wret = w^0$, so the weights do not change during training and the neuron $v$ remains blank.
\end{lemma}
\begin{proof}
All neurons $u \in P_v \sm \{u_1\}$ are blank at the initialization $w^0$ by the assumptions in the lemma statement (for the case of $u = u_2$, this follows because because $r_2 = \eps_2 = 0$, so $f_{u_2}(x;w^0) = 0$ for all $x \in \{-1,1\}^n$). Therefore, by \cref{claim:notdepend} and \cref{eq:1acond1}, for edge $e = (u,v)$ the algorithm $\TrainNeuron$ returns weight $\aeret = 0$.

Now consider the parameters $b_v$ and $a_{e_1}$. We compute the partial derivatives of the idealized loss:
$$\pd{\ellt}{b_v} = ((r_1a_{e_1})^2 + b_v - \hat{\zeta}(\emptyset;w^0)) \mbox{ and }\pd{\ellt}{a_{e_1}} = ((r_1a_{e_1})^2 + b_v - \hat{\zeta}(\emptyset;w^0)) (2 (r_1)^2 a_{e_1}).$$
Since the event $\Estat$ holds, by \cref{item:trainneuronstationary} of \cref{lem:trainneuronstationarity}, we have
$$\|\nabla_{w_v} \elltr(\wsgd)\|_{\infty} \leq \estat,$$ which implies
$$|(r_1a_{e_1}^{SGD})^2 + b_v^{SGD} - \hat{\zeta}(\emptyset;w^0)| \leq \estat,$$
and $$(\gamma_1 - 2(r_1)^2|(r_1a_{e_1}^{SGD})^2 + b_v^{SGD} - \hat{\zeta}(\emptyset;w^0)|) |a_{e_1}^{SGD}| \leq \estat,$$
which means that
$$(\gamma_1 - 2(r_1)^2 \estat) |a_{e_1}^{SGD}| \leq \estat,$$
so, by \cref{eq:1acond3},
$$|a_{e_1}^{SGD}| \leq \estat / (\gamma_1 / 2) = 2\estat / \gamma_1,$$
and hence
$$|b_v^{SGD}| \leq |\hat{\zeta}(\emptyset;w^0)| + \estat + (r_1)^2|2\estat / \gamma_1|^2.$$
Thus, in \cref{step:truncate} of $\TrainNeuron$ since $\tau > \max(2\estat / \gamma_1, |\hat{\zeta}(\emptyset;w^0)| + \estat + (r_1)^2|2\estat / \gamma_1|^2)$ by \cref{eq:1acond1} and \cref{eq:1acond2}, we have $\bvret = 0$ and $\aeoneret = 0$. So overall we have $\wvret = \vec{0}$ for all parameters, so $\wret = [w_{-v}^0, \wvret] = [w_{-v}^0, \vec{0}] = w^0$, and the neuron remains blank.
\end{proof}

\subsection{$\TrainNeuron$ correctness: training a neuron with two active inputs whose product is not useful (\cref{lem:2activenotprimed})}

The next main result of this section is the correctness of $\TrainNeuron$ in the case in which both $u_1$ and $u_2$ are active neurons, but the monomial $\chi_S(x)$ that is approximately computed by their product only has low correlation with the error function $\zeta(x;w^0)$. In this case learning the product of the active inputs would not significantly decrease the loss, and the $L_2$ regularization on the weights dominates. Thus the neuron remains blank after training because of the rounding step in \cref{step:truncate} of $\TrainNeuron$.

\begin{lemma}[$\TrainNeuron$ correctness: two active inputs, product not useful]\label{lem:2activenotprimed}
Define
$$\eps^{(1)} := 2\estat / \min(\lambda_1,\lambda_2)$$
$$\eps^{(2)} = (1 + \max_{i \in \{1,2\}} r_i^2 \Ustat)\estat$$ 
$$\eps^{(3)} := 8\eps^{(2)} (1 + (\Ustat)^2 + |\hat{\zeta}(S;w^0)|)\max(1,|r_1r_2|^2) / \min(\lambda_1,\lambda_2)^2$$
$$\eps^{(4)} := 2(\sqrt{\frac{\max(\lambda_1,\lambda_2)}{ \min(\lambda_1,\lambda_2)}}\sqrt{\max(\lambda_1, \lambda_2) + |\hat{\zeta}(S;w^0)|}  + \eps^{(2)}\max(\lambda_1,\lambda_2)) / \min(1,|r_1r_2|^2)$$
Suppose that \cref{ass:atmosttwoactive} holds and that event $\Estat$ from \cref{lem:trainneuronstationarity} holds. Suppose also that
\begin{align}\label{eq:2anpcond1}
\tau > \max(\eps^{(1)},\eps^{(3)},\eps^{(4)})
\end{align}
\begin{align}\label{eq:2anpcond2}
\tau > |\hat{\zeta}(\emptyset; w^0)| + 2(r_1^2 + r_2^2)(\max(\eps^{(1)},\eps^{(3)},\eps^{(4)}))^2 + \estat
\end{align}
\begin{align}
4(r_1^2 + r_2^2) \estat  / \min(\lambda_1, \lambda_2) \leq 1/2 \label{eq:2anpcond5}
\end{align}
Then after running $\TrainNeuron(v,w^0)$, the weights are not changed during training (i.e., $\wret = w^0$) and so the neuron $v$ remains blank.
\end{lemma}

Before proving this lemma, let us prove a helper claim:
\begin{claim}\label{claim:firstordercond2active}
Suppose that \cref{ass:atmosttwoactive} and the event $\Estat$ from \cref{lem:trainneuronstationarity} both hold, and also that $S_1 \neq S_2$. Define 
\begin{align}
\rho = (2r_1r_2 a_{e_1}^{SGD} a_{e_2}^{SGD} + \hat{\zeta}(S;w^0))(2r_1r_2). \label{eq:rhodef}
\end{align}
Then, for any distinct $i,j \in \{1,2\}$
\begin{align}\label{eq:ekeyrephrased}
|\rho a_{e_j}^{SGD} + \gamma_i a_{e_i}^{SGD}| \leq (1 + \max_{i \in \{1,2\}} r_i^2 \Ustat)\estat := \eps^{(2)},
\end{align}
and
\begin{align}
|\gamma_1 \gamma_2 - \rho^2| |a_{e_i}^{SGD}| \leq (\gamma_j + |\rho|)\eps^{(2)}.     \label{eq:twoactiveprimedhelper}
\end{align}
\end{claim}
\begin{proof}
First, write the derivatives of the idealized loss with respect to the parameters $b_v, a_{e_1}$, and $a_{e_2}$:
$$\pd{\elltr}{b_v} = \left|(r_1 a_{e_1})^2 + (r_2 a_{e_2})^2 + b_v + \hat{\zeta}(\emptyset;w^0)\right|.$$
Further, for any distinct $i,j \in \{1,2\}$,
\begin{align*}
\pd{\elltr}{a_{e_i}} = (2r_1r_2 a_{e_1} a_{e_2} + \hat{\zeta}(S;w^0)) (2r_1r_2 a_{e_j}) + \left(\pd{\ellt}{b_v}\right)(2r_i^2 a_{e_i}) + \gamma_i a_{e_i}.
\end{align*}
By the guarantee in \cref{item:trainneuronstationary} of \cref{lem:trainneuronstationarity} and the event $\Estat$, we have $\|\nabla_{w_v} \elltr(\wsgd)\|_{\infty} \leq \estat$. It follows that
\begin{align*}
|(2r_1r_2 a_{e_1}^{SGD} a_{e_2}^{SGD} + \hat{\zeta}(S;w^0))(2r_1r_2a_{e_j}^{SGD}) + \gamma_i a_{e_i}^{SGD}| \leq \estat (1 + r_i^2 a_{e_i}^{SGD}).
\end{align*}
Finally, by \cref{item:trainneuron4} of \cref{lem:trainneuronstationarity} we also have the bound $\|\wvsgd\|_{\infty} \leq \Ustat$, which combined with the above equation implies
$$|(2r_1r_2 a_{e_1}^{SGD} a_{e_2}^{SGD} + \hat{\zeta}(S;w^0))(2r_1r_2a_{e_j}^{SGD}) + \gamma_i a_{e_i}^{SGD}| \leq \estat (1 + r_i^2 \Ustat),$$ which is the claimed inequality \cref{eq:ekeyrephrased} when rewritten in terms of $\rho$.

Multiplying \cref{eq:ekeyrephrased} for $i=1,j=2$ by $\gamma_2$:
$$|\gamma_2 \rho a_{e_2}^{SGD} + \gamma_1 \gamma_2 a_{e_1}^{SGD}| \leq \gamma_2 \eps^{(2)},$$ and multiplying \cref{eq:ekeyrephrased} for $i=2,j=1$ by $|\rho|$:
$$|\rho^2 a_{e_1}^{SGD} + \gamma_2 \rho a_{e_2}^{SGD}| \leq |\rho| \eps^{(2)}.$$ Combining the above two inequalities by the triangle inequality,
\begin{align*}|\gamma_1 \gamma_2 - \rho^2| |a_{e_1}^{SGD}| \leq (\gamma_2 + |\rho|) \eps^{(2)}.\end{align*}
\cref{eq:twoactiveprimedhelper} follows by a symmetric argument.
\end{proof}
Now we may prove the main result of this subsection:
\begin{proof}[Proof of \cref{lem:2activenotprimed}]
We claim that \begin{align}|a_{e_1}^{SGD}|, |a_{e_2}^{SGD}| \leq \max(\eps^{(1)},\eps^{(3)},\eps^{(4)}).\label{eq:2anpaesgdclaim}\end{align} This is proved below, but first let us see the consequences. Plugging \cref{eq:2anpaesgdclaim} into the stationarity condition $\left|\pd{\ellt}{b_v} \mid_{w = \wsgd}\right| \leq \estat$ guaranteed by \cref{item:trainneuronstationary} of \cref{lem:trainneuronstationarity}, we obtain 
\begin{align*}
|b_v^{SGD}| &\leq |\hat{\zeta}(\emptyset;w^0)| + 2(r_1 a_{e_1}^{SGD})^2 + 2(r_2a_{e_2}^{SGD})^2 + \estat \\
&\leq |\hat{\zeta}(\emptyset; w^0)| + 2(|r_1|^2 + |r_2|^2)(\max(\eps^{(1)},\eps^{(3)},\eps^{(4)}))^2 + \estat.
\end{align*}
Therefore $|a_{e_1}^{SGD}|,|a_{e_2}^{SGD}| < \tau$ by \cref{eq:2anpcond1} and $|b_v^{SGD}| < \tau$ by \cref{eq:2anpcond2}. So \cref{step:truncate} of $\TrainNeuron$ rounds $a_{e_1}^{SGD}$, $a_{e_2}^{SGD}$ and $b_v^{SGD}$ to $\aeoneret = \aetworet = \bvret = 0$. Furthermore, \cref{claim:notdepend} and \cref{eq:2anpcond1} imply $\aeprimeret = 0$ for all $e' = (u',v) \in E$ such that $u' \not\in \{u_1,u_2\}$. Overall, this implies $\wret = w^0$, since $\wvret = \vec{0} = w_{v}^0$.

Therefore, it only remains to show \eqref{eq:2anpaesgdclaim}. We prove it with a case analysis.

\textbf{Case 1}: If $S_1 = S_2$, we have $S = S_1 \cup S_2 \sm (S_1 \cap S_2) = \emptyset$. In this case, \cref{item:trainneuronstationary} of \cref{lem:trainneuronstationarity} guarantees the stationarity conditions $\left|\pd{\elltr}{b_v}\mid_{w = \wsgd}\right| \leq \estat$ and $\left|\pd{\elltr}{a_{e_i}} \mid_{w = \wsgd} \right| \leq \estat$ for any $i \in \{1,2\}$, i.e.,
\begin{align*}
|(r_1 a_{e_1}^{SGD} + r_2a_{e_2}^{SGD})^2 + b_v^{SGD} + \hat{\zeta}(S; w^0)| \leq \estat.
\end{align*}
\begin{align*}
\left|2r_i(r_1 a_{e_1}^{SGD} + r_2 a_{e_2}^{SGD})((r_1 a_{e_1}^{SGD} + r_2a_{e_2}^{SGD})^2 + b_v^{SGD} + \hat{\zeta}(S; w^0)) + \gamma_i a_{e_i}^{SGD}\right| \leq \estat.
\end{align*}
Combining these two inequalities and the triangle inequality, we obtain
\begin{align*}
\left|\gamma_i a_{e_i}^{SGD}\right| &\leq (1 + |2r_i(r_1 a_{e_1}^{SGD} + r_2 a_{e_2}^{SGD})|)\estat,
\end{align*}
So
\begin{align*}
    \max_i |a_{e_i}^{SGD}|
    &\leq (1 + 4(r_1^2 + r_2^2) \max_i |a_{e_i}^{SGD}|) \estat / \min(\lambda_1, \lambda_2) \\
    &\leq \estat / \min(\lambda_1, \lambda_2) + \frac{1}{2} \max_i |a_{e_i}^{SGD}|. &\mbox{by \cref{eq:2anpcond5}}
\end{align*}
This means that
\begin{align*}
\max_i |a_{e_i}^{SGD}| \leq \estat / \min(\lambda_1,\lambda_2) = \eps^{(1)},
\end{align*}
concluding the analysis of this case.

\textbf{Case 2}: Otherwise, we are in the case that $S_1 \neq S_2$. Let $\rho = (2r_1r_2 a_{e_1}^{SGD} a_{e_2}^{SGD} + \hat{\zeta}(S;w^0))(2r_1r_2)$ be defined as in \cref{eq:rhodef}.

\textbf{Case 2a}: If $|\gamma_1 \gamma_2 - \rho^2| \geq \gamma_1 \gamma_2 / 2$,  then by \cref{eq:twoactiveprimedhelper}, which is guaranteed by \cref{claim:firstordercond2active},
\begin{align*}
&\max_{i \in [2]}|a_{e_i}^{SGD}| \\
&\leq (\max_{j \in [2]}\gamma_j + |\rho|)\eps^{(2)} / |\gamma_1 \gamma_2 - \rho^2| \\
&\leq (\max_{j \in [2]}\gamma_j + |\rho|)\eps^{(2)} / (\gamma_1 \gamma_2 / 2) \\
&\leq (\max(\lambda_1,\lambda_2) + |\rho|)\eps^{(2)} / (\gamma_1 \gamma_2 / 2) \\
&\leq (\max(\lambda_1,\lambda_2) + |(2r_1r_2 a_{e_1}^{SGD} a_{e_2}^{SGD} + \hat{\zeta}(S;w^0))(2r_1r_2)|)\eps^{(2)} / (\gamma_1 \gamma_2 / 2) \\
&\leq (1 + |2r_1r_2|^2|a_{e_1}^{SGD} a_{e_2}^{SGD}| + |2r_1r_2||\hat{\zeta}(S;w^0)|) \eps^{(2)} / (\gamma_1 \gamma_2 / 2) \\
&\leq 8\eps^{(2)}(1 + |a_{e_1}^{SGD} a_{e_2}^{SGD}| + |\hat{\zeta}(S;w^0)|)\max(1,|r_1r_2|^2) / (\gamma_1 \gamma_2) \\
&\leq 8\eps'(1 + (\Ustat)^2 + |\hat{\zeta}(S;w^0)|)\max(1,|r_1r_2|^2) / (\gamma_1 \gamma_2) &\mbox{by \cref{item:trainneuron4} of \cref{lem:trainneuronstationarity}} \\
&\leq \eps^{(3)}.
\end{align*}

\textbf{Case 2b}: Otherwise, if $|\gamma_1 \gamma_2 - \rho^2| \leq \gamma_1\gamma_2 / 2$, then 
\begin{align}|\rho| \in [\sqrt{\gamma_1 \gamma_2 / 2}, \sqrt{3\gamma_1 \gamma_2 / 2}]. \label{eq:rhointervalnp}
\end{align}
In this case, 
\begin{align*}
|2 r_1r_2 a_{e_1}^{SGD} a_{e_2}^{SGD}| &\leq |\rho/(2r_1r_2)| + |\hat{\zeta}(S;w^0)| \\
&\leq \sqrt{3\gamma_1 \gamma_2 / 2} / |2r_1r_2| + |\hat{\zeta}(S;w^0)| &\mbox{by \cref{eq:rhointervalnp}} \\
&\leq \sqrt{\gamma_1 \gamma_2} / |r_1r_2| + |\hat{\zeta}(S;w^0)|.
\end{align*}
Therefore, $\min_i |a_{e_i}^{SGD}| \leq \sqrt{\sqrt{\gamma_1 \gamma_2} / (2|r_1r_2|^2) + |\hat{\zeta}(S;w^0)| / |2r_1r_2|}$
Also, by \cref{eq:ekeyrephrased} of \cref{claim:firstordercond2active}, for any distinct $i,j \in \{1,2\}$ we have
\begin{align*}
|a_{e_j}^{SGD} + \gamma_i a_{e_i}^{SGD} / \rho| &\leq \eps^{(2)} / |\rho|,
\end{align*}
Therefore, by the triangle inequality:
\begin{align*}\max_i |a_{e_i}^{SGD}| &\leq  (\max(\gamma_1,\gamma_2)\min_i |a_{e_i}^{SGD}|  + \eps^{(2)}) / |\rho| \\
&\leq (\max(\gamma_1,\gamma_2)\min_i |a_{e_i}^{SGD}|  + \eps^{(2)}) / \sqrt{\gamma_1 \gamma_2 / 2} &\mbox{by \cref{eq:rhointervalnp}}.
\end{align*}
Therefore
\begin{align*}
\max_i |a_{e_i}^{SGD}| &\leq (\max(\gamma_1,\gamma_2)\sqrt{\sqrt{\gamma_1 \gamma_2} / (2|r_1r_2|^2) + |\hat{\zeta}(S;w^0)| / |2r_1r_2|}  + \eps^{(2)}) / \sqrt{\gamma_1 \gamma_2 / 2} \\
&\leq 2(\sqrt{\frac{\max(\gamma_1,\gamma_2)}{ \min(\gamma_1,\gamma_2)}}\sqrt{\sqrt{\gamma_1 \gamma_2} + |\hat{\zeta}(S;w^0)|}  + \eps^{(2)}\sqrt{\gamma_1 \gamma_2}) / \min(1,|r_1r_2|^2) \\
&\leq \eps^{(4)}.
\end{align*}
\end{proof}

\subsection{$\TrainNeuron$ correctness: training a neuron with two active inputs whose product is useful (\cref{lem:2activeprimedblank,lem:2activeprimedactive,lem:2activeprimedproblowbound})}

We now prove \cref{lem:2activeprimedblank,lem:2activeprimedactive,lem:2activeprimedproblowbound}, which are our final main results on $\TrainNeuron$'s correctness. These results state that if a neuron with two active inputs is trained, and if learning the product of the inputs would significantly contribute to reducing the loss, then with polynomially lower bounded probability the neuron learns the product up to some small relative error, and remains blank otherwise.

For the following definition recall that $\rho = (2r_1r_2 a_{e_1}^{SGD} a_{e_2}^{SGD} + \hat{\zeta}(S;w^0))(2r_1r_2)$ as defined in \cref{eq:rhodef}.

\begin{definition}\label{def:ecasetwo}
Let $\Ecasetwo$ be the event that $|\gamma_1 \gamma_2 - \rho^2| < \gamma_1 \gamma_2 / 2$. 
\end{definition}

In our analysis, when $|\hat{\zeta}(S;w^0)|$ is sufficiently large (i.e., when the learning a neuron that represents $\chi_S$ would significantly reduce the loss, then the event $\Ecasetwo$ corresponds to when $\TrainNeuron$ creates an active neuron.

\begin{lemma}[Two active inputs, product is useful, case when neuron remains blank] \label{lem:2activeprimedblank}
Suppose that \cref{ass:atmosttwoactive} holds, and the event $\Estat \cap (\neg \Ecasetwo)$ holds, and $S_1 \neq S_2$. 
Finally, recall the definitions of $\eps^{(1)},\eps^{(2)},\eps^{(3)}$
$$\eps^{(1)} = 2\estat / \min(\lambda_1, \lambda_2)$$
$$\eps^{(2)} = \estat (1 + \max_{i \in \{1,2\}} r_i^2 \Ustat)$$
$$\eps^{(3)} =  8\eps^{(2)} (1 + (\Ustat)^2 + |\hat{\zeta}(S;w^0)|)\max(1,|r_1r_2|^2) / \min(\lambda_1,\lambda_2)^2$$
 and suppose that the following hold:
\begin{align}\tau > \max(\eps^{(1)},\eps^{(3)})\label{eq:2apcond1}\end{align}
\begin{align}\label{eq:2apcond4}
\tau > |\hat{\zeta}(\emptyset; w^0)| + 2(|r_1|^2 + |r_2|^2) (\eps^{(3)})^2 + \estat
\end{align}
Then $\wret = w^0$ (and $v$ remains a blank neuron).
\end{lemma}
\begin{proof}
Since $S_1 \neq S_2$, and the event $\neg \Ecasetwo$ implies $|\gamma_1 \gamma_2 - \rho^2| \geq \gamma_1 \gamma_2 / 2$, the proof of this lemma is identical to the proof for Case 2a in \cref{lem:2activenotprimed}.
\end{proof}

\begin{lemma}[Two active inputs, product is useful, case when new active neuron is created] \label{lem:2activeprimedactive}

Suppose that \cref{ass:atmosttwoactive} holds, and the event $\Estat \cap \Ecasetwo$ holds. Suppose also that $f_{u_1}(x;w^0)$ depends only on variables in $S_1$, and $f_{u_2}(x;w^0)$ depends only on variables in $S_2$, and that $S_2 \neq \emptyset$ and $S_1 \cap S_2 = \emptyset$. Suppose also that $\eps_1 \leq 1$ and $\eps_2 = 0$.
Finally, recall the definition
$$\eps^{(1)} = 2\estat / \min(\lambda_1,\lambda_2),$$ and
suppose also that
\begin{align}
\tau > \eps^{(1)} \label{eq:2apactivecond1}
\end{align}
\begin{align}|\hat{\zeta}(S;w^0)| \geq \sqrt{3}\max(\lambda_1,\lambda_2) / |r_1r_2|. \label{eq:2apcond2}\end{align}
\begin{align}\eps^{(2)} \leq \min(\lambda_1, \lambda_2) \sqrt{ |\hat{\zeta}(S;w^0)| / |r_1r_2|}/8\label{eq:twoprimegiven1}\end{align}
\begin{align}\label{eq:tauaeiactiveuppbound}
    \tau < \frac{1}{8} \sqrt{\frac{\lambda_1}{\lambda_2} |\hat{\zeta}(S;w^0)| / |r_1r_2|},
\end{align}
\begin{align}
\tau < (|\hat{\zeta}(S;w^0)| / 4) (\min_i |r_i|) / (\max_i |r_i|)  - |\hat{\zeta}(\emptyset;w^0)| - \estat \label{eq:taubvactiveuppbound}
\end{align}
for some large enough universal constant $C > 0$.

Then
\begin{enumerate} 
    \item We may write $f_v(x;\wret) = r \chi_{S}(x) + h(x)$, such that $\hat{h}(S) = 0$, the error is bounded by $|h(x)| \leq |r|\enewrel$, where \begin{align}
        \enewrel &= (4\estop + 2|\hat{\zeta}(\emptyset;w^0)|)/|\hat{\zeta}(S;w^0)| + 32 \frac{\lambda_2}{\lambda_1} |\eps_1|^2 |r_1/r_2| + \eps_1 (8\frac{|r_1|}{|r_2|} \sqrt{\frac{\gamma_2}{\gamma_1}} + 1) \nonumber 
        \end{align}
         and the scaling factor $r$ is close to $-\hat{\zeta}(S;w^0)$:
        $$|r + \hat{\zeta}(S;w^0)| \leq \frac{4\sqrt{\gamma_1 \gamma_2}}{|r_1r_2|}.$$\label{item:2apa1}
    \item The weights after training are bounded: $\aeprimeret = 0$ for all $e' = (u',v) \in E$ such that $u' \not\in \{u_1,u_2\}$, and \begin{align*}|\aeoneret|, |\aetworet| \leq 4\sqrt{\frac{\lambda_2}{\lambda_1} |\hat{\zeta}(S;w^0)| / |r_1r_2|}.\end{align*} \label{item:2apaebound}
    \item The error bias is bounded: $$|\hat{\zeta}(\emptyset; \wret)| \leq \estat.$$ \label{item:2apabias}
\end{enumerate}
\end{lemma}

\begin{lemma}[Two active inputs, product is useful: two cases and probability lower bound]\label{lem:2activeprimedproblowbound}
Suppose that \cref{ass:atmosttwoactive} holds, as well as the conditions of \cref{lem:2activeprimedblank}. Suppose also that $r_1,r_2 \neq 0$, $S_1 \neq S_2$, and that the following inequalities hold, where $C > 0$ is a large enough universal constant,
\begin{align}
\eta \geq 4\tau \label{eq:etalargerthantau}
\end{align}
\begin{align}|2r_1r_2 \eta^2| < |\hat{\zeta}(S;w^0)| / 16 \label{eq:twor1r2etasquaredvszeta}\end{align}
\begin{align}|2r_1r_2 \tau^2| < |\hat{\zeta}(S;w^0)| / 16\label{eq:twor1r2tausquaredvszeta}\end{align}
\begin{align}
|r_1r_2\eta^2\hat{\zeta}(S;w^0)| \geq 16(r_1^2 \eta^2 + r_2^2 \eta^2 + |\hat{\zeta}(\emptyset;w^0)|)^2
\label{eq:lossdecnastytermbound}
\end{align}
\begin{align}
   \frac{1}{32} |r_1r_2\eta^2\hat{\zeta}(S;w^0)| - \lambda_2 W \eta^2 - C \max_{u \in V} \max_{x \in \{-1,1\}^n} (|f_u(x;w^0)|^3 + 1) (\max_{i \in \{1,2\}} \eps_i) > 0 \label{eq:etamorethanzerocond}
\end{align}
Then 
$$\PP[\Ecasetwo \cap \Estat \mid w^0] \geq \min(1,\sqrt{ |r_1r_2 \hat{\zeta}(S;w^0)|}/8) -  \PP[\neg \Estat \mid w^0].$$
\end{lemma}

\subsubsection{Proof of \cref{lem:2activeprimedactive}}
\begin{proof}[Proof of \cref{lem:2activeprimedactive}]
The proof is modularized into several claims:

\begin{claim}[Input weights from blank neurons are sent to zero]
$\aeprimeret = 0$ for all $e' = (u',v) \in E$ such that $u' \not\in \{u_1,u_2\}$.
\end{claim}
\begin{proof}
By \cref{claim:notdepend}, since the precondition \cref{eq:2apactivecond1} holds, as well as \cref{ass:atmosttwoactive} and the event $\Estat$.
\end{proof}

 So it only remains to examine $\aeoneret$, $\aetworet,$ and $\bvret$.

\begin{claim}
The following bounds on $|a_{e_i}^{SGD}|$ hold:
\begin{align}\max_{i \in \{1,2\}} |a_{e_i}^{SGD}| &\geq \sqrt{|\hat{\zeta}(S;w^0)| / |4r_1r_2|}. \label{eq:aeisgdlowerbound}
\end{align}
\begin{align}
\min_{i \in \{1,2\}} |a_{e_i}^{SGD}| &\leq \sqrt{|\hat{\zeta}(S;w^0)|/|r_1r_2|} \label{eq:aeisgdupperbound}
\end{align}
\end{claim}
\begin{proof}
Since $\Ecasetwo$ holds we must have
$\rho^2 \in [\gamma_1 \gamma_2 / 2, 3\gamma_1 \gamma_2 / 2]$,  so \begin{align}|\rho| \in [\sqrt{\gamma_1 \gamma_2 / 2}, \sqrt{3\gamma_1 \gamma_2 / 2}].\label{eq:rhointerval}\end{align}
Plugging \cref{eq:rhointerval} into the definition of $\rho$ implies that 
$$||2r_1r_2 a_{e_1}^{SGD} a_{e_2}^{SGD}| - |\hat{\zeta}(S;w^0)|| \leq \sqrt{3\gamma_1 \gamma_2 / 2} / |2r_1r_2|.$$
Since $|\hat{\zeta}(S;w^0)| \geq \sqrt{3}\max(\lambda_1,\lambda_2) / |r_1r_2| > 2 \sqrt{3\gamma_1\gamma_2/2}  / |2r_1r_2|$ by \cref{eq:2apcond2}, this means
\begin{align}
|2r_1r_2a_{e_1}^{SGD}a_{e_2}^{SGD}| \in [|\hat{\zeta}(S;w^0)| / 2, 2|\hat{\zeta}(S;w^0)|]. \label{eq:rlowerbound}
\end{align}
\cref{eq:aeisgdlowerbound,eq:aeisgdupperbound} immediately follow.
\end{proof}

\begin{claim}
For any distinct $i,j \in \{1,2\}$ we have
\begin{align}
\frac{1}{4}|a_{e_{i}}^{SGD}| \leq \sqrt{\frac{\gamma_{j}}{\gamma_{i}}}|a_{e_{j}}^{SGD}| \leq 4|a_{e_{i}}^{SGD}|. \label{eq:aeiratiointerval}
\end{align}
\end{claim}
\begin{proof}
Moreover, plugging \cref{eq:rhointerval} into \cref{eq:ekeyrephrased}, for all distinct $i,j \in \{1,2\}$, we also have
\begin{align}|a_{e_j}^{SGD} + \gamma_i a_{e_i}^{SGD} / \rho| &\leq \eps^{(2)} / |\rho| \nonumber \\
&\leq \eps^{(2)} / \sqrt{\gamma_1 \gamma_2 / 2} \nonumber \\
&\leq \sqrt{|\hat{\zeta}(S;w^0)| / |4r_1r_2|}/2 \nonumber &\mbox{by \cref{eq:twoprimegiven1}}\\
&\leq \max_{i^* \in \{1,2\}} |a_{e_{i^*}}^{SGD}| / 2 &\mbox{by \cref{eq:aeisgdlowerbound}}\label{eq:sandwichbound}\end{align}
Let $i^*,j^* \in \{1,2\}$ be distinct indices such that $|a_{e_{i^*}}^{SGD}| = \max_{i \in \{1,2\}} |a_{e_i}^{SGD}|$ and $|a_{e_{j^*}}^{SGD}| = \min_{j \in \{1,2\}} |a_{e_j}^{SGD}|$. Therefore,
\begin{align*}
|a_{e_{i^*}}^{SGD} + \gamma_{j^*} a_{e_{j^*}}^{SGD} / {\rho}| &\leq |a_{e_{i^*}}^{SGD}|/2 &\mbox{by \cref{eq:sandwichbound}}
\end{align*}
As a consequence,
\begin{align*}
\frac{1}{2}|a_{e_{i^*}}^{SGD}| \leq |\gamma_{j^*} a_{e_{j^*}}^{SGD} / {\rho}| \leq \frac{3}{2}|a_{e_{i^*}}^{SGD}|.
\end{align*}
And because of the bounds in \cref{eq:rhointerval}, we have
\begin{align*}
\frac{1}{4}|a_{e_{i^*}}^{SGD}| \leq \sqrt{\frac{\gamma_{j^*}}{\gamma_{i^*}}}|a_{e_{j^*}}^{SGD}| \leq 4|a_{e_{i^*}}^{SGD}|.
\end{align*}
This immediately implies \cref{eq:aeiratiointerval}.
\end{proof}

We now use the above claims to bound the range of $[\min_i |a_{e_i}^{SGD}|, \max_i |a_{e_i}^{SGD}|]$:
\begin{claim}
\begin{align}
\min_i |a_{e_i}^{SGD}| &\geq \frac{1}{8} \sqrt{\frac{\lambda_1}{\lambda_2} |\hat{\zeta}(S;w^0)| / |r_1r_2|} \label{eq:minaeilowbound}
\end{align}
\begin{align}
\max_i |a_{e_i}^{SGD}| &\leq 4 \sqrt{\frac{\lambda_1}{\lambda_2} |\hat{\zeta}(S;w^0)| / |r_1r_2|}\label{eq:aeisgdmaxuppbound}
\end{align}
\end{claim}
\begin{proof}
We first show \cref{eq:minaeilowbound}:
\begin{align*}
\min_i |a_{e_i}^{SGD}| &\geq \frac{1}{4} \sqrt{\frac{\min(\gamma_1, \gamma_2)}{\max(\gamma_1,\gamma_2)}} \max_i |a_{e_i}^{SGD}| &\mbox{by \cref{eq:aeiratiointerval}} \\ 
&\geq \frac{1}{4} \sqrt{\frac{\min(\gamma_1, \gamma_2)}{\max(\gamma_1,\gamma_2)} |\hat{\zeta}(S;w^0)| / |4r_1r_2|} & \mbox{by \cref{eq:aeisgdlowerbound}} \\
&\geq \frac{1}{8} \sqrt{\frac{\lambda_1}{\lambda_2} |\hat{\zeta}(S;w^0)| / |r_1r_2|} &\mbox{since $\lambda_1 \leq \lambda_2$ and $\gamma_1,\gamma_2 \in \{\lambda_1,\lambda_2\}$}
\end{align*}

And similarly we show \cref{eq:aeisgdmaxuppbound}:
\begin{align*}
\max_i |a_{e_i}^{SGD}| &\leq 4 \sqrt{\frac{\max(\gamma_1,\gamma_2)}{\min(\gamma_1,\gamma_2)}}\min_i |a_{e_i}^{SGD}| &\mbox{by \cref{eq:aeisgdmaxuppbound}} \\
&\leq 4\sqrt{\frac{\max(\gamma_1,\gamma_2)}{\min(\gamma_1,\gamma_2)}}\sqrt{|\hat{\zeta}(S;w^0)| / |r_1r_2|} \\
&\leq 4\sqrt{\frac{\lambda_2}{\lambda_1} |\hat{\zeta}(S;w^0)| / |r_1r_2|}. 
\end{align*}
\end{proof}

\begin{claim}\label{claim:ecasetwosgdequalsret}
$\aeoneret = a_{e_1}^{SGD}$, $\aetworet = a_{e_2}^{SGD}$ and $\bvret = b_v^{SGD}$.
\end{claim}
\begin{proof}
First, from the previous claim,
\begin{align*}
\min_i |a_{e_i}^{SGD}| &\geq \frac{1}{8} \sqrt{\frac{\lambda_1}{\lambda_2} |\hat{\zeta}(S;w^0)| / |r_1r_2|} &\mbox{by \cref{eq:minaeilowbound}} \\
&> \tau & \mbox{by \cref{eq:tauaeiactiveuppbound}}.
\end{align*}
Furthermore, by the stationarity condition $\left|\pd{\elltr}{b_v} \mid_{w = \wsgd}\right| \leq \estat$, which is guaranteed by the event $\Estat$ and \cref{item:trainneuronstationary} from \cref{lem:trainneuronstationarity}:
$$|(r_1 a_{e_1}^{SGD})^2 + (r_2a_{e_2}^{SGD})^2 + b_v^{SGD} + \hat{\zeta}(\emptyset;w^0)| \leq \estat.$$
which means that
\begin{align*}|b_v^{SGD}| &\geq (r_1 a_{e_1}^{SGD})^2 + (r_2 a_{e_2}^{SGD})^2 - |\hat{\zeta}(\emptyset;w^0)| - \estat \\
&\geq \max_i (r_i a_{e_i}^{SGD})^2 - |\hat{\zeta}(\emptyset;w^0)| - \estat \\ 
&\geq \min_i (r_i)^2 |\hat{\zeta}(S;w^0)| / |4r_1r_2| - |\hat{\zeta}(\emptyset;w^0)| - \estat &\mbox{by \cref{eq:aeisgdlowerbound}} \\
&\geq (|\hat{\zeta}(S;w^0)| / 4) (\min_i |r_i|) / (\max_i |r_i|)  - |\hat{\zeta}(\emptyset;w^0)| - \estat \\
&> \tau &\mbox{by \cref{eq:taubvactiveuppbound}}
\end{align*}

Therefore, since $|a_{e_1}^{SGD}|,|a_{e_2}^{SGD}|,|b_v^{SGD}| > \tau$, the rounding in \cref{step:truncate} of $\TrainNeuron$ keeps the weights from $\NeuronSGD$ unchanged.
\end{proof}

We may now begin to prove the items of \cref{lem:2activeprimedactive}.
\begin{claim}
\cref{item:2apaebound} holds.
\end{claim}
\begin{proof}
This is true because by \cref{eq:aeisgdmaxuppbound}, we have $\max_i |a_{e_i}^{SGD}| \leq 4 \sqrt{\frac{\lambda_2}{\lambda_1} |\hat{\zeta}(S;w^0)| / |r_1 r_2|}$, and by the previous claim we have $a_{e_1}^{SGD} = \aeoneret$ and $a_{e_2}^{SGD} = \aetworet$.
\end{proof}

We now proceed to analyze the relative error of the active neuron that is created.
\begin{claim}\label{claim:newrel}
Neuron $v$ becomes active, with low relative error: i.e., $f_v(x;\wret) = r \chi_S(x) + h(x)$, where $r = 2r_1r_2a_{e_1}^{SGD}a_{e_2}^{SGD}$ and $h(x) \leq |r|\enewrel$ for any $x$, and $\hat{h}(S) = 0$.
This is the first half of \cref{item:2apa1} of the lemma.
\end{claim}
\begin{proof}
 \begin{align*}
f_v(x;\wret) &= \bvret + (\aeoneret (r_1 \chi_{S_1}(x) + h_1(x)) + \aetworet r_2 \chi_{S_2}(x))^2 & \mbox{since $\eps_2 = 0$} \\
&= b_v^{SGD} + (a_{e_1}^{SGD} (r_1 \chi_{S_1}(x) + h_1(x)) + a_{e_2}^{SGD} r_2 \chi_{S_2}(x))^2 &\mbox{by \cref{claim:ecasetwosgdequalsret}} \\
&= 2r_1r_2a_{e_1}^{SGD}a_{e_2}^{SGD} \chi_{S_1}(x) \chi_{S_2}(x) + T_1(x) + T_2(x)
&= r \chi_{S}(x) + h(x),
\end{align*}
where we have defined $r = 2r_1r_2a_{e_1}^{SGD}a_{e_2}^{SGD}$, $h(x) = T_1(x) + T_2(x)$, and \\
\begin{align*}
T_1(x) &= b_v^{SGD} + (r_1 a_{e_1}^{SGD})^2 + (r_2 a_{e_2}^{SGD})^2\\
T_2(x) &= (a_{e_1}^{SGD}h_1(x)) (2r_1 a_{e_1}^{SGD}\chi_{S_1}(x) + a_{e_1}^{SGD}h_1(x) + 2r_2 a_{e_2}^{SGD}\chi_{S_2}(x)).
\end{align*}\

Note that $\hat{h}(S) = 0$ since $\hat{T}_1(S) = 0$ and $\hat{T}_2(S) = 0$ and by linearity of the Fourier transform $\hat{h}(S) = \hat{T}_1(S) + \hat{T}_2(S)$. In particular, $\hat{T}_1(S) = 0$ because $T_1(x)$ is a constant and $S \neq \emptyset$ because since $S_1$ and $S_2$ are disjoint we have $S = (S_1 \cup S_2) \sm (S_1 \cap S_2) = S_1 \cup S_2 \supset S_2 \neq \emptyset$. Further, $\hat{T}_2(S) = 0$ since, first of all, $h_1(x)\chi_{S_1}(x)$ and $h_1(x)^2$ depend only on variables in $S_1$ so they cannot be correlated with $\chi_S$ because, which depends on all the variables in $S_2$, which is nonempty. And, secondly, $\hat{h}_1(S_1) = 0$ by \cref{ass:atmosttwoactive}, so $h_1(x)\chi_{S_2}(x)$ cannot be correlated to $\chi_S(x)$ because $S_2$ is nonempty.

In order to bound $|T_1(x)|$, let us first compute the derivative of the regularized loss with respect to $b_v$:
\begin{align*}
\pd{\ellr}{b_v} &= \pd{\ell}{b_v} \\
&= \pd{}{b_v} \EE_{x \sim \{-1,1\}^n}[\frac{1}{2} (b_v + (\sum_{i \in [2]} a_{e_i} f_{u_i}(x;w^0))^2 + \zeta(x;w^0))^2] \\
&= \EE_{x \sim \{-1,1\}^n}[b_v + (\sum_{i \in [2]} a_{e_i} f_{u_i}(x;w^0))^2 + \zeta(x;w^0)] \\
&= b_v + \hat{\zeta}(\emptyset;w^0) + \EE_{x \sim \{-1,1\}^n}[(\sum_{i \in [2]} a_{e_i} f_{u_i}(x;w^0))^2] \\
&= b_v + \hat{\zeta}(\emptyset;w^0) + (r_1 a_{e_1})^2 + (r_2 a_{e_2})^2 \\
&\quad + \EE_{x \sim \{-1,1\}^n}[2r_2 a_{e_1} a_{e_2} (r_1\chi_{S_1}(x) + h_1(x)) \chi_{S_2}(x) + 2r_1 (a_{e_1})^2 \chi_{S_1}(x) h_1(x) + (a_{e_1} h_1(x))^2] \\
&= b_v + \hat{\zeta}(\emptyset;w^0) + (r_1 a_{e_1})^2 + (r_2 a_{e_2})^2 + \EE_{x \sim \{-1,1\}^n}[(a_{e_1} h_1(x))^2],
\end{align*}
where in the last line we use that $\EE_{x \sim \{-1,1\}^n}[\chi_{S_1}(x) h_1(x)] = 0$ by \cref{ass:atmosttwoactive}. We also use that $\EE_{x \sim \{-1,1\}^n}[(\chi_{S_1}(x) + h_1(x)) \chi_{S_2}(x)] = 0$ since $\chi_{S_1}(x) + h_1(x)$ only depends on $\{x_i\}_{i \in S_i}$ and $S_1 \cap S_2 = \emptyset$, and $S_2 \neq \emptyset$ by assumption.
Therefore,
\begin{align*}
|T_1(x)| &= |b_v + (r_1 a_{e_1}^{SGD})^2 + (r_2 a_{e_2}^{SGD})^2| \\
&\leq \left|\pd{\ellr}{b_v} \mid_{w = \wsgd}\right| + |\hat{\zeta}(\emptyset;w^0)| + |\EE_{x \sim \{-1,1\}^n}[(a_{e_1}^{SGD} h_1(x))^2]| \\
&\leq 2\estop + |\hat{\zeta}(\emptyset;w^0)| + |\EE_{x \sim \{-1,1\}^n}[(a_{e_1}^{SGD} h_1(x))^2]| &\mbox{by \cref{lem:trainneuronstationarity}} \\
&\leq 2\estop + |\hat{\zeta}(\emptyset;w^0)|  + (a_{e_1}^{SGD} r_1 \eps_1)^2 &\mbox{since $|h_1(x)| \leq |r_1 \eps_1|$}\\
&\leq 2\estop + |\hat{\zeta}(\emptyset;w^0)| + 16\frac{\lambda_2}{\lambda_1}|\hat{\zeta}(S;w^0)||r_1 \eps_1|^2 / |r_1r_2| &\mbox{by \cref{eq:aeisgdmaxuppbound}} \\
&\leq 2\estop + |\hat{\zeta}(\emptyset;w^0)| + 16 \frac{\lambda_2}{\lambda_1} |\hat{\zeta}(S;w^0)| |\eps_1|^2 |r_1/r_2|
\end{align*}

We now bound $|T_2(x)|$. Since $|h_1(x)| \leq |r_1| \eps_1$, and $|\chi_{S_1}(x)|, |\chi_{S_2}(x)| \leq 1$,
\begin{align*}
|T_2(x)| &\leq 2\eps_1|r_1 a_{e_1}^{SGD}|(|r_1 a_{e_1}^{SGD}|(1 + \eps_1) + |r_2 a_{e_2}^{SGD}|).
\end{align*}

By the above bounds on $T_1(x)$ and $T_2(x)$, we have
\begin{align*}
|h(x)| &\leq |T_1(x)| + |T_2(x)| \\
&\leq |T_1(x)| + 2\eps_1|r_1 a_{e_1}^{SGD}|(|r_1 a_{e_1}^{SGD}|(1 + \eps_1) + |r_2 a_{e_2}^{SGD}|) \\
&\leq |r| (|T_1(x)|/|r| + \eps_1 (\frac{|r_1 a_{e_1}^{SGD}|}{|r_2 a_{e_2}^{SGD}|} (1+ \eps_1) + 1)) & \mbox{by definition of $r$} \\
&\leq |r| (|T_1(x)|/|r| + \eps_1 (2\frac{|r_1 a_{e_1}^{SGD}|}{|r_2 a_{e_2}^{SGD}|} + 1)) &\mbox{by $\eps_1 \leq 1$} \\
&\leq |r| (|T_1(x)| / |r| + \eps_1 (8\frac{r_1}{r_2} \sqrt{\frac{\gamma_2}{\gamma_1}} + 1) &\mbox{by \cref{eq:aeiratiointerval}} \\
&= |r| (|T_1(x)| / |2r_1r_2 a_{e_1}^{SGD} a_{e_2}^{SGD}| + \eps_1 (8\frac{r_1}{r_2} \sqrt{\frac{\gamma_2}{\gamma_1}} + 1)) &\mbox{by definition of $r$} \\
&\leq |r| (2|T_1(x)| / |\hat{\zeta}(S;w^0)| + \eps_1 (8\frac{r_1}{r_2} \sqrt{\frac{\gamma_2}{\gamma_1}} + 1)) &\mbox{by \cref{eq:rlowerbound}} \\
&\leq |r|( (4\estop + 2|\hat{\zeta}(\emptyset;w^0)|)/|\hat{\zeta}(S;w^0)|  \\ &\quad\quad\quad + 32 \frac{\lambda_2}{\lambda_1} |\eps_1|^2 |r_1 / r_2| + \eps_1 (8\frac{r_1}{r_2} \sqrt{\frac{\gamma_2}{\gamma_1}} + 1)) \\
&:= |r| \enewrel.
\end{align*}
\end{proof}

\begin{claim}
The error in the direction of $\chi_S$ is greatly reduced to close to zero:
\begin{align}|r + \hat{\zeta}(S;w^0)| \leq \frac{4\sqrt{\gamma_1 \gamma_2}}{|r_1r_2|}. \label{eq:rclosebound}\end{align}
This proves the second part of \cref{item:2apa1}.
\end{claim}
\begin{proof}
By \cref{eq:ekeyrephrased}, for any distinct $i,j \in \{1,2\}$, we have $$\rho \leq \frac{\gamma_i |a_{e_i}^{SGD}| + \eps^{(2)}}{|a_{e_j}^{SGD}|}.$$ By the bound in \cref{eq:aeiratiointerval} on the ratio of $|a_{e_i}^{SGD}|$ and $|a_{e_j}^{SGD}|$, this means:
\begin{align*}\rho &\leq 4 \sqrt{\gamma_j \gamma_i} + \frac{\eps^{(2)}}{|a_{e_j}^{SGD}|} \\
&\leq 4 \sqrt{\gamma_1 \gamma_2} + \frac{\eps^{(2)}}{|a_{e_j}^{SGD}|}.
\end{align*}
By the lower bound on $\max_{i} |a_{e_i}^{SGD}|$ in \cref{eq:aeisgdlowerbound}, this implies
\begin{align*}
\rho &\leq 4 \sqrt{\gamma_1 \gamma_2} + \frac{\eps^{(2)}}{\sqrt{|\hat{\zeta}(S; w^0)| / |4r_1r_2|}} \\
&\leq 8\sqrt{\gamma_1 \gamma_2} &\mbox{by \cref{eq:twoprimegiven1}}. 
\end{align*}
Since $\rho = (2r_1r_2 a_{e_1}^{SGD} a_{e_2}^{SGD} + \hat{\zeta}(S;w^0))(2r_1r_2) = (r + \hat{\zeta}(S;w^0))(2r_1r_2)$, this proves \cref{eq:rclosebound}:
\begin{align*}
|r + \hat{\zeta}(S; w^0)| &\leq \frac{4 \sqrt{\gamma_1 \gamma_2}}{|r_1r_2|}. 
\end{align*}
\end{proof}

\begin{claim}\label{claim:biassmall}
We now prove \cref{item:2apabias}, which controls the final bias of the error: $|\hat{\zeta}(\emptyset; \wret)| \leq 2\estop$.
\end{claim}
\begin{proof}
Using \cref{eq:decompexpanded},
\begin{align*}
&\pd{\ellr}{b_v} \mid_{w = \wsgd} \\
&= \pd{\ell}{b_v} \mid_{w = \wsgd} &\mbox{$R$ does not depend on $b_v$}  \\
&= \pd{}{b_v} \frac{1}{2} \EE_{x \sim \{-1,1\}^n}[(f(x;w^0) + b_v^{SGD} + (\sum_{i \in [2]} a_{e_i}^{SGD} f_{u_i}(x;w^0))^2 - g(x))^2] \\
&= \EE_{x \sim \{-1,1\}^n}[(f(x;w^0) + b_v^{SGD} + (\sum_{i \in [2]} a_{e_i}^{SGD} f_{u_i}(x;w^0))^2 - g(x))] \\
&= \EE_{x \sim \{-1,1\}^n}[(f(x;w^0) + \bvret + (\sum_{i \in [2]} \aeiret f_{u_i}(x;w^0))^2 - g(x))]   & \mbox{by \cref{claim:ecasetwosgdequalsret}} \\
&= \EE_{x \sim \{-1,1\}^n}[\zeta(x;\wret)] \\
&= |\hat{\zeta}(\emptyset; \wret)|.
\end{align*}
By event $\Estat$ and the stationarity guarantee in \cref{item:trainneuronstationary} of \cref{lem:trainneuronstationarity}, we have $\|\nabla_{w_v} \ellr(\wsgd)\|_{\infty} \leq 2\estop$.
\end{proof}

\end{proof}

\subsubsection{Proof of \cref{lem:2activeprimedproblowbound}}
\begin{proof}[Proof of \cref{lem:2activeprimedproblowbound}]
Let $\Egoodinit$ be the event that the following conditions \cref{eq:Fgoodinit1,eq:Fgoodinit2,eq:Fgoodinit3} hold. These conditions imply that the random perturbation at initialization in \cref{step:perturb} is ``good,'' and ensure that the optimization in $\NeuronSGD$ will not fall into a saddle point or spurious local minimum:
\begin{align}
|\aeoneperturb|, |\aetwoperturb| > \eta/2,\label{eq:Fgoodinit1}
\end{align}
\begin{align}
\sgn(r_1r_2\aeoneperturb \aetwoperturb) = \sgn(-\hat{\zeta}(S;w^0))\label{eq:Fgoodinit2}
\end{align}
\begin{align}
|\bvperturb| \leq \sqrt{|r_1r_2\hat{\zeta}(S;w^0)|}\eta / 8. \label{eq:Fgoodinit3}
\end{align}

Since $\aeoneperturb, \aetwoperturb,\bvperturb$ are chosen i.i.d. uniformly at random from $[-\eta, \eta]$, the events that \cref{eq:Fgoodinit1,eq:Fgoodinit2,eq:Fgoodinit3} hold are independent of each other and of $w^0$. So $$\PP[\Egoodinit \mid w^0] \geq (1/4)^2 \cdot (1/2) \cdot \min(1, \sqrt{|r_1r_2\hat{\zeta}(S;w^0)|} / 8).$$
We make the following claim:
\begin{claim}\label{claim:problowboundaux}
If $(\neg \Ecasetwo) \cap \Estat \cap \Egoodinit$ holds, then $\ellr(\wperturb) < \ellr(\wsgd)$.
\end{claim}
On the other hand, \cref{lem:trainneuronstationarity} guarantees that under the event $\Estat$ we have $\ellr(\wperturb) \geq \ellr(\wsgd)$, so to avoid a contradiction we must have $$\PP[(\neg \Ecasetwo) \cap \Estat \cap \Egoodinit \mid w^0] = 0.$$
So by a union bound,
\begin{align*}
\PP[\Ecasetwo &\cap \Estat \cap \Egoodinit \mid w^0] \\
&= \PP[\Estat \cap \Egoodinit \mid w^0] \geq 1 - \PP[\neg \Egoodinit] - \PP[\neg \Estat \mid w^0] \\
&\geq \min(1,\sqrt{ |2r_1r_2 \hat{\zeta}(S;w^0)|}) / C - (1 - \PP[\Estat \mid w^0]).
\end{align*}
\end{proof}
It only remains to prove the helper claim:

\begin{proof}[Proof of \cref{claim:problowboundaux}]
We begin by comparing $\ellt(\wsgd)$ and $\ellt(\wperturb)$. First, we lower-bound $\ellt(\wsgd)$ under event $(\neg \Ecasetwo) \cap \Estat$:
\begin{align*}
&\ellt(\wsgd) - \frac{1}{2}\sum_{\substack{S' \subset [n] \\ S' \neq \emptyset, S}} (\hat{\zeta}(S';w^0))^2 \\
&= \frac{1}{2}(2r_1r_2a_{e_1}^{SGD}a_{e_2}^{SGD} + \hat{\zeta}(S;w^0))^2 + \frac{1}{2}(r_1^2(a_{e_1}^{SGD})^2 + r_2^2(a_{e_2}^{SGD})^2 + b_v^{SGD} + \hat{\zeta}(\emptyset;w^0))^2 \\
&\geq \frac{1}{2}(2r_1r_2a_{e_1}^{SGD}a_{e_2}^{SGD} + \hat{\zeta}(S;w^0))^2 \\
&\geq \frac{1}{2} (-2|r_1r_2|\tau^2 + |\hat{\zeta}(S;w^0)|)^2
\end{align*}
where in the last line we use that that under event $(\neg \Ecasetwo) \cap \Estat$ we have $|a_{e_1}^{SGD}|, |a_{e_2}^{SGD}| < \tau$ by
\cref{lem:2activeprimedblank}, and also $2|r_1r_2|\tau^2 \leq |\hat{\zeta}(S;w^0)|$ by \cref{eq:twor1r2tausquaredvszeta}.
On the other hand, under event $\Egoodinit$:
\begin{align*}
(2r_1r_2\aeoneperturb\aetwoperturb + \hat{\zeta}(S;w^0))^2 \leq (2|r_1r_2|(\eta / 2)^2 - |\hat{\zeta}(S;w^0)|)^2,
\end{align*}
since $|\aeoneperturb|,|\aetwoperturb| \leq \eta / 2$ by \cref{eq:Fgoodinit1}, $\sgn(2r_1r_2\aeoneperturb\aetwoperturb) = -\sgn(\hat{\zeta}(S;w^0))$ by \cref{eq:Fgoodinit2}, and $2|r_1r_2|(\eta/2)^2 \leq |\hat{\zeta}(S;w^0)|$ by \cref{eq:twor1r2etasquaredvszeta}.

Furthermore, by the fact that $|\aeoneperturb|, |\aetwoperturb| < \eta$,
\begin{align*}
(r_1^2(\aeoneperturb)^2 + r_2^2(\aetwoperturb)^2 + \bvperturb + \hat{\zeta}(\emptyset;w^0))^2 &\leq (r_1^2 \eta^2 + r_2^2 \eta^2 + |\bvperturb| + |\hat{\zeta}(\emptyset;w^0)|)^2.
\end{align*}

So combining the above bounds we obtain:
\begin{align*}
&\ellt(\wperturb) - \frac{1}{2}\sum_{\substack{S' \subset [n] \\ S' \neq \emptyset, S}} (\hat{\zeta}(S';w^0))^2 \\
&= \frac{1}{2}(2r_1r_2\aeoneperturb\aetwoperturb + \hat{\zeta}(S;w^0))^2 + \frac{1}{2}(r_1^2(\aeoneperturb)^2 + r_2^2(\aetwoperturb)^2 + \bvperturb + \hat{\zeta}(\emptyset;w^0))^2 \\
&\leq \frac{1}{2}(-|2r_1r_2|(\eta /2)^2 + |\hat{\zeta}(S;w^0)|)^2 + \frac{1}{2}(r_1^2 \eta^2 + r_2^2 \eta^2 + |\bvperturb| + |\hat{\zeta}(\emptyset;w^0)|)^2.
\end{align*}
This implies that:
\begin{align*}
&\ellt(\wsgd) - \ellt(\wperturb) \\
&\geq \frac{1}{2}|2r_1r_2 \tau^2|^2 - |2r_1r_2 \tau^2\hat{\zeta}(S;w^0)| - \frac{1}{2}|r_1r_2 \eta^2 / 2|^2 + |r_1r_2\eta^2\hat{\zeta}(S;w^0) / 2| \\
&\quad - \frac{1}{2}(r_1^2 \eta^2 + r_2^2 \eta^2 + |\bvperturb| + |\hat{\zeta}(\emptyset;w^0)|)^2 \\
&\geq \frac{1}{2} \left(- |2r_1r_2 \tau^2\hat{\zeta}(S;w^0)| + |r_1r_2\eta^2\hat{\zeta}(S;w^0) / 2|\right) - \frac{1}{2}(r_1^2 \eta^2 + r_2^2 \eta^2 + |\bvperturb| + |\hat{\zeta}(\emptyset;w^0)|)^2,
\end{align*}
where for the second inequality we use \cref{eq:twor1r2etasquaredvszeta,eq:twor1r2tausquaredvszeta}.
Thus, using $\eta \geq 4\tau$ by \cref{eq:etalargerthantau}, we have
\begin{align}
&\ellt(\wsgd) - \ellt(\wperturb) \nonumber \\
&\geq \frac{1}{8} |r_1r_2\eta^2\hat{\zeta}(S;w^0)| - \frac{1}{2}(r_1^2 \eta^2 + r_2^2 \eta^2 + |\bvperturb| + |\hat{\zeta}(\emptyset;w^0)|)^2 \nonumber \\
&\geq \frac{1}{8} |r_1r_2\eta^2\hat{\zeta}(S;w^0)| - (r_1^2 \eta^2 + r_2^2 \eta^2 + |\hat{\zeta}(\emptyset;w^0)|)^2 - 2|\bvperturb|^2 \nonumber \\
&\geq \frac{1}{16} |r_1r_2\eta^2\hat{\zeta}(S;w^0)|- |\bvperturb|^2 &\mbox{by \cref{eq:lossdecnastytermbound}} \nonumber \\
&\geq \frac{1}{32} |r_1r_2\eta^2\hat{\zeta}(S;w^0)|. &\mbox{by \cref{eq:Fgoodinit3}} \label{eq:elltsgdperturbcomparison}
\end{align}
This now lets us prove that $\ellr(\wsgd) \geq \ellr(\wperturb)$. In the first inequality we use that $w_{-v}^0 = \wsgd_{-v} = \wperturb_{-v}$ and $w_v^0 = \vec{0}$.
\begin{align*}
&\ellr(\wsgd) - \ellr(\wperturb) = \ell(\wsgd) - \ell(\wperturb) + R(\wsgd) - R(\wperturb) \\
&\geq \ell(\wsgd) - \ell(\wperturb) + R(w^0) - R(w^0) - \frac{1}{2} \max(\lambda_1,\lambda_2) \sum_{e = (u,v) \in E} |\aeperturb|^2 \\
&= \ell(\wsgd) - \ell(\wperturb) - \frac{1}{2} \max(\lambda_1,\lambda_2) \sum_{e = (u,v) \in E} |\aeperturb|^2 \\
&\geq \ell(\wsgd) - \ell(\wperturb) - \max(\lambda_1,\lambda_2) W \eta^2
\end{align*}
In the last line we have used $|\aeperturb| \leq \eta$ for all $e = (u,v) \in E$, and also there are at most $2W$ possible edges feeding into $v$: $|\{(u,v) \in E\}| \leq 2W$.
Now, by the triangle inequality and since the idealized loss $\ellt$ is close to the true loss $\ell$,
\begin{align*}
&\ell(\wsgd) - \ell(\wperturb) \\
&\geq 
\ellt(\wsgd) - \ellt(\wperturb) - |\ellt(\wsgd) - \ell(\wsgd)| - |\ellt(\wperturb) - \ell(\wperturb)| \\
&\geq \ellt(\wsgd) - \ellt(\wperturb) \\
&\quad- C \max_{u \in V} \max_{x \in \{-1,1\}^n} (\|\wvperturb\|_{\infty}^4 + \|\wvsgd\|_{\infty}^4 + 1) (|f_u(x;w^0)|^3 + 1) (\max_{i \in \{1,2\}} \eps_i) &\mbox{by \cref{eq:idealizedlossclose}},
\end{align*}
for some large constant $C > 0$. Plugging in the bound $\|\wvperturb\|_{\infty} \leq \eta \leq 1$ by construction, and $\|\wvsgd\|_{\infty} < \tau \leq 1$ from \cref{lem:2activeprimedblank} under event $(\neg \Ecasetwo) \cap \Estat$, we have
\begin{align*}
&\ell(\wsgd) - \ell(\wperturb) \\
&\geq
\ellt(\wsgd) - \ellt(\wperturb) - C' \max_{u \in V} \max_{x \in \{-1,1\}^n} (|f_u(x;w^0)|^3 + 1) (\max_{i \in \{1,2\}} \eps_i),
\end{align*}
for some large enough constant $C' > 0$. So combining the above bounds:
\begin{align*}
\ellr(\wsgd) - \ellr(\wperturb) &\geq \frac{1}{32} |r_1r_2\eta^2\hat{\zeta}(S;w^0)| - \lambda_2 W \eta^2 - C' \max_{u \in V} \max_{x \in \{-1,1\}^n} (|f_u(x;w^0)|^3 + 1) (\max_{i \in \{1,2\}} \eps_i) \\
&> 0,
\end{align*}
by \cref{eq:etamorethanzerocond}, taking a large enough constant in \cref{eq:etamorethanzerocond}.
Thus, we conclude that $\ellr(\wperturb) < \ellr(\wsgd)$.
\end{proof}

\section{Correctness of $\TrainNetworkLayerwise$ (proof of Theorem~\ref{thm:restatedtheorem})}\label{sec:trainnetworklayerwiseproof}

In this section, we prove Theorem~\ref{thm:restatedtheorem} by using \cref{lem:atmost1active,lem:2activenotprimed,lem:2activeprimedblank,lem:2activeprimedactive,lem:2activeprimedproblowbound} to prove that certain events hold with high probability during the execution of $\TrainNetworkLayerwise$ (\cref{alg:forward}). A key property that we will prove is maintained throughout training is that every active neuron computes some monomial up to a good relative approximation. Let us formalize the notion of approximation, as it will be needed later:
\begin{definition}
Let $v \in V$ be a neuron, and let $S \subset [n]$ be a subset of indices. We say that $v$ computes the monomial $\chi_S$ up to relative error $\eps_{rel}$ if there is some $r \in \RR$ such that for all $x \in \{-1,1\}^n$ we have $$|f_v(x) - r \chi_S| \leq \eps_{rel} |r|.$$ We call $r$ the ``scaling'' factor for neuron $v$'s approximation.
\end{definition}

As an example, for any $i \in [n]$, the input $\vin{i}$ computes $x_i = \chi_{\{i\}}$ with zero relative error, since $f_{\vin{i}}(x) = x_i$ for all $x$. And furthermore $\vin{0}$ computes the monomial $1 = \chi_{\emptyset}$ with zero relative error, since $f_{\vin{0}}(x) = 1$ for all $x$.

\subsection{Definition of events}
The basis of our proof is showing that certain events and invariants hold with high probability during training. We now define them. In order to do this, recall the definition of the error function at iteration $t \in \{0,\ldots,WL\}$ of $\TrainNetworkLayerwise$:
$$\zeta(x;w^t) = f(x;w^t) - g(x),$$ and recall its Fourier coefficients: $$\hat{\zeta}(S;w^t) = \EE_{x \sim \{-1,1\}^n} [\zeta(x;w^t) \chi_S(x)].$$

\subsubsection{Representation of monomials events}

The first group of events states that all of the neurons in the neural network are either blank or represent a monomial approximately. Furthermore, they state that all low-order monomials in $g$ of degree at most $i+1$ are represented in the network after the first $i$ layers have been trained.

\begin{definition}
We say that a subset $S \subset [n]$ is represented at iteration $t$ with scaling $r$ and relative error $\eps_{rel}$ if there is a neuron $u \in V \sm \Vin$ such that for all $x \in \{-1,1\}^n$:
\begin{itemize}
    \item $f_u(x; w^t) = r \chi_S(x) + h(x)$, where
    \item $|h(x)| \leq \eps_{rel}$ and $\hat{h}(S) = 0$.
\end{itemize}
We write that neuron $u$ \textit{represents} (the monomial corresponding to) $S$.
\end{definition}

\begin{definition}[$\erel{}$, $\efouriermove$, $\elearned$]
Let $\erel{0} = 0$, and for any $i \in [L]$ inductively define 
$$\erel{i} = 16M\estop + 128 M^2 \frac{\lambda_2}{\lambda_1} (\erel{i-1})^2 +  (32M^2 \sqrt{\frac{\lambda_1}{\lambda_2}} + 1)\erel{i-1}$$

Furthermore, define
$$\efouriermove = 100M^2L\estop$$

And let
$$\elearned = 16M^2 \lambda_2$$
\end{definition}

For any $i \in \{0,1,\ldots,L\}$, let $$t_i = W i$$ be the iteration at which layers $1,\ldots,i$ have been trained in $\TrainNetworkLayerwise$.

\begin{definition}
For any $t \in (t_{i-1},t_i]$, let $\Erep{S,t}$ be the event that at time $t$ there is exactly one neuron $u_S$ representing $S$, with relative error $\eps_S \leq \erel{i}$, and with scaling factor $r_S$ such that $|r_S - \hat{g}(S)| \leq \efouriermove t + \elearned$.
\end{definition}

\begin{definition}
For any $i \in [L]$ and $t \in (t_{i-1},t_i]$, let $\Enobadactive{t}$ be the event that, for any neuron $v \in V \sm \Vin$ that is active at iteration $t$ (i.e., such that there exists $x$ with $f_v(x;w^t) \neq 0$), $v$ represents $S$ such that $\hat{g}(S) \neq 0$ and $|S| \leq i+1$.
\end{definition}

\begin{definition}[Event: first $i$ layers represent all monomials of degree at most $i+1$]
For convenience of notation, let $\Ereplayer{0}$ to be an event that always occurs. For any $i \in \{1,\ldots,L\}$, let $\Ereplayer{i}$ be the event that:
 $\Erep{S,t_i}$ holds for each $S \subset [n]$ such that $|S| \leq i+1$ and $\hat{g}(S) \neq \emptyset$, and that $\Enobadactive{t_i}$ holds.
\end{definition}

\begin{definition}[Polarization of Fourier coefficients]
For any $t \in [WL]$, let $\Epol{t}$ be the event that for any $S \subset [n]$;
    \begin{itemize}
        \item If $\hat{g}(S) = 0$, then $|\hat{\zeta}(S; w^t)| \leq \efouriermove t$.
        \item If $\hat{g}(S) \neq 0$ and $S$ is represented in the network at time $t$, then $|\hat{\zeta}(S; w^t)| \leq \elearned + \efouriermove t$.
        \item If $\hat{g}(S) \neq 0$ and $S$ is not represented in the network at time $t$, then $|\hat{\zeta}(S; w^t) + \hat{g}(S)| \leq \efouriermove t$.
    \end{itemize}
\end{definition}

\subsubsection{Boundedness of bias and network parameters invariants}
In order to apply the guarantees for $\TrainNeuron$, we also need to maintain certain technical events that ensure that the parameters and weights of the network do not blow up too much. This ensure smoothness of the objective during training.

\begin{definition}
For any $t \in \{0,\ldots,t_L\}$, let $\Ebias{t}$ be the event that $|\hat{\zeta}(\emptyset; w^t)| \leq 2\estop \leq \estat$. In other words, this is the event that on iteration $t$ the error is unbiased.
\end{definition}

\begin{definition}
For any $t \in \{0,\ldots,t_L\}$, the $\Eneurbound{t}$ event is that all neurons at iteration $t$ have magnitude upper-bounded by $2M$: i.e.,
\begin{equation}\max_{u \in V} \max_{x \in \{-1,1\}^n} |f_u(x;w^t)| \leq 2M. \label{eq:gooduneurbound}\end{equation}
\end{definition}

\begin{definition}
For any $t \in \{0,\ldots,t_L\}$, the $\Eparambound{t}$ event is that at iteration $t$ we have the following bound on the trained weights:
\begin{align}\max_{e \in E} |a_e^t| \leq 16M^2 \sqrt{\lambda_2 / \lambda_1}.\label{eq:goodparambound}
\end{align}
\end{definition}

\subsubsection{Network connectivity events}

Finally, we have certain events that control the connectivity structure of the network. First, we ensure (because of the sparsity of the network), that every neuron has at most two active inputs and if it has two then one of them is from $\Vin$.
\begin{definition}
For simplicity of the definition, let $V_0 = \emptyset$.

For any $i \in \{0,\ldots,L-1\}$, $\Enothree{i}$ be the event that after training layers $1,\ldots,i$ (i.e., at iteration $t_{i}$), there is no $v_i \in V_i$ such that $$|\{(u,v_i) \in E \mbox{ such that } u \mbox{ is active at iteration } t_i\}| \geq 3,$$ and also there is no $v_i \in V_i$ such that $$|\{(v_{i-1},v_i) \in E \mbox{ such that } v_{i-1} \mbox{ is active at iteration $t_i$ and } v_{i-1} \in V_{i-1}\}| \geq 2.$$
\end{definition}

Second, we also ensure that the network architecture is sufficiently connected that the product of any pair of trained neurons can be learned.

\begin{definition}
For simplicity of the definition, let $V_0 = \emptyset$. Let
\begin{align}
\nshared = 64M^2 \log(16sL/\delta) \label{eq:setnshared}
\end{align}
For any $i \in \{0,\ldots,L-1\}$, let $\Econn{i}$ be the event that, at iteration $t_i$, for all distinct pairs of neurons $u,u' \in \{u \in V_{i-1} \cup \Vin \mbox{ s.t. } u \mbox{ is active at iteration $t_i$}\}$, $$|\{v_i \in V_i : (u,v_i) \in E, (u',v_i) \in E\}| \geq \nshared.$$
\end{definition}

\subsection{$\Ereplayer{L}$ suffices to ensure learning}
We now show that the $\Ereplayer{L}$ event is enough to prove that the loss is bounded by $\eps$ at the final iteration $t_L$ (proved in \cref{lem:ereplayerLsuffices}). Thus, the goal of the remainder of the proof will be to show that $\Ereplayer{L}$ occurs with high probability.

\begin{claim}[Bounded relative error during training]\label{claim:erelglobbound}
For all $i \in \{0,\ldots,L\}$, $\erel{i} \leq (1 + 1/L)^i (16Mi \estop)$. In particular, $\erel{i} \leq 45ML\estop \leq \eps / (2Ms) \leq 1/2$.
\end{claim}
\begin{proof}
The proof is by induction on $i$. In the base case $i = 0$ we have $\erel{0} = 0$, satisfying the bound. For $i \in \{1,\ldots,L\}$, we have by the inductive hypothesis
\begin{align*}
\erel{i} &= 16M\estop + 128 M^2 \frac{\lambda_2}{\lambda_1} ((1 + 1/L)^{i-1} (16M(i-1) \estop))^2 \\ &\quad\quad +  (32M^2 \sqrt{\frac{\lambda_1}{\lambda_2}} + 1)(1 + 1/L)^{i-1} (16M(i-1) \estop).
\end{align*}
Since
\begin{align*}
128 M^2 &\frac{\lambda_2}{\lambda_1} ((1 + 1/L)^{i-1} (16M(i-1) \estop)) \\
&\leq 128 M^2 \frac{\lambda_2}{\lambda_1} ((1 + 1/L)^{L} (16ML \estop)) \\
&= 128 M^2 (64M^2 L)^2 ((1 + 1/L)^{L} (16ML \estop)) &\mbox{by \cref{eq:setlambda1lambda2ratio}} \\
&= 2^{25} M^7 L^3 \estop \\
&\leq 1/(2L) &\mbox{by \cref{eq:setestop}}
\end{align*}
and
\begin{align*}
32M^2 \sqrt{\frac{\lambda_1}{\lambda_2}} &= 1/(2L),
\end{align*}
this means
\begin{align*}
\erel{i} &\leq 16M\estop + (1/(2L) + 1/(2L))(1 + 1/L)^{i-1} (16M(i-1) \estop) \\
&\leq 16M\estop + (1 + 1/L)^{i} (16M(i-1) \estop) \\
&\leq (1 + 1/L)^{i} (16Mi \estop).
\end{align*}
This proves the first part of the claim. To see the second part, note that, for any $i \in \{0,\ldots,L\}$
\begin{align*}
\erel{i} &\leq (1 + 1/L)^L (16ML\estop) \leq 45 ML\estop \leq \eps / (2Ms) \leq 1/2.
\end{align*}
\end{proof}

\begin{claim}[Bounded error in Fourier coefficients during training]
\label{claim:efourmoveglobbound}
For any $t \leq WL$,
$$\efouriermove t + \elearned \leq 32 M^2 \lambda_2 \leq \eps / (4Ms) \leq 1/(4M).$$
\end{claim}
\begin{proof}
\begin{align*}\efouriermove t + \elearned &\leq \efouriermove WL + \elearned \\
&= 100M^2WL^2\estop + 16M^2 \lambda_2 \\
&\leq 32M^2 \lambda_2 &\mbox{by $\estop \leq \lambda_2 / (100WL^2)$ in \cref{eq:setestop}} \\
&\leq \eps / (4Ms) &\mbox{by $\lambda_2 \leq \eps / (128M^3s)$ in \cref{eq:setlambda2}} \\
&\leq 1/(4M). &\mbox{since $s \geq 1$, $\eps \leq 1$}
\end{align*}
\end{proof}

\begin{lemma}\label{lem:ereplayerLsuffices}
If $\Ereplayer{L}$ holds, then $\ell(w^{t_L}) \leq \eps$.
\end{lemma}
\begin{proof}
Since $L \geq n-1$, the event $\Ereplayer{L}$ states that the active neurons of the network at iteration $t_L$ are in bijective correspondence with the subsets $S \subset [n]$ such that $\hat{g}(S) \neq 0$. In other words, for each $S$ such that $\hat{g}(S) \neq 0$ there is exactly one neuron $u_S$ such that $f_{u_S}(x;w^{t_L}) = r_S \chi_{S}(x) + h_S(x)$ where $|h_S(x)| \leq \erel{L} |r_S|$ and $|r_S - \hat{g}(S)| \leq \efouriermove t_i + \elearned$. Furthermore, there are no other active neurons, meaning that:
$$f(x;w^{t_L}) = \sum_{S : \hat{g}(S) \neq 0} f_{u_S}(x;w^{t_L}).$$
This implies that the error function is always bounded:
\begin{align*}|\zeta(x;w^{t_L})| &= |f(x;w^{t_L}) - g(x)| \\
&= |\sum_{S : \hat{g}(S) \neq 0} f_{u_S}(x;w^{t_L}) - \hat{g}(S) \chi_S(x)| \\
&= |\sum_{S : \hat{g}(S) \neq 0} (r_S - \hat{g}(S)) \chi_S(x) + h_S(x)| \\
&\leq \sum_{S : \hat{g}(S) \neq 0} |r_S - \hat{g}(S)| + |h_S(x)| \\
&\leq \sum_{S : \hat{g}(S) \neq 0} |r_S - \hat{g}(S)| + \erel{L} |r_S| \\
&\leq \sum_{S : \hat{g}(S) \neq 0} |r_S - \hat{g}(S)|(1 + \erel{L}) + \erel{L} |\hat{g}(S)| \\
&\leq \sum_{S : \hat{g}(S) \neq 0} 2|r_S - \hat{g}(S)| + \erel{L} |\hat{g}(S)| &\mbox{by \cref{claim:erelglobbound}} \\
&\leq \sum_{S : \hat{g}(S) \neq 0} 2(\efouriermove t_L + \elearned) +  M\erel{L} \\
&\leq 2s(\efouriermove t_L + \elearned) + Ms\erel{L} \\
&\leq 2s(\efouriermove t_L + \elearned) + \eps / 2 &\mbox{by \cref{claim:erelglobbound}} \\
&\leq \eps / 2 + \eps /2 &\mbox{by \cref{claim:efourmoveglobbound}} \\
&= \eps.
\end{align*}
As a consequence, we may bound the loss at the final time step $t_L$:
\begin{align*}
    \ell(w^{t_L}) &= \EE_{x \sim \{-1,1\}^n}[\ell(x;w^{t_L})] \\
    &= \EE_{x \sim \{-1,1\}^n}[\zeta(x;w^{t_L})^2] \\
    &\leq \EE_{x \sim \{-1,1\}^n}[\eps^2] \\
    &\leq \eps^2 \\
    &\leq \eps.
\end{align*}
\end{proof}

\subsection{$\Ereplayer{L}$ occurs with high probability}

In this section, we prove that $\Ereplayer{L}$ occurs with high probability, essentially concluding the proof of the theorem because of \cref{lem:ereplayerLsuffices}. First, we define the intersection of the events defined above, which we will show holds with high probability by induction on the iteration number.
\begin{definition}\label{def:estepgood}
Define the event $$\Estepgood{0} = \Epol{0} \cap \Eneurbound{0} \cap \Ebias{0} \cap \Eparambound{0} \cap \Enobadactive{0}.$$
For any $t \in \{1,\ldots,t_L\}$, inductively define the event $$\Estepgood{t} = \Estepgood{t-1} \cap \Epol{t} \cap \Eneurbound{t} \cap \Eparambound{t} \cap \Ebias{t} \cap \Enobadactive{t}.$$
\end{definition}

\begin{definition}\label{def:elayergood}
Define the event $$\Elayergood{0} = \Econn{0} \cap \Enothree{0} \cap \Ereplayer{0}.$$
For any $i \in \{1,\ldots,L\}$, inductively define the event
$$\Elayergood{i} = \Elayergood{i-1} \cap \Ereplayer{i} \cap \Econn{i} \cap \Enothree{i}.$$
\end{definition}

\subsubsection{$\Estepgood{t}$ follows from $\Estepgood{t-1}$ and $\Elayergood{i-1}$ with high probability}

The first element of our induction is given by \cref{lem:goodstepinduction} below. It bounds the runtime of an iteration of $\TrainNetworkLayerwise$ and proves that with high probability the event $\Estepgood{}$ continues to hold. First, we prove a couple of helper claims.

\begin{claim}
Let $t \in \{0,\ldots,t_L\}$. Under the event $\Estepgood{t}$, we have
\begin{align}
    \ellr(w^t) \leq 2^{22} W^2L^3M^8s^2
    \label{eq:goodlrbound}
\end{align}
\end{claim}
\begin{proof}
We bound the regularized loss at $w^t$, using that under the event $\Estepgood{t}$, both $\Eneurbound{t}$ and $\Eparambound{t}$ hold:
\begin{align*}\ellr(w^t) &= \ell(w^t) + R(w^t) \\ 
&= \EE_{x \sim \{-1,1\}^n}[\ell(x;w^t)] + R(w^t) \\
&\leq \max_x |\ell(x;w^t)| + R(w^t) \\
&\leq \max_x (f(x;w^t) - g(x))^2 + R(w^t) \\
&\leq \max_x 2g(x)^2 + 2f(x;w^t)^2 + R(w^t) \\
&\leq (Ms)^2 + 2f(x;w^t)^2 + R(w^t) &\mbox{by \cref{claim:gisupperbounded}} \\
&\leq (Ms)^2 +2WL\max_u \max_x f_u(x;w^t)^2 + R(w^t) \\
&\leq (Ms)^2 + 8WLM^2 + R(w^t) &\mbox{by \cref{eq:gooduneurbound}, since $\Eneurbound{t}$ holds} \\
&\leq (Ms)^2 + 8WLM^2 + \sum_{e \in E} (a_e^t)^2 &\mbox{since $\lambda_1,\lambda_2 \leq 1$} \\
&\leq (Ms)^2 + 8WLM^2 + 2W^2L \max_e |a_e^t|^2 \\
&\leq (Ms)^2 + 8WLM^2 + 2W^2L (16M^2 \sqrt{\lambda_2/\lambda_1})^2 &\mbox{by \cref{eq:goodparambound}, since $\Eparambound{t}$ holds} \\
&\leq (Ms)^2 + 8WLM^2 + 2W^2L (2^{20} M^8 L^2) &\mbox{by \cref{eq:setlambda1lambda2ratio}} \\
&\leq 2^{22} W^2L^3M^8s^2.
\end{align*}
\end{proof}

Let $\Estatiter{t}$ denote the event that on call $t$ to $\TrainNeuron$, the event $\Estat$ from \cref{lem:trainneuronstationarity} holds.

\begin{lemma}\label{lem:goodestatlowbound}
Let
\begin{align}
\dstat = \delta / (64 LWsM^2) \label{eq:setdstat}.
\end{align}
For any $t \in (t_{i-1},\ldots,t_i]$, if the event $\Estepgood{t-1} \cap \Elayergood{i-1}$ holds, then
$\PP[\Estatiter{t} \mid w^{t-1}] \geq 1 - \dstat$.
Furthermore, if $\Estat{t}$ holds then the call to $\TrainNeuron$ exits after at most $O(\kt{2393})$ time.
\end{lemma}
\begin{proof}
The lemma follows by applying \cref{lem:trainneuronstationarity}. 
Indeed, for some large enough constant $C$ so that we can apply \cref{lem:trainneuronstationarity}, we bound the learning rate by taking $c_{\alpha} > 0$ small enough:
\begin{align*}
\alpha &\leq 1/(C (\lambda_1\lambda_2)^{-5}  2^{56}\kt{72}) &\mbox{by \cref{eq:setalpha}} \\
&\leq 1/(C (\lambda_1\lambda_2)^{-5} \kt{16} ( (2M)^{32} + 1) (\kt{8} ((2M)^{16} + 1) + 2^{22}\kt{8}) \\
&\leq 1/(C (\lambda_1\lambda_2)^{-5} \kt{16} ( (2M)^{32} + 1) (\kt{8} ((2M)^{16} + 1) + \ellr(w^0)^4)) &\mbox{by \cref{eq:goodlrbound}} \\
&< \min_{u \in V} \min_{x \in \{-1,1\}^n} (C(\lambda_1\lambda_2)^{-5}\kt{16})^{-1}( |f_u(x; w^0)|^{32} + 1)^{-1} \\
&\quad\quad\qquad\qquad\qquad \cdot (\kt{8} (|f_u(x;w^0)|^{16} + 1) + \ellr(w^0)^4)^{-1} &\mbox{by $\Eneurbound{t-1}$}
\end{align*}
the bound on the number of iterations in \cref{lem:trainneuronstationarity} is at most:
\begin{align*}
\tbound &= C(\kt{2}(\max_{u \in V} \max_{x \in \{-1,1\}^n} |f_u(x; w^0)|^4 + 1) + \ellr(w^{t-1})) / (\alpha (\estop)^2) \\
&\leq C(\kt{2}((2M)^4 + 1) + \ellr(w^{t-1}) / (\alpha (\estop)^2) &\mbox{by $\Eneurbound{t-1}$} \\
&\leq C(\kt{2}((2M)^4 + 1) + 2^{22} \kt{8}) / (\alpha (\estop)^2) &\mbox{by \cref{eq:goodlrbound}} \\
&\leq 2^{23}C\kt{10} / (\alpha (\estop)^2) \\
&\leq C^2 (\lambda_1\lambda_2)^{-5} \kt{82} / (\estop)^2 &\mbox{by \cref{eq:setalpha}} \\
&\leq C^3 (\lambda_1\lambda_2)^{-5} \kt{942} &\mbox{by \cref{eq:setestop}}
\end{align*}
and the minibatch size satisfies the following because the constant $c_B$ is large enough:
\begin{align*}
B &\geq C^4 (\lambda_1\lambda_2)^{-4} \kt{910} &\mbox{by \cref{eq:setB}} \\
&\geq C^3(\lambda_1\lambda_2)^{-4} \kt{909} \log(1/\dstat) &\mbox{by \cref{eq:setdstat}} \\
&\geq C^2(\lambda_1\lambda_2)^{-3} (2^{98} \kt{908}) \log(2\tbound/\dstat) &\mbox{by $\tbound$ bound} \\
&\geq C(\lambda_1\lambda_2)^{-3} (2^{98} \kt{48}) \log(2\tbound/\dstat) / \estop^2 &\mbox{by \cref{eq:setestop}} \\
&\geq C(\lambda_1\lambda_2)^{-3} \kt{8} ((2M)^8 + 1)(\kt{8}(2M)^{16}  + 2^{88} \kt{32}) \log(2\tbound/\dstat) / \estop^2 \\
&\geq C(\lambda_1\lambda_2)^{-3} \kt{8} ((2M)^8 + 1)(\kt{8}(2M)^{16}  + \ellr(w^0)^4) \log(2\tbound/\dstat) / \estop^2 &\mbox{by \cref{eq:goodlrbound}} \\
&\geq C (\lambda_1\lambda_2)^{-3}\kt{8} (\max_{u \in V} \max_{x \in \{-1,1\}^n} |f_u(x; w^0)|^8 + 1) \\
&\quad\quad\qquad\qquad \cdot (\kt{8}(\max_u |f_u(x; w^0)|^{16} + 1) + \ellr(w^0)^4) \log(2\tbound / \dstat) / \estop^2,
\end{align*}
Thus, we can apply \cref{lem:trainneuronstationarity} and derive the claimed bounds, including the runtime bound of $O(\kt{}B\tbound) = O((\lambda_1\lambda_2)^{-9} \kt{1853}) = O(\kt{2393})$.
\end{proof}

Now we are ready to prove the main result of this subsection, which is the inductive step showing that $\Estepgood{t}$ is maintained with high probability. The proof calls on the guarantees on $\TrainNeuron$ proved in \cref{lem:atmost1active,lem:2activenotprimed,lem:2activeprimedblank,lem:2activeprimedactive,lem:2activeprimedproblowbound}.
\begin{lemma}\label{lem:goodstepinduction}
For any layer $i \in [L]$ and iteration $t \in [WL]$ such that $t \in [t_{i-1}+1,t_i]$ and $\Estepgood{t-1} \cap \Elayergood{i-1}$ holds, then
$$\PP[\Estepgood{t} \cap \Estatiter{t} \mid w^{t-1}] \geq 1 - \dstat.$$
Furthermore, if the neuron $v \in V_i$ trained in iteration $t$ has exactly two active parents representing $S_1$ and $S_2$, and $\hat{g}(S) \neq \emptyset$, and $\neg \Erep{S,t-1}$ holds for $S = (S_1 \cup S_2) \sm (S_1 \cap S_2)$, then with lower-bounded probability $v$ is trained to be a neuron that represents $S$:
$$\PP[\Erep{S,t} \mid w^{t-1}] \geq 1/(64M^2).$$
\end{lemma}
\begin{proof}
We prove that if $\Estatiter{t}$ occurs then $\Estepgood{t}$ also occurs. This suffices to prove the first part of the claim since $\PP[\Estatiter{t} \mid w^{t-1}] \geq 1 - \dstat$ by \cref{lem:goodestatlowbound}.

Let $v \in V_i$ in layer $i$ be the neuron that we update with $\TrainNeuron$ on iteration $t$. Then by event $\Enothree{i-1}$, $v$ must have at most two active parents at iteration $t-1$: i.e.,
$$|\{u \in V : (u,v) \in E \mbox{ and }\exists x\mbox{ s.t. } f_{u}(x;w^{t-1}) \neq 0\}| \leq 2.$$

For ease of notation, let $u_1,u_2 \in V_{i-1} \cup \Vin$ be two parents of $v$ such that $(u_1,v),(u_2,v) \in E$, and such that all other parents are blank\footnote{This notation assumes that $v$ has at least two parents, but this is only for the sake of convenience since the case where $v$ has no parents or one parent essentially follows by the same arguments, letting $r_1 = r_2 = 0$ or $r_2 = 0$.}: if $(u,v) \in E$ and $u \not\in \{u_1,u_2\}$, then $f_u(x;w^{t-1}) \equiv 0$. Write the functions computed at $u_1,u_2$ as
$$f_{u_j}(x; w^{t-1}) = r_j \chi_{S_j}(x) + h_j(x),$$
where $r_1,r_2 \in \RR$, $S_1,S_2 \subset [n]$, and $\hat{h}_j(S_j) = 0$. Since $\Elayergood{i-1}$ implies $\Ereplayer{i-1}$, we know that for any active neuron $u' \in V_{i-1} \cup \Vin$ we have $f_{u'}(x;w^{t-1}) = r' \chi_{S'}(x) + h'(x)$, where $\hat{h'}(S') = 0$ and $|r'| \leq |r' - \hat{g}(S')| + |\hat{g(S')}| \leq \elearned + \efouriermove (t-1) + M \leq 1/(4M) + M \leq 2M$ by \cref{claim:efourmoveglobbound}, so
\begin{align}
|r_1|,|r_2| \in \{0\} \cup [1/2M, 2M], \label{eq:goodr1r2interval}
\end{align} and there are $\eps_1,\eps_2 > 0$ satisfying
\begin{align}
\eps_1, \eps_2 \leq \erel{i-1} \label{eq:goodeps1eps2bound}
\end{align} such that for all $x \in \{-1,1\}^n$,
$|h_j(x)| \leq |r_j|\eps_j$.
Therefore, we may bound $\estat(w^{t-1},\eps_1,\eps_2)$ and $\Ustat(w^{t-1})$:
\begin{claim}
Under $\Estepgood{t-1}$ and $\Elayergood{i-1}$, we have
\begin{align}
\estat(w^{t-1},\eps_1,\eps_2) \lsim (\lambda_1\lambda_2)^{-3}\kt{67} \estop \leq 1
    \label{eq:goodestatbound}
\end{align}
\begin{align}
    \Ustat(w^{t-1}) \lsim (\lambda_1 \lambda_2)^{-1}\kt{20}\label{eq:goodustatbound}
\end{align}
\end{claim}
\begin{proof}
We bound $\estat(w^{t-1},\eps_1,\eps_2)$, first recalling that by the definition in \cref{lem:trainneuronstationarity},
\begin{align*}
&\estat(w^{t-1},\eps_1,\eps_2) \\
&\lsim \estop + \max_{u \in V} \max_{x \in \{-1,1\}^n} (\lambda_1\lambda_2)^{-3}\kt{12}(\ellr(w^{t-1})^3 + 1)(|f_u(x;w^{t-1})|^{30} + 1)(\max_{i \in \{1,2\}} \eps_i) \\
&\lsim \estop + \max_{u \in V} \max_{x \in \{-1,1\}^n} (\lambda_1\lambda_2)^{-3}\kt{12}(\ellr(w^{t-1})^3 + 1)(|f_u(x;w^{t-1})|^{30} + 1)\erel{i-1} &\mbox{by \cref{eq:goodeps1eps2bound}} \\
&\lsim \estop + (\lambda_1\lambda_2)^{-3}\kt{12} (W^2L^3M^8s^2)^3 (|f_u(x;w^{t-1})|^{30} + 1)\erel{i-1} &\mbox{by \cref{eq:goodlrbound}} \\
&\lsim \estop + (\lambda_1\lambda_2)^{-3}\kt{12} (W^2L^3M^8s^2)^3 M^{30}\erel{i-1} &\mbox{by \cref{eq:gooduneurbound}} \\
&\leq \estop + (\lambda_1\lambda_2)^{-3}\kt{12 + 24 + 30} \erel{i-1} \\
&= \estop + (\lambda_1\lambda_2)^{-3}\kt{66} \erel{i-1} \\
&\lsim (\lambda_1\lambda_2)^{-3}\kt{67} \estop &\mbox{by \cref{claim:erelglobbound}}.
\end{align*}
In the above bounds, we have used that \cref{eq:goodlrbound} holds because $\Estepgood{t-1}$ holds, and \cref{eq:gooduneurbound} holds because $\Eneurbound{t-1}$ holds by $\Estepgood{t-1}$.

Similarly, we bound $\Ustat(w^{t-1},\eps_1,\eps_2)$, recalling the definition in \cref{lem:trainneuronstationarity}:
\begin{align*}
     \Ustat(w^{t-1}) &\lsim (\lambda_1\lambda_2)^{-1}\kt{4} (\max_u \max_{x \in \{-1,1\}^n} |f_u(x;w^{t-1})|^{8} + 1)(\ellr(w^{t-1}) + 1) \\
     &\lsim (\lambda_1\lambda_2)^{-1}\kt{4}(\max_u \max_{x \in \{-1,1\}^n} |f_u(x;w^{t-1})|^{8} + 1)\kt{8} &\mbox{by \cref{eq:goodlrbound}} \\
     &\lsim (\lambda_1\lambda_2)^{-1}\kt{4} M^{8}\kt{8} &\mbox{by \cref{eq:gooduneurbound}} \\
     &= (\lambda_1 \lambda_2)^{-1}\kt{20}.
\end{align*}
Here again, we have used \cref{eq:goodlrbound} and \cref{eq:gooduneurbound} because $\Estepgood{t-1}$ holds.
\end{proof}

We may now break the analysis into cases, writing $\estat = \estat(w^{t-1},\eps_1,\eps_2)$ and $\Ustat = \Ustat(w^t)$ for shorthand.

\paragraph{Case 1: At most one active input} If $v$ has at most one active parent at iteration $t-1$, then $f_{u_2}(x;w^{t-1}) \equiv 0$ without loss of generality.

\textit{Checking preconditions of \cref{lem:atmost1active}}. In this case, we apply \cref{lem:atmost1active}, first checking that the preconditions apply.
In the below, let taking $C > 0$ to be a large enough universal constant.
\cref{eq:1acond1} applies since
\begin{align*}
2\estat/ \min(\lambda_1,\lambda_2) &\leq 2\estat / (\lambda_1\lambda_2) \\
&\leq C(\lambda_1\lambda_2)^{-4} \kt{67} \estop &\mbox{by \cref{eq:goodestatbound}} \\
&\leq 1 / (2^{21} M^7 L) &\mbox{by $\estop \leq (\lambda_1\lambda_2)^{4} \kt{-74} / (2^{21} C)$ in \cref{eq:setestop}} \\
&< \tau &\mbox{by \cref{eq:settau}}
\end{align*}
Furthermore, \cref{eq:1acond2} applies, since
\begin{align*}
&|\hat{\zeta}(\emptyset;w^0)| + \estat + (r_1)^2|2\estat / \min(\lambda_1,\lambda_2)|^2 \\
&\leq 2\estat + (r_1)^2|2\estat / \min(\lambda_1,\lambda_2)|^2 &\mbox{by $\Ebias{t-1}$} \\
&\leq  C(\lambda_1 \lambda_2)^{-3} \kt{67}\estop (1 + (r_1)^2| 2 / \min(\lambda_1,\lambda_2)|^2) &\mbox{by \cref{eq:goodestatbound}} \\
&\leq C(\lambda_1 \lambda_2)^{-3} \kt{67}\estop (1 + (4M /\lambda_1)^2) &\mbox{since $|r_1| \leq 2M$} \\
&\leq 64C(\lambda_1 \lambda_2)^{-3} \kt{67}\estop / \lambda_1^2 &\mbox{by \cref{eq:settau}} \\
&\leq \tau
\end{align*}
Finally, \cref{eq:1acond3} applies since
\begin{align*}
\estat &< C(\lambda_1\lambda_2)^{-3} \kt{67} \estop &\mbox{by \cref{eq:goodestatbound}} \\
&\leq \lambda_1/(4M)^2 &\mbox{by $\estop \leq (\lambda_1\lambda_2)^4 \kt{-69} / (4C)$ in \cref{eq:setestop}} \\
&\leq \min(\lambda_1,\lambda_2) / (2 r_1)^2 &\mbox{by $|r_1| \leq 2M$}
\end{align*}
Thus, all the preconditions of \cref{lem:atmost1active} hold.

\textit{Applying \cref{lem:atmost1active}}. Therefore, under event $\Estatiter{t}$ after running $\TrainNeuron(v,w^{t-1})$ we have $w^t = w^{t-1}$. So since the weights are unchanged and we assume that $\Estatiter{t}$ holds, $\Estepgood{t}$ follows from $\Estepgood{t-1}$.

\paragraph{Case 2: Exactly two active inputs} If $v$ has exactly two active parents at iteration $t-1$, then by \cref{eq:goodr1r2interval}, we must have $|r_1|,|r_2| \in [1/(2M), 2M]$. Furthermore, let $S = S_1 \cup S_2 \sm (S_1 \cap S_2)$, so that $\chi_{S}(x) = \chi_{S_1}(x) \chi_{S_2}(x)$. We further subdivide into two cases, depending on whether $S$ is represented by a neuron in the network and whether $\hat{g}(S) = 0$.

In order to analyze this section, let us first upper-bound the quantities 
$\eps^{(1)},\eps^{(2)},\eps^{(3)}$,
\begin{align}
\eps^{(1)} &= 2\estat(w^{t-1},\eps_1,\eps_2) / \min(\lambda_1,\lambda_2) \nonumber \\
&\lsim (\lambda_1\lambda_2)^{-3}\kt{67} \estop / \min(\lambda_,\lambda_2) \nonumber \\
&\leq (\lambda_1\lambda_2)^{-4}\kt{67} \estop. \label{eq:eps1bound}
\end{align}
\begin{align}
\eps^{(2)} &= (1 + \max_i r_i^2 \Ustat(w^{t-1}))\estat(w^{t-1},\eps_1,\eps_2) \nonumber \\
&\leq (1 + (2M)^2)\Ustat(w^{t-1}))\estat(w^{t-1},\eps_1,\eps_2) &\mbox{by \cref{eq:goodr1r2interval}} \nonumber \\
&\lsim (1 + (2M)^2)(\lambda_1\lambda_2)^{-1} \kt{20}\estat(w^{t-1},\eps_1,\eps_2) &\mbox{by \cref{eq:goodustatbound}} \nonumber \\
&\lsim (1 + (2M)^2)(\lambda_1\lambda_2)^{-1} \kt{20} (\lambda_1\lambda_2)^{-3} \kt{67} \estop &\mbox{by \cref{eq:goodestatbound}} \nonumber \\
&\lsim (\lambda_1\lambda_2)^{-4} \kt{89} \estop \label{eq:eps2bound}
\end{align}
And since $\Estepgood{t-1}$ implies $\Epol{t-1}$, we have
\begin{align}
|\hat{\zeta}(S;w^{t-1})| &\leq |\hat{g}(S)| + (t-1)\efouriermove + \elearned \nonumber \\
&\leq |\hat{g}(S)| + (M/4) &\mbox{by \cref{claim:efourmoveglobbound}} \nonumber \\
&\leq 2M &\mbox{by $|\hat{g}(S)| \leq M$} \label{eq:zetahatfouriercoeffbound}\end{align}
\begin{align}
\eps^{(3)} &= 8\eps^{(2)} (1 + (\Ustat)^2 + |\hat{\zeta}(S;w^{t-1})|)\max(1,|r_1r_2|^2) / \min(\lambda_1,\lambda_2)^2 \nonumber \\
&\lsim ((\lambda_1\lambda_2)^{-4} \kt{89} \estop) (1 + (\Ustat)^2 + |\hat{\zeta}(S;w^{t-1})|)\frac{\max(1,|r_1r_2|^2)}{\min(\lambda_1,\lambda_2)^2} \nonumber &\mbox{by \cref{eq:eps2bound}} \\
&\leq ((\lambda_1\lambda_2)^{-6} \kt{89} \estop)(1 + (\Ustat)^2 + |\hat{\zeta}(S;w^{t-1})|)\max(1,|r_1r_2|^2) \nonumber \\
&\lsim  ((\lambda_1\lambda_2)^{-6} \kt{93} \estop)(1 + (\Ustat)^2 + |\hat{\zeta}(S;w^{t-1})|) &\mbox{by \cref{eq:goodr1r2interval}} \nonumber \\
&\lsim ((\lambda_1\lambda_2)^{-6} \kt{93} \estop)((\lambda_1,\lambda_2)^{-2} \kt{40} + |\hat{\zeta}(S;w^{t-1})|) &\mbox{by \cref{eq:goodustatbound}} \nonumber \\
&\lsim (\lambda_1\lambda_2)^{-8} \kt{133} \estop. &\mbox{by \cref{eq:zetahatfouriercoeffbound}} \label{eq:eps3bound}
\end{align}

\paragraph{Case 2a: Exactly two active inputs, product is not useful} Consider the case in which either $\hat{g}(S) = 0$ or $S$ is already represented by some neuron at iteration $t-1$ (i.e., $\Erep{S,t-1}$ holds). Since $\Estepgood{t-1}$ implies $\Epol{t-1}$, we have the following bound
\begin{align*}|\hat{\zeta}(S;w^{t-1})| &\leq  (t-1)\efouriermove + \elearned \\
&\leq 32M^2\lambda_2. &\mbox{by \cref{claim:efourmoveglobbound}}
\end{align*} 

Using this, we will prove that the $\TrainNeuron(v;w^{t-1})$ will with high probability leave the neuron $v$ blank and the weights unchanged after training. Intuitively, this is because if neuron $v$ were trained to represent the monomial $\chi_S$, then this would not reduce the loss significantly since the Fourier coefficient $\hat{\zeta}(S;w^{t-1})$ of the error is small. Therefore the regularization term dominates and pushes the trained weights on this iteration to close to zero.

\textit{Checking preconditions of \cref{lem:2activenotprimed}}. We apply \cref{lem:2activenotprimed} to conduct our analysis. In order to check that the preconditions are satisfied, let us first upper bound the quantity $\eps^{(4)}$.
\begin{align}
\eps^{(4)} &= 2(\sqrt{\frac{\max(\lambda_1,\lambda_2)}{ \min(\lambda_1,\lambda_2)}}\sqrt{\max(\lambda_1, \lambda_2) + |\hat{\zeta}(S;w^{t-1})|}  + \eps^{(2)}\frac{\max(\lambda_1,\lambda_2))}{\min(1,|r_1r_2|^2)} \nonumber \\
&= 2(\sqrt{\lambda_2 / \lambda_1}\sqrt{\lambda_2 + |\hat{\zeta}(S;w^{t-1})|}  + \lambda_2 \eps^{(2)}) / \min(1,|r_1r_2|^2) \nonumber \\
&\lsim \kt{4}(\sqrt{\lambda_2 / \lambda_1}\sqrt{\lambda_2 + |\hat{\zeta}(S;w^{t-1})|}  + \lambda_2 \eps^{(2)}) &\mbox{by \cref{eq:goodr1r2interval}} \nonumber \\
&\lsim \kt{4}(\sqrt{\lambda_2 / \lambda_1}\sqrt{\lambda_2 + |\hat{\zeta}(S;w^{t-1})|}  + (\lambda_1\lambda_2)^{-4} \kt{89} \estop) \nonumber &\mbox{by \cref{eq:eps2bound}} \\
&\lsim \kt{4}(\sqrt{\lambda_2 / \lambda_1}\sqrt{\lambda_2 + |\hat{\zeta}(S;w^{t-1})|})  + (\lambda_1\lambda_2)^{-5} \kt{93} \estop \nonumber \\
&\lsim \kt{4}\sqrt{\lambda_2 / \lambda_1}\sqrt{\lambda_2 + 32M^2\lambda_2} + (\lambda_1\lambda_2)^{-5} \kt{93} \estop \nonumber \\
&\lsim \kt{6}\sqrt{\lambda_2 + 32M^2\lambda_2} + (\lambda_1\lambda_2)^{-5} \kt{93} \estop &\mbox{by \cref{eq:setlambda1lambda2ratio}} \nonumber \\
&\lsim \kt{7}\sqrt{\lambda_2} + (\lambda_1\lambda_2)^{-5} \kt{93} \estop \label{eq:eps4bound}
\end{align}

In the following, let $C > 0$ be some large enough universal constant. \cref{eq:2anpcond1} holds, because
\begin{align*}
\max(\eps^{(1)},\eps^{(3)},\eps^{(4)}) &< C(\kt{7}\sqrt{\lambda_2} + (\lambda_1\lambda_2)^{-8}\kt{133}\estop)  &\mbox{by \cref{eq:eps1bound,eq:eps3bound,eq:eps4bound}} \\
&\leq 1/(2^{21} M^7 L) + C(\lambda_1\lambda_2)^{-8}\kt{133}\estop &\mbox{since $\sqrt{\lambda_2} \leq 1 / (2^{21} \kt{14} C)$ by \cref{eq:setlambda2}} \\
&\leq 1/(2^{20} M^7 L) &\mbox{since $\estop \leq 1 / (2^{21} \kt{140} C)$ by \cref{eq:setestop}} \\
&= \tau &\mbox{by \cref{eq:settau}}
\end{align*}
And \cref{eq:2anpcond2} holds,
\begin{align*}
&|\hat{\zeta}(\emptyset; w^0)| + 2(r_1^2 + r_2^2)(\max(\eps^{(1)},\eps^{(3)},\eps^{(4)}))^2 + \estat \\
&\leq 2(r_1^2 + r_2^2)(\max(\eps^{(1)},\eps^{(3)},\eps^{(4)}))^2 + 2\estat &\mbox{by $\Ebias{t-1}$} \\
&\leq 8M^2(\max(\eps^{(1)},\eps^{(3)},\eps^{(4)}))^2 + 2\estat &\mbox{by \cref{eq:goodr1r2interval}} \\
&\leq 8M^2 /(2^{20} M^7 L)^2 + 2\estat  &\mbox{by \cref{eq:eps1bound,eq:eps3bound,eq:eps4bound}} \\
&\leq 1/(2^{21} M^7 L) + 2\estat \\
&\leq 1/(2^{21} M^7 L)  + C(\lambda_1\lambda_2)^{-3} \kt{67} \estop &\mbox{by \cref{eq:goodestatbound}} \\
&\leq 1/(2^{20} M^7 L) &\mbox{by $\estop \leq (\lambda_1\lambda_2)^3 \kt{-74} / (2^{21} C)$ in \cref{eq:setestop}}  \\
&< \tau &\mbox{by \cref{eq:settau}}
\end{align*}
Finally, \cref{eq:2anpcond5} applies, since 
\begin{align*}
4&(r_1^2 + r_2^2) \estat  / \min(\lambda_1, \lambda_2) \\
&= 4(r_1^2 + r_2^2) \estat  / \lambda_2 \\
&\leq 32\kt{2} \estat / \lambda_2 &\mbox{by \cref{eq:goodr1r2interval}} \\
&\lsim C(\lambda_1\lambda_2)^{-4}\kt{69} \estop &\mbox{by \cref{eq:goodestatbound}} \\
&\leq 1/2 &\mbox{since $\estop \leq (\lambda_1\lambda_2)^4 \kt{-69} / (2C)$ by \cref{eq:setestop}}
\end{align*}
Thus, all the preconditions to \cref{lem:2activenotprimed} hold.

\textit{Applying \cref{lem:2activenotprimed}}. So we conclude that if the event $\Estatiter{t}$ for the $t$th call to $\TrainNeuron$ holds, then we have $w^t = w^{t-1}$. The network weights are unchanged and neuron $v$ remains blank. Thus $\Estepgood{t}$ follows from $\Estepgood{t-1}$ in this case.

\paragraph{Case 2b: Exactly two active inputs, product is useful} The final case is if
$\hat{g}(S) \neq 0$ and $S$ is not represented by some neuron at iteration $t-1$. Since $\Estepgood{t-1}$ implies $\Epol{t-1}$, we have
\begin{align*}|\hat{\zeta}(S;w^{t-1}) - \hat{g}(S)|
&\leq \efouriermove t \\
&\leq 1/(4M) &\mbox{by \cref{claim:efourmoveglobbound}}
\end{align*}
So since $\hat{g}(S) \neq 0$ implies that $|\hat{g}(S)| \in [1/M, M]$ by assumption,
\begin{align}
|\hat{\zeta}(S;w^{t-1})| \in [3/(4M),(5/4)M] \subset [1/(2M), 2M] \label{eq:hzetacase2binterval}
 \end{align}

In this case, we will prove that with polynomially-lower bounded probability $\TrainNeuron$ trains $v$ to approximately represent the monomial $\chi_S(x)$. And otherwise, with high probability it leaves the weights unchanged: $w^{t} = w^{t-1}$. To show this, we will use the guarantees for $\TrainNeuron$ proved in\cref{lem:2activeprimedblank,lem:2activeprimedactive,lem:2activeprimedproblowbound}. In order to apply these, we must first verify the preconditions.

\textit{Checking preconditions of \cref{lem:2activeprimedblank,lem:2activeprimedactive,lem:2activeprimedproblowbound}}. By $\Enothree{i-1}$, we know that $v$ cannot have two active parents on the previous layer. Therefore, we must have $u_2 \in \Vin$ without loss of generality, so $\gamma_2 = \lambda_1$, $\eps_2 = 0$, and $|S_2| \leq 1$ because $v_2$ is an input. Now, if $u_1 \in \Vin$ then the preconditions of the lemma with respect to the sets $S_1,S_2$ also hold because $u_1$ and $u_2$ are distinct inputs and therefore $S_2 \neq \emptyset$ without loss of generality and $S_1 \cap S_2 = \emptyset$.

On the other hand, suppose that $i > 1$ and $u_1 \in V_{i-1}$. Then $|S_1| \leq i$ by $\Ereplayer{i-1}$. Since $\hat{g}(S)\neq 0$ and $\Erep{S,t-1}$ does not hold, $S$ is not represented by a neuron in the first $i-1$ layers so by $\Ereplayer{i-1}$ we conclude that $|S| > i$. Therefore since $|S_2| \leq 1$ and $S \neq S_1$ we have $|S_2| = 1$ and $|S_1| = i$ and $|S| = i+1$. Thus, $f_{u_1}(x;w^{t-1})$ only depends on the variables in $S_1$. This is because by $\Elayergood{i-1}$, we must have $\Enothree{0} \cap \dots \Enothree{i-1}$, so the predecessors of neuron $u_1$ all have in-degree at most $2$, and at least one of the parents is in $\Vin$.

We conclude that in all cases $S_2 \neq \emptyset$, $\eps_2 = 0$, $S_1 \cap S_2 = \emptyset$, and $f_{u_1}(x;w^{t-1})$ depends only on variables $\{x_j\}_{j \in S_1}$ and $f_{u_2}(x;w^{t-1})$ depends only on variables $\{x_j\}_{j \in S_2}$. It only remains to verify \cref{eq:2apcond1,eq:2apcond4,eq:2apactivecond1,eq:2apcond2,eq:twoprimegiven1,eq:tauaeiactiveuppbound,eq:taubvactiveuppbound,eq:etalargerthantau,eq:twor1r2etasquaredvszeta,eq:twor1r2tausquaredvszeta,eq:lossdecnastytermbound,eq:etamorethanzerocond}   ,
which we do below. First, \cref{eq:2apcond1,eq:2apactivecond1,eq:2apcond4} hold, since $\tau > \max(\eps^{(1)},\eps^{(3)})$ and $\tau > |\hat{\zeta}(\emptyset; w^0)| + 2(|r_1|^2 + |r_2|^2) (\eps^{(3)})^2 + \estat$ by the same reasoning as in Case 2a. 

In the arguments below, let $C > 0$ be a sufficiently large universal constant. \cref{eq:2apcond2} holds, since
\begin{align*}
&\sqrt{3}\max(\lambda_1,\lambda_2) / |r_1r_2| \\
&\leq \sqrt{3}\lambda_2 / (1/(2M))^2 &\mbox{by \cref{eq:goodr1r2interval}} \\
&\leq 32 \lambda_2 M^2 \\
&\leq 1/(2M) &\mbox{by $\lambda_2 \leq 1/(64M^3)$ in \cref{eq:setlambda2}} \\
&\leq |\hat{\zeta}(S;w^0)| &\mbox{by \cref{eq:hzetacase2binterval}}\end{align*}
\cref{eq:twoprimegiven1} holds, since
\begin{align*}
&\min(\lambda_1,\lambda_2) \sqrt{ |\hat{\zeta}(S;w^0)| / |r_1r_2|}/8 \\
&= \lambda_1 \sqrt{ |\hat{\zeta}(S;w^0)| / |r_1r_2|}/8 \\
&\geq \lambda_1 / (8\sqrt{|r_1r_2| 2M}) & \mbox{since $|\hat{\zeta}(S;w^0)| \geq 1/(2M)$ by \cref{eq:hzetacase2binterval}} \\
&\geq \lambda_1 / (8 \cdot (2M)^{3/2}) &\mbox{since $|r_1|,|r_2| \leq 2M$ by \cref{eq:goodr1r2interval}} \\
&\geq C (\lambda_1\lambda_2)^{-4} \kt{89} \estop &\mbox{since $\estop \leq (\lambda_1 \lambda_2)^{5} \kt{-91} / (32C)$ by \cref{eq:setestop}} \\
&\geq \eps^{(2)}. &\mbox{by \cref{eq:eps2bound}}
\end{align*}
\cref{eq:tauaeiactiveuppbound} holds because
\begin{align*}
\frac{1}{8} \sqrt{\frac{\lambda_1}{\lambda_2} |\hat{\zeta}(S;w^0)| / |r_1r_2|} &\geq \frac{1}{32M^2} \sqrt{\frac{\lambda_1}{\lambda_2} |\hat{\zeta}(S;w^0)|} &\mbox{since $|r_1|,|r_2| \leq 2M$ by \cref{eq:goodr1r2interval}} \\
&\geq \frac{1}{32 \sqrt{2} M^{(3/2)}} \sqrt{\frac{\lambda_1}{\lambda_2} } &\mbox{ since $|\hat{\zeta}(S;w^0)| \geq 1/(2M)$ by \cref{eq:hzetacase2binterval}} \\
&\geq \frac{1}{2048 \sqrt{2} M^{(7/2)} L} &\mbox{by \cref{eq:setlambda1lambda2ratio}} \\ 
 &>   \tau &\mbox{by \cref{eq:eps2bound}},
\end{align*}
\cref{eq:taubvactiveuppbound} holds because
\begin{align*}
&(|\hat{\zeta}(S;w^0)| / 4) (\min_i |r_i|) / (\max_i |r_i|)  - |\hat{\zeta}(\emptyset;w^0)| - \estat \\
&\geq 1/(8M) (\min_i |r_i|) / (\max_i |r_i|)  - |\hat{\zeta}(\emptyset;w^0)| - \estat &\mbox{since $|\hat{\zeta}(S;w^0)| \geq 1/(2M)$ by \cref{eq:hzetacase2binterval}} \\
&\geq 1/(32M^3)  - |\hat{\zeta}(\emptyset;w^0)| - \estat & \mbox{by \cref{eq:goodr1r2interval}} \\
&\geq 1/(32M^3) - 2\estat &\mbox{by $\Ebias{t-1}$} \\
&\geq 1/(32M^3) - C(\lambda_1\lambda_2)^{-3} \kt{67} \estop &\mbox{by \cref{eq:goodestatbound}} \\
&\geq 1/(64M^3) &\mbox{by $\estop \leq (\lambda_1\lambda_2) \kt{-70} / (64C)$ in \cref{eq:setestop}} \\
&>  \tau &\mbox{by \cref{eq:settau}}
\end{align*}
\cref{eq:etalargerthantau} holds because $\eta \geq 4\tau$ because $\eta = 4\tau$ by definition in \cref{eq:seteta}.

\cref{eq:twor1r2etasquaredvszeta,eq:twor1r2tausquaredvszeta} hold because 
\begin{align*}
|2r_1r_2 \tau^2| &< |2r_1r_2 \eta^2| &\mbox{by \cref{eq:seteta}} \\
&\leq 8M^2 \eta^2 &\mbox{by \cref{eq:goodr1r2interval}} \\
&\leq 1/(32M) &\mbox{since $\eta = 4\tau \leq 1/(16M^2)$ by \cref{eq:seteta}} \\
&\leq |\hat{\zeta}(S;w^0)| / 16 &\mbox{since $|\hat{\zeta}(S;w^0)| \geq 1/(2M)$ by \cref{eq:hzetacase2binterval}}
\end{align*}
In order to show \cref{eq:lossdecnastytermbound}, we first prove
\begin{align*}
\estat &\leq C(\lambda_1\lambda_2)^{-3}\kt{67} \estop &\mbox{by \cref{eq:goodestatbound}} \\
&\leq 1/(2^{36} M^{14} L^2) &\mbox{since $\estop \leq (\lambda_1\lambda_2)^3\kt{-81} / (2^{36} C)$ by \cref{eq:setestop}} \\
&= \eta^2 &\mbox{by \cref{eq:seteta}}
\end{align*}
Therefore, \cref{eq:lossdecnastytermbound} holds because
\begin{align*}
|r_1r_2\eta^2\hat{\zeta}(S;w^0)| &\geq 1/(4M^2)\eta^2|\hat{\zeta}(S;w^0)| &\mbox{by \cref{eq:goodr1r2interval}} \\
&\geq 1/(8M^3) \eta^2 &\mbox{by \cref{eq:hzetacase2binterval}} \\
&\geq 2^{12} M^4 \eta^4 &\mbox{by $\eta^2 \leq 2^{-15} M^{-7}$ in \cref{eq:seteta}} \\
&\geq 2^{10} M^4 (\eta^2 + \estat)^2 &\mbox{by $\estat < \eta^2$ proved above} \\
&\geq 16(4M^2\eta^2 + 4M^2 \eta^2 + \estat)^2  \\
&\geq 16(r_1^2 \eta^2 + r_2^2 \eta^2 + \estat)^2 &\mbox{by \cref{eq:goodr1r2interval}} \\
&\geq 16(r_1^2 \eta^2 + r_2^2 \eta^2 + |\hat{\zeta}(\emptyset;w^0)|)^2 &\mbox{by $\Ebias{t-1}$}
\end{align*}
Finally, in order to show \cref{eq:etamorethanzerocond}, we first prove the following two bounds:
\begin{align*}
   \frac{1}{32} |r_1r_2\eta^2\hat{\zeta}(S;w^0)| &\geq \frac{1}{256 M^3} \eta^2 &\mbox{by \cref{eq:goodr1r2interval,eq:hzetacase2binterval}} \\
   &>   2\lambda_2 W \eta^2 &\mbox{since $\lambda_2 W < 1/(512M^3)$ by \cref{eq:setlambda2}}
\end{align*}
and
\begin{align*}
&\frac{1}{32} |r_1r_2\eta^2\hat{\zeta}(S;w^0)| \\
&\geq \frac{1}{256M^3} \eta^2 &\mbox{by \cref{eq:goodr1r2interval,eq:hzetacase2binterval}} \\
&= 2^{-44} M^{-17} L^{-2} &\mbox{by \cref{eq:seteta}} \\
&\geq 2^{11} C M^4 L \estop &\mbox{by $\estop <   \kt{21}/ (2^{55} C)$ in \cref{eq:setestop}} \\
&\geq 2C ((2M)^3 + 1) (45ML \estop)  \\
&\geq 2C ((2M)^3 + 1) \erel{i-1} &\mbox{by \cref{claim:erelglobbound}} \\
&\geq 2C ((2M)^3 + 1) (\max_{i \in \{1,2\}} \eps_i) &\mbox{by $\Erep{i-1}$} \\
&> 2C \max_{u \in V} \max_{x \in \{-1,1\}^n} (|f_u(x;w^0)|^3 + 1) (\max_{i \in \{1,2\}} \eps_i) &\mbox{by $\Eneurbound{t-1}$}
\end{align*}
\cref{eq:etamorethanzerocond} holds by combining the above two bounds, since we take $C$ greater than or equal to the constant from \cref{lem:2activeprimedproblowbound}:
\begin{align*}
\frac{1}{32} |r_1r_2\eta^2\hat{\zeta}(S;w^0)| > \lambda_2 W \eta^2  + C \max_{u \in V} \max_{x \in \{-1,1\}^n} (|f_u(x;w^0)|^3 + 1) (\max_{i \in \{1,2\}} \eps_i)
\end{align*}
  
Therefore, the preconditions of \cref{lem:2activeprimedblank,lem:2activeprimedactive,lem:2activeprimedproblowbound} all hold.

\textit{Applying \cref{lem:2activeprimedblank,lem:2activeprimedactive,lem:2activeprimedproblowbound}}. Let $\Ecasetwoiter{t}$ be the event defined in \cref{def:ecasetwo} for the iteration $t$ call to $\TrainNeuron$. \cref{lem:2activeprimedblank} states that if $(\neg \Ecasetwoiter{t}) \cap \Estatiter{t}$ holds, then $w^t = w^{t-1}$. In this case $\Estepgood{t}$ follows from $\Estepgood{t-1}$ because the parameters of the neural network are unchanged.

On the other hand, if $\Ecasetwoiter{t} \cap \Estatiter{t}$ holds, then \cref{lem:2activeprimedactive} states that the weights $w_v$ corresponding to neuron $v$ are trained so that neuron $v$ becomes an active neuron. In particular, \cref{item:2apa1} of \cref{lem:2activeprimedactive} states that $f_v(x;w^t) = r \chi_{S}(x) + h(x)$, where
$$|r + \hat{\zeta}(S;w^{t-1})| \leq \frac{4\sqrt{\gamma_1 \gamma_2}}{|r_1r_2|} \leq 4M\lambda_2 \leq \elearned,$$
and $h(x) \leq |r|\enewrel$ for
$$\enewrel = (4\estop + 2|\hat{\zeta}(\emptyset;w^{t-1})|)/|\hat{\zeta}(S;w^{t-1})| + 32 \frac{\lambda_2}{\lambda_1} |\eps_1|^2 |r_1/r_2| + \eps_1 (8\frac{|r_1|}{|r_2|} \sqrt{\frac{\gamma_2}{\gamma_1}} + 1).$$
Recall that $\gamma_2 = \lambda_1$ since we have assumed without loss of generality that $u_2 \in \Vin$. If $u_1 \in \Vin$ then we also have $\eps_1 = 0$ because $f_{u_1}(x)$ computes either the constant $1$ or an input monomial in $x_1,\ldots,x_n$. Therefore, $$\enewrel = (4\estop + 2|\hat{\zeta}(\emptyset;w^{t-1})|) / |\hat{\zeta}(S;w^{t-1})| \leq 16M\estop \leq \erel{i}$$ by $\Ebias{t-1}$ and \cref{eq:hzetacase2binterval}. On the other hand, if $u_1 \not\in \Vin$ then we have $u_1 \in V_{i-1}$ because it must be in the previous layer, and the regularization is $\gamma_1 = \lambda_2$ because $u_1$ is not an input in $\Vin$. Also, by $\Ereplayer{i-1}$ we must have that the relative error $f_{u_1}(x)$ is $\eps_1 \leq \erel{i-1}$. So if $u_1 \not\in \Vin$ then \begin{align*}
\enewrel &\leq \frac{(4\estop + 2|\hat{\zeta}(\emptyset;w^{t-1})|)}{|\hat{\zeta}(S;w^{t-1})|} + 32 \frac{\lambda_2 |r_1|}{\lambda_1|r_2|} |\erel{i-1}|^2 + \eps_1 (8\frac{|r_1|}{|r_2|} \sqrt{\frac{\lambda_1}{\lambda_2}} + 1) \\
&\leq 16M\estop + 32 \frac{\lambda_2}{\lambda_1} |\erel{i-1}|^2 |r_1/r_2| + (8\frac{|r_1|}{|r_2|} \sqrt{\frac{\lambda_1}{\lambda_2}} + 1)\erel{i-1} &\mbox{by $\Ebias{t-1}$ and  \cref{eq:hzetacase2binterval}} \\
&\leq 16M\estop + 128M^2 \frac{\lambda_2}{\lambda_1} |\erel{i-1}|^2 + (32M^2 \sqrt{\frac{\lambda_1}{\lambda_2}} + 1)\erel{i-1} &\mbox{by \cref{eq:goodr1r2interval}} \\
&= \erel{i}
\end{align*}
In both cases $\enewrel \leq \erel{i}$, and so $\Erep{S,t}$ holds because the network has been updated so that neuron $v$ now computes $\chi_S$ with at most $\erel{i}$ relative error.

Finally, \cref{lem:2activeprimedactive} allows us to prove that $\Estepgood{t}$ holds. First, we show that $\Epol{t}$ holds. The condition on $\hat{\zeta}(S;w^t)$ follows since $\hat{\zeta}(S;w^t) = \hat{\zeta}(S;w^{t-1}) + r$, and so we have $|\hat{\zeta}(S;w^t)| \leq \elearned \leq \efouriermove t + \elearned$. Furthermore, since $\neg \Erep{S,(t-1)}$, by $\Epol{t-1}$ we have $|\hat{\zeta}(S;w^{t-1}) + \hat{g}(S)| \leq \efouriermove (t-1)$, and so combining by triangle inequality with the bound on $|r + \hat{\zeta}(S;w^{t-1})|$, we have
\begin{align*}
|r - \hat{g}(S)| \leq \efouriermove t + \elearned \leq M/4
\end{align*}
This means that $|r| \leq |\hat{g}(S)| + M/4 \leq 5M/4$. So $|h(x)| \leq |r| \erel{i} \leq (5M/4)(45 ML \estop) \leq 100ML^2 \estop = \efouriermove$ by \cref{claim:erelglobbound}. Since for any $S' \subset [n]$, we have $|\hat{\zeta}(S';w^t) - \hat{\zeta}(S';w^{t-1})| = |\hat{f_v}(S';w^t)| = |\hat{h}(S')| = |\EE_{x \sim \{-1,1\}^n}[h(x)\chi_{S'}(x)]| \leq \max_x |h(x)| \leq \efouriermove$. Therefore, $\Epol{t}$ follows from $\Epol{t-1}$ and this bound.

To prove that $\Eneurbound{t}$ holds, note that $|f_u(x;w^t)| = |f_u(x;w^{t-1})| \leq 2M$ for all $u \neq v$ by $\Eneurbound{t-1}$. And $|f_v(x;w^t)| \leq |r| + |h(x)| \leq (5/4M)(1 + 45ML\estop) \leq 2M$ since $45ML\estop \leq 1/2$ by \cref{claim:erelglobbound}.

$\Eparambound{t}$ holds because
\begin{align*}
\max_{e \in E} |a_e^{t}| &= \max(\max_{e \in E} |a_e^{t-1}|, \max_{e = (u,v) \in E} |a_e^t|) \\
&\leq \max(16M^2 \sqrt{\lambda_2 / \lambda_1}, \max_{e = (u,v) \in E} |a_e^t|) &\mbox{by $\Eparambound{t-1}$} \\
&\leq \max(16M^2 \sqrt{\lambda_2 / \lambda_1}, 4\sqrt{\frac{\lambda_2}{\lambda_1} |\hat{\zeta}(S;w^0)| / |r_1r_2|}) &\mbox{by \cref{item:2apaebound} of \cref{lem:2activeprimedactive}} \\
&\leq 16M^2 \sqrt{\lambda_2 / \lambda_1} &\mbox{by \cref{eq:hzetacase2binterval,eq:goodr1r2interval}}
\end{align*}
$\Ebias{t}$ holds by \cref{item:2apabias} of \cref{lem:2activeprimedactive}. And $\Enobadactive{t}$ holds because the active neuron that has been created represents $S$, where $\hat{g}(S) \neq 0$ and $|S| = i+1$.

Thus, in this case $\Estepgood{t} = \Estepgood{t-1} \cap \Epol{t} \cap \Eneurbound{t} \Eparambound{t} \cap \Ebias{t}$ holds. Therefore, our analysis shows that if $\hat{g}(S) \neq 0$ and $(\neg \Erep{S,t-1}) \cap \Estepgood{t-1} \cap \Elayergood{i-1}$ holds, then
\begin{align*}
\PP[\Erep{S,t} \mid w^{t-1}] &\geq \PP[\Ecasetwoiter{t} \cap \Estatiter{t} \mid w^{t-1}] \\
&\geq \min(1,\sqrt{ |r_1r_2 \hat{\zeta}(S;w^{t-1})|}/8) -  \PP[\neg \Estatiter{t} \mid w^{t-1}] &\mbox{by \cref{lem:2activeprimedproblowbound}}\\
&\geq \min(1,\sqrt{|\hat{\zeta}(S;w^{t-1})|}/(16M)) -  \PP[\neg \Estatiter{t} \mid w^{t-1}] &\mbox{by \cref{eq:goodr1r2interval}} \\
&\geq \min(1,1/(32M^2)) -  \PP[\neg \Estatiter{t} \mid w^{t-1}] &\mbox{by \cref{eq:hzetacase2binterval}} \\
&\geq 1 / (32M^2) -  \PP[\neg \Estatiter{t} \mid w^{t-1}] \\
&\geq 1/(64M^2),
\end{align*}
since $\PP[\neg \Estatiter{t} \mid w^{t-1}] \leq \dstat \leq 1/(64M^2)$ by \cref{lem:goodestatlowbound}.
\end{proof}

\subsubsection{$\Ereplayer{i} \cap \Estepgood{t_i}$ follows from $\Estepgood{t_{i-1}} \cap \Elayergood{i-1}$ with high probability}

Another ingredient in the induction is showing that the updates from iterations $t_{i-1}+1$ through $t_i$ suffice for $\Ereplayer{i}$ to hold with high probability. Essentially, if the degree at most $i$ monomials were represented after training layers $1$ through $i-1$, then with high probability the degree at most $i+1$ monomials are represented after training layers $1$ through $i$.

\begin{lemma}\label{lem:replayerinduction}
$\PP[\Ereplayer{i} \cap \Estepgood{t_i} \mid \Estepgood{t_{i-1}} \cap \Elayergood{i-1}] \geq 1 - Ws\dstat - \delta / (8L)$
\end{lemma}
\begin{proof}
$\Estepgood{t_i}$ implies $\Enobadactive{t_i}$. Therefore it remains to show that for any $S \subset [n]$ with $|S| \leq i+1$ and $\hat{g}(S) \neq 0$ that $\Erep{S,t_i}$ holds with high probability.

Suppose that $i > 1$, then for any $S$ with $|S| \leq i$, we have that $\Erep{S,t_i}$ holds by the inductive hypothesis $\Elayergood{i-1}$. Therefore, it remains to prove $\Erep{S,t_i}$ holds with high probability for any $S \subset [n]$ such that $|S| = i+1$ and $\hat{g}(S) \neq 0$. 

Fix such a subset $S$ with $|S| = i+1$ and $\hat{g}(S) \neq 0$. Since $g$ satisfies the staircase property in \cref{def:staircaseprop}, there must be a set $S_1 \subset S$ such that $\hat{g}(S') \neq 0$ and $|S \sm S_1| = 1$. By the event $\Ereplayer{i-1}$, which is implied by $\Elayergood{i-1}$, there is a neuron $u_1 \in V_{i-1}$ such that $u_1$ represents $S_1$. On the other hand, letting $S_2 = S \sm S_1$, because $|S_2| = 1$ there is a neuron $u_2 \in \Vin$ such that $u_2$ represents $S_2$. Therefore, by $\Econn{i-1}$, it holds that $|\{v \in V_i: (u_1,v), (u_2,v) \in E\}| \geq \nshared$.

Let $t^{(1)} \leq \dots \leq t^{(k)}$ be the iterations such that the neuron $v \in V_i$ trained at iteration $t^{(j)}$ satisfies $(u_1,v), (u_2,v) \in E\}$. By the above argument, $k \geq \nshared$.

For any $t \in (t_{i-1},t_i]$ if $w^{t-1}$ is such that $(\neg \Erep{S,t-1}) \cap \Estepgood{t-1} \cap \Elayergood{i-1}$ holds, then by \cref{lem:goodstepinduction} we have that $\Estepgood{t} \cap \Elayergood{i-1}$ holds with probability at least $1 - \dstat$. In addition, if $t \in \{t^{(1)},\ldots,t^{(k)}\}$, then $\Erep{S,t} \cap \Estepgood{t} \cap \Elayergood{i-1}$ holds with probability at least $1 / (64M^2)$. Since once $\Erep{S,t}$ holds, it is also true that $\Erep{S,t'}$ holds for all $t' \geq t$, analyzing the Markov chain implies \begin{align*}\PP[\Erep{S,t_i} \cap \Estepgood{t_i} \mid \Estepgood{t_{i-1}} \cap \Elayergood{i-1}] &\geq 1 - W\dstat - (1 - 1/(64M^2))^k \\
&\geq 1 - W\dstat - \delta / (8sL),\end{align*} since $k \geq \nshared \geq 64M^2 \log(16sL/\delta)$ by \cref{eq:setnshared}.

By a union bound over all $S$ such that $|S| = i+1$ and $\hat{g}(S) \neq 0$, we have $\PP[\Ereplayer{i} \cap \Estepgood{t_i} \mid \Estepgood{t_{i-1}} \cap \Estepgood{i-1}] \geq 1 - Ws\dstat - \delta / (8L)$.

The case where $i = 1$ is similar: here it suffices to show that $\Erep{S,t_1}$ holds with high probability for any $S \subset [n]$ such that $|S| \leq 2$ and $\hat{g}(S) \neq 0$. An analogous argument to the above works, appealing to $\Econn{0}$ and the fact that for each $S' \subset [n]$ with $|S'| \leq 1$, there is a neuron $u \in \Vin$ computing $\chi_S$.
\end{proof}

\subsubsection{$\Econn{i} \cap \Enothree{i}$ follows from $\Ereplayer{i}$ with high probability}

The final element of the inductive step is to guarantee that the network connectivity events for the edges to layer $i$ after the training of layers $1$ through $i-1$ has concluded. The idea behind proof here is that the edges to layer $i$ are independent of the state of the network parameters at iteration $t_{i-1}$, and since $\Ereplayer{i-1}$ guarantees that there are at most $s$ active neurons at iteration $t_{i-1}$ we may ensure these events hold with high probability.

\begin{lemma}\label{lem:econnproblowbound}
For any $i \in \{0,\ldots,L-1\}$, conditioned on $\Ereplayer{i}$ and $w^{t_i}$, the event $\Econn{i}$ holds with probability at least $1 - \delta / (8L)$.
\end{lemma}
\begin{proof}
First we consider shared children of pairs of inputs in $\Vin$. By \cref{eq:setW,eq:setp1},
\begin{align*}
(p_1)^2 W &\geq 10\log(4WL/\delta) \nshared \\
&\geq 10\log(4(n+1)L/\delta) \nshared,
\end{align*} for any distinct $u,u' \in \Vin$ we have $$\PP[|\{v_1 \in V_1 : (u,v_1) , (u',v_1) \in E\}| \geq \nshared] \geq 1 - \delta / (4(n+1)L)^2$$ by a Hoeffding bound, as all edges from $\Vin$ to $V_{i+1}$ are i.i.d. with probability $p_1$ and independent of $w^{t_i}$. Therefore, by a union bound, for all pairs of distinct $u,u' \in \Vin$, with probability at least $1 - \delta / (16L)$ we have that $|\{v_1 \in V_1 : (u,v_1), (u',v_1)\}| \geq \nshared$.

Now, we consider the number of children of an input and a neuron on the previous layer. For $i \geq 1$, note that $\Ereplayer{i}$ implies that the number of active neurons in $V_i$ after iteration $t_i$ must be at most $s$ -- because each active neuron corresponds to a unique nonzero Fourier coefficient of $g$. Furthermore, these active neurons are trained independently of the edges from layer $i$ to layer $i+1$. Hence for any active neuron $v_i \in V_i$ that is active at iteration $t_i$ and any input $u \in \Vin$, the expected number of neurons $v_{i+1} \in V_{i+1}$ that have $v_i$ and $u$ as parents is
\begin{align*}
p_1p_2 W &\geq 10\log(4WL/\delta)\nshared,
\end{align*}
by \cref{eq:setW,eq:setp1,eq:setp2}. So by a Hoeffding bound,
\begin{align*}\PP[|\{v_{i+1} \in V_{i+1} : (v_i,v_{i+1}) \in E, (u,v_{i+1}) \in E\}| \geq \nshared] &\geq 1 - \delta/(4WL)^2 \\
&\geq 1 - \delta/(16 W(n+1) L).
\end{align*}
Finally, a union bound over the at most $W(n+1)$ pairs of an input $u \in \Vin$ and a neuron on layer $V_i$ imply that with probability at least $1 - \delta / (16 L)$ for all such pairs $|\{v_{i+1} \in V_{i+1} : (v_i,v_{i+1}) \in E, (u,v_{i+1}) \in E\}| \geq \nshared$.

The lemma follows by a union bound of the above two results.
\end{proof}

\begin{lemma}\label{lem:enothreeproblowbound}
For any $i \in \{0,\ldots,L-1\}$, conditioned on $\Ereplayer{i}$ and on $w^{t_i}$, the event $\Enothree{i}$ holds with probability at least $1 - \delta / (8 L)$.
\end{lemma}
\begin{proof}
By $\Ereplayer{i}$, there are at most $s$ active neurons in $V_i$, because each one corresponds to a distinct nonzero Fourier coefficient of $g$. For any $v \in V_{i+1}$ and distinct $u_1,u_2,u_3 \in \Vin \cup \{u \in V_i : u \mbox{is active at iteration } t_i\}$, call the tuple $(v,u_1,u_2,u_3)$ ``bad'' if $(u_1,v),(u_2,v),(u_3,v) \in E$. The probability that a tuple $(v,u_1,u_2,u_3)$ is bad is at most $\max(p_1,p_2)^3 = (p_1)^3$, since these edges are independent of $w^{t_i}$, because by the layerwise training $w^{t_i}$ depends only on presence or absence the edges up to the layer $V_i$. The number of bad tuples is thus at most
\begin{align*}
(p_1)^3 W(s+n+1)^3 &\leq \\
&\leq \delta / (16 L)
\end{align*} in expectation by \cref{eq:setW,eq:setp1}. Therefore, by a Markov bound there are no bad tuples with probability at least $1 - \delta / (16L)$.

Furthermore, the number of neurons $v \in V_{i+1}$ that have at least two active $u_1,u_2 \in V_i$ is in expectation at most \begin{align*}
(sp_2)^2 W &\leq 
&\leq \delta / (16L)
\end{align*}, by a similar argument and \cref{eq:setW,eq:setp2}. So a Markov bound shows there are no such neurons with probability at least $1 - \delta/(16L)$.

So by a union bound $\Enothree{i}$ holds with probability at least $1 - \delta/(8L)$.
\end{proof}

\subsubsection{Proof of \cref{thm:restatedtheorem}}

We conclude by combining the inductive steps \cref{lem:goodstepinduction,lem:replayerinduction} to prove that $\Ereplayer{L}$ holds with high probability, and then recalling this is sufficient by \cref{lem:ereplayerLsuffices}.

\begin{lemma}\label{lem:ereplayerLproblowbound}
$\Ereplayer{L}$ holds with probability at least $1 - \delta$.
\end{lemma}
\begin{proof}
We prove by induction on $i$ that for any $i \in \{0,\ldots,L-1\}$, $$\PP[\Elayergood{i} \cap \Estepgood{t_i}] \geq 1 - (2i+1)\delta / (4L).$$

For the base case $i = 0$, note that $\Ebias{0}$ holds because $\hat{\zeta}(\emptyset;w^0) = -\hat{g}(\emptyset) = 0$. Further, $\Eneurbound{0}$, $\Eparambound{0}$, and $\Enobadactive{0}$ follow from the fact that the network is initialized to all zeros, and $\Epol{0}$ holds because $\hat{\zeta}(S;W^0) = -\hat{g}(S)$ for all $S \subset [n]$. $\Ereplayer{0}$ holds because by definition it always holds. Finally, given $\Ereplayer{0}$, \cref{lem:econnproblowbound,lem:enothreeproblowbound} imply that $\Econn{0} \cap \Enothree{0}$ hold with probability at least $1 - \delta / (4L)$. Combining these with the definition of $\Elayergood{0}$ and $\Estepgood{0}$ in \cref{def:estepgood,def:elayergood}, it follows that $$\PP[\Elayergood{0} \cap \Estepgood{t_0}] \geq 1 - \delta / (4L).$$

For the inductive step for any $i > 1$, \cref{lem:replayerinduction} implies that $$\PP[\Ereplayer{i} \cap \Estepgood{t_i} \mid \Elayergood{i-1} \cap \Estepgood{t_{i-1}}] \geq 1 - Ws\dstat - \delta / (8L) \geq 1 - \delta / (4L).$$ Also, \cref{lem:econnproblowbound,lem:enothreeproblowbound} imply that 
$$\PP[\Econn{i} \cap \Enothree{i} \mid \Ereplayer{i} \cap \Estepgood{t_i} \cap \Elayergood{i-1}] \geq 1 - \delta / (4L).$$

Since $$\Elayergood{i} = \Elayergood{i-1} \cap \Ereplayer{i} \cap \Econn{i} \cap \Enothree{i},$$ we conclude that $$\PP[\Elayergood{i} \cap \Estepgood{t_i}] \geq \PP[\Elayergood{i-1} \cap \Estepgood{t_{i-1}}] - \delta / (2L) \geq (2i+1) \delta / (4L).$$ This concludes the induction.

Applying the claim with $i = L-1$, we obtain $$\PP[\Elayergood{L-1} \cap \Estepgood{t_{L-1}}] \geq 1 - (2L-1) \delta / (4L),$$ and so by one final application of \cref{lem:replayerinduction} we have $\PP[\Ereplayer{L} \cap \Estepgood{t_{L}}] \geq 1 - (2L) \delta / 2 \geq 1 - \delta/2$.
\end{proof}

\begin{proof}[Proof of \cref{thm:restatedtheorem}]
By \cref{lem:ereplayerLsuffices} that if $\Ereplayer{L}$ holds then $\ell(w^{t_L}) < \eps$. Further, \cref{lem:ereplayerLproblowbound} proves that $\PP[\Ereplayer{L}] \geq 1 - \delta$. The runtime bound follows because there are $t_L = WL = O(\kt{})$ iterations, each of which can be implemented in $O(\kt{2393})$ time and samples by \cref{lem:trainneuronstationarity}.
\end{proof}

\end{document}